\documentclass[sigconf]{acmart}

\AtBeginDocument{%
  \providecommand\BibTeX{{%
    \normalfont B\kern-0.5em{\scshape i\kern-0.25em b}\kern-0.8em\TeX}}}

\setcopyright{acmlicensed}
\copyrightyear{2024}
\acmYear{2024}
\acmDOI{XXXXXXX.XXXXXXX}

\acmConference[Preprint]{}{}{}

\usepackage{amsmath,amsfonts,bm}
\usepackage{xspace}
\usepackage{xifthen}

\usepackage{graphicx}
\usepackage{multirow}
\usepackage{booktabs} 
\usepackage{array} 
\usepackage{pifont} 
\usepackage[ruled,vlined]{algorithm2e}
\usepackage{colortbl}
\usepackage{xcolor}
\usepackage{array}
\usepackage{fontawesome5}
\usepackage{adjustbox}
\usepackage{circledtext}
\usepackage{pifont}
\usepackage{tcolorbox}
\usepackage{lipsum} 
\usepackage{graphicx}  
\usepackage[utf8]{inputenc}
\usepackage{textcomp}
\usepackage{subcaption}

\renewcommand{\vec}[1]{\mathbf{#1}}                     
\newcommand{\set}[1]{\uppercase{#1}}                    

\newcommand{\func}[1]{\mathcal{\uppercase{#1}}}     
\newcommand{\op}[1]{\mathbb{\uppercase{#1}}}        

\newcommand{\indep}{\perp \!\!\! \perp}

\newcommand{\eqnspace}{\;\;\;\;}

\newtheorem{postulate}{Postulate}
\newtheorem{theorem}{Theorem}

\newcommand{\header}[1]{\textbf{#1:}}
\newcommand{\highlight}[3]{\textcolor{#2}{#1: #3}}
\newcommand{\xmark}{\text{\ding{55}}}

\newcolumntype{G}{>{\columncolor{gray!30}}c}

\definecolor{PastelApricot}{RGB}{255, 229, 204}
\definecolor{PastelMelon}{RGB}{255, 179, 186}

\definecolor{PastelLavender}{rgb}{0.8, 0.8, 1.0}
\definecolor{PastelLavenderFrame}{rgb}{0.7, 0.7, 0.9}

\tcbset{
    lmprompt/.style={
        colback=PastelLavender!5!white,  
        colframe=PastelLavenderFrame!75!black, 
        fonttitle=\bfseries,
        coltitle=black,
        colbacktitle=PastelLavenderFrame!75!white,
        boxrule=0.4mm,
        width=\textwidth,
        before skip=10pt, 
        after skip=10pt, 
    }
}

\newcommand{\renee}{Ren\'ee\xspace}
\newcommand{\dexml}{DEXML\xspace}
\newcommand{\alg}{SKIM\xspace}
\newcommand{\llama}{LLaMA2}

\newcommand{\dataset}{\func D\xspace}

\newcommand{\orcas}{LF-Orcas-800K\xspace}
\newcommand{\wiki}{LF-WikiTitlesHierarchy-2M\xspace}

\newcommand{\cmark}{\ding{51}} 

\newcommand{\slmaug}{UDAPDR\xspace}

\newcommand{\yash}[1]{\highlight{YASH}{purple}{#1}}

\definecolor{tealblue}{rgb}{0.21, 0.46, 0.53}

\newcommand{\tab}{\phantom{xxxx}\xspace}
\newcommand{\fourtabs}{\tab\tab\tab\tab}
\newcommand{\eighttabs}{\fourtabs\fourtabs\xspace}

\begin{document}

\title{On the Necessity of World Knowledge for Mitigating \\ Missing Labels in Extreme Classification}

\author{Jatin Prakash}
\authornote{Equal contribution; Work done while at Microsoft Research; Correspondence to: jatin.prakash@nyu.edu, anirudh.buvanesh@umontreal.ca, yprabhu@microsoft.com}
\affiliation{%
  \institution{New York University}
  \country{USA}
}

\author{Anirudh Buvanesh}
\authornotemark[1]
\affiliation{%
  \institution{Mila, Université de Montréal}
  \country{Canada}
  }

\author{Bishal Santra}
\affiliation{%
  \institution{Microsoft Research}
  \country{India}
}

\author{Deepak Saini}
\affiliation{%
  \institution{Microsoft}
  \country{USA}
}

\author{Sachin Yadav}
\affiliation{%
  \institution{Microsoft Research}
  \country{India}
}

\author{Jian Jiao}
\affiliation{%
  \institution{Microsoft}
  \country{USA}
}

\author{Yashoteja Prabhu}
\affiliation{%
  \institution{Microsoft Research}
  \country{India}
}

\author{Amit Sharma}
\affiliation{%
  \institution{Microsoft Research}
  \country{India}
}

\author{Manik Varma}
\affiliation{%
  \institution{Microsoft Research}
  \country{India}
}





\renewcommand{\shortauthors}{Prakash et al.}

\received{20 February 2007}
\received[revised]{12 March 2009}
\received[accepted]{5 June 2009}
\newcommand{\mcheck}{\raisebox{-2pt}{\includegraphics[height=1em]{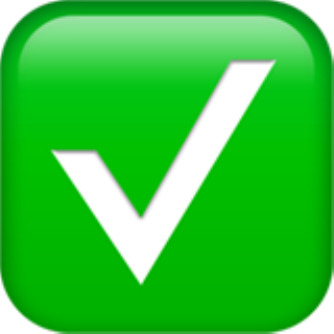}}}
\newcommand{\mques}{\raisebox{-2pt}{\includegraphics[height=1em]{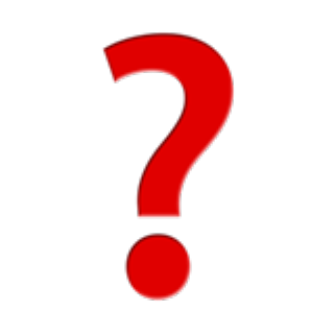}}}



\begin{abstract}

Extreme Classification (XC) aims to map a query to the most relevant documents from a very large document set. XC algorithms used in real-world applications learn this mapping from datasets curated from implicit feedback, such as user clicks. However, these datasets inevitably suffer from missing labels. In this work, we observe
that \textit{systematic} missing labels lead to missing knowledge, which is critical for accurately modelling relevance between queries and documents. 
We formally show that this absence of knowledge cannot be recovered using existing methods such as propensity weighting and data imputation strategies that solely rely on the training dataset. While LLMs provide an attractive solution to augment the missing knowledge, leveraging them in applications with low latency requirements and large document sets is challenging. To incorporate missing knowledge at scale, we propose \alg (\textbf{S}calable \textbf{K}nowledge \textbf{I}nfusion for \textbf{M}issing Labels), an algorithm that leverages a combination of small LM and abundant unstructured meta-data to effectively mitigate the missing label problem. 
We show the efficacy of our method on large-scale public datasets through exhaustive unbiased evaluation ranging from human annotations to simulations inspired from industrial settings.
\alg outperforms existing methods on \textit{Recall@100} by more than 10 absolute points. 
Additionally, \alg scales to proprietary query-ad retrieval datasets containing 10 million documents, outperforming contemporary methods by 12\% in offline evaluation and increased ad click-yield by 1.23\% in an online A/B test conducted on a popular search engine. We release our code, prompts, trained XC models and finetuned SLMs at: \textcolor{purple}{\href{https://github.com/bicycleman15/skim}{github.com/bicycleman15/skim}}

\end{abstract}
\maketitle


%

\section{Introduction}
\label{sec:introduction}

Extreme Classification (XC) addresses the challenge of mapping a query to the most relevant subset of documents from a very large set of documents. XC algorithms have demonstrated impressive performance across a wide range of applications, including document tagging \cite{you2019attentionxml, babbar2017dismec}, product recommendation \cite{medini2019extreme, dahiya2023ngame, dahiya2021siamesexml}, and search and advertisement \cite{mohan2024oak,jain2023renee}. However, these datasets are prone to \textit{missing labels}, as it is impossible to annotate all query-document pairs when the document space is vast, consisting hundreds of millions or even billions of documents. 
Retrieval models trained on such datasets with missing labels may fail to accurately capture the query-document relevance. 
In the aforementioned applications, XC algorithms drive the retrieval stage of the pipeline. 
Thus any inaccuracies here cannot be tackled in later stages, such as the ranking layer, and would adversely impact the overall system performance, making the retrieval stage crucial and the focus of this paper.



The extreme classification (XC) community has developed two primary approaches to address the challenge of missing labels. The first approach, propensity-based methods \cite{jain2016extreme, qaraei2021convex, wei2021towards, wydmuch2021propensity}, constructs unbiased loss functions by reweighting the original loss with query-document observational probabilities (i.e. propensities). 
The second approach encompasses naive imputation methods \cite{kharbanda2023gandalf, buvanesh2024enhancing}, which aim to infer missing labels using various techniques. 
However, both these approaches are limited by their heavy reliance on large-scale but potentially biased training data, as shown later in this paper.



%

In this work, we revisit the problem of missing labels in XC from the novel perspective of \textit{world knowledge}. XC, fundamentally a knowledge-intensive task, benefits significantly from external information that aids in disambiguating named entities such as people, brands or locations, remembering useful facts and acquiring domain-knowledge-specific associations.
For example, as shown in Figure~\ref{fig:teaser}, determining the relevance of a document like \textit{``medical.net/what-are-genes''} to the query \textit{``what is an exon?''} requires the understanding that \textit{exons} are integral components of \textit{genes}. Such world knowledge becomes increasingly important in short-text based XC tasks where queries and documents are pithy. 

As XC models are trained using supervised learning algorithms on query-document pair datasets, they need to acquire the necessary knowledge from the training dataset alone. 
However, the vast and diverse nature of world knowledge presents a significant challenge: limited human-annotated ground truth datasets struggle to comprehensively cover the long tail of unique information required for optimal performance. 
We refer to the special case of missing labels in XC, where all the labels containing some knowledge go missing together, as the phenomenon of ``systematic missing label bias.''
We rigorously demonstrate, both theoretically and empirically, that traditional debiasing or supervised learning techniques that rely solely on the training dataset are inherently limited in their ability to recover such biases and solve this learning problem accurately. 
To overcome this limitation, it is necessary to go beyond the world knowledge encoded in the training datasets.

Recently, Large Language Models (LLMs) have been gaining popularity as effective knowledge bases \cite{he2024can}, where they store vast world knowledge in their enormous parameters. Additionally, they can accurately predict user search preferences as good as human labellers \cite{thomas2024large}. With these advances, LLMs offer a potential solution to bridge this missing knowledge accurately in retrieval datasets. However, it is challenging to utilize them in applications having large output spaces as they fail to scale and satisfy the latency requirements of XC applications. A promising alternative is to utilize much more scalable Small Language Models (SLMs) for augmenting missing knowledge since these too are trained on huge amounts of internet data \cite{touvron2023llama}. However, they lack parametric world knowledge making them unreliable for accurate generation \cite{kandpal2023large}. 

\begin{figure}[tbp]
	\centering
\includegraphics[width=\linewidth]{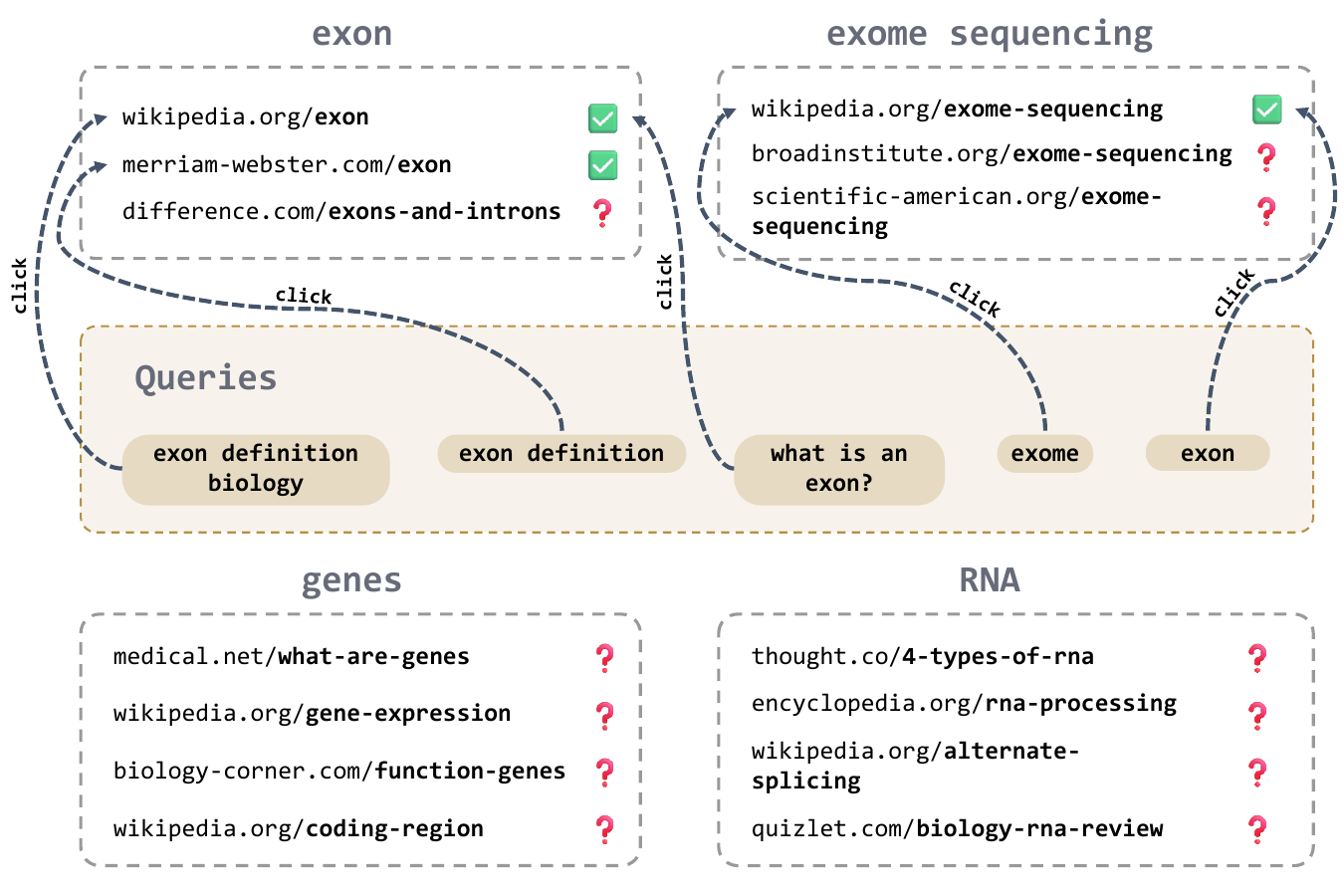}
		\caption{Connections between queries and documents define the knowledge available in a retrieval dataset. In the above example, the document concept \textit{exome sequencing} is connected to the query ``exon'' through a user click, encoding the knowledge that concepts \textit{exon} and \textit{exome sequencing} are related. The document concept \textit{exon} is not directly connected to query ``exome'' but this relationship can possibly be learnt (see Table \ref{tab:label_concepts} in appendix). 
        However, the relevance of \textit{exon} to \textit{genes} or \textit{RNA} is impossible to learn through this dataset because there are no connecting clicks providing this knowledge, that is, those connections (documents) are \textit{systematically} missing for \textit{exon}. (Note:  \mcheck~implies \textit{relevant \& clicked}, \mques~implies \textit{relevant \& missing})
  }
		\label{fig:teaser}
\end{figure}



To address these challenges, we propose \alg (\textbf{S}calable \textbf{K}now-ledge \textbf{I}nfusion for \textbf{M}issing Labels), an algorithm that elegantly combines document metadata and SLMs to mitigate the problem of missing knowledge in extreme classification at scale. 
We observe that in XC datasets, documents have readily available abundant metadata that contain useful knowledge to recover the missing labels, however, metadata is present in the form of unstructured text, making it hard to use \cite{mittal2024graphregularizedencodertraining}.
Thus, our approach consists of two key stages: (i) \textit{diverse query generation}: using unstructured metadata text, we generate diverse synthetic queries that cater to different knowledge items within the document,
(ii) \textit{retrieval-based mapping}: we employ a retrieval model to map the generated synthetic queries to the space of the train queries in the dataset.
Through this process, \alg extracts only the relevant knowledge pertinent to the retrieval task from unstructured metadata and augments it to the existing training dataset, making it generally applicable to any XC or retrieval model.
\alg efficiently scales to datasets containing millions of documents, while maintaining the accuracy and reliability necessary for XC tasks.


In this work, we additionally note that datasets collected from click logs, like their train sets, have test sets corrupted by the same missing labels bias making offline evaluation unreliable. Instead, we propose to evaluate \alg and other baselines using a combination of more trustworthy alternatives: (i) small-scale human-annotated test sets, (ii) simulation settings where we aim to incorporate biases encountered in real-world industrial systems, and (iii) conducting live A/B tests.
We test the efficacy of \alg not just on large-scale public datasets, but also on a real-world query-ad keyword retrieval dataset, where \alg scales seamlessly to 10 million ad keywords (documents). 
\alg outperforms contemporary methods significantly in all offline evaluation metrics, and in A/B tests on a popular search engine showed significant gains in click yield. 

This paper makes the following salient contributions:
\begin{itemize}
    \item Studies a novel and important connection between the world knowledge and the missing labels in XC, and introduces the notion of systematic missing label bias.
    \item Establishes, through rigorous theoretical arguments, the inefficacy of standard debiasing techniques such as propensity-scoring and naive label imputation in mitigating systematic missing labels.
    \item Proposes \alg, a highly accurate and scalable method, to effectively mitigate the systematic missing labels problem by leveraging the knowledge present in LLMs/SLMs and unstructured metadata text.
    \item Reliably evaluates unbiased performance of \alg with competing baselines using a plethora of methods, including human annotated test set, controlled simulation and live A/B test conducted on a popular search engine. 
    We additionally perform extensive ablations to support our design choices.
    \item We open-source our code, trained XC models, finetuned SLMs, generated synthetic datasets using \alg, LLM prompts \& finetuning LLM responses at: \textcolor{purple}{\href{https://github.com/bicycleman15/skim}{github.com/bicycleman15/skim}}
\end{itemize}
\section{Related work}

\textbf{Extreme classification}: 
The leading methods for extreme classification can be broadly categorized into two families: one-versus-all classifier-based methods \cite{jain2016extreme, prabhu2014fastxml} and dual encoders \cite{kharbanda2022cascadexml, jain2023renee, gupta2022elias, dahiya2023ngame, gupta2024dualencoders}. Some of the recent innovations in XC 
include transition from sparse feature based models \cite{agrawal2013multi,babbar2017dismec,jain2016extreme,babbar2019data,bhatia2015sparse,barezi2019submodular,jain2019slice,prabhu2014fastxml,prabhu2018extreme,prabhu2018parabel,khandagale2020bonsai} to deep networks \cite{dahiya2021deepxml, dahiya2023deep, dahiya2023ngame}, end-to-end model training \cite{jain2023renee}, leveraging label text with Siamese networks \cite{lu2020twinbert,qu2020rocketqa,lu2021less},  and effective negative sampling strategies \cite{dahiya2023ngame,jiang2021lightxml,rawat2021disentangling,guo2019breaking,xiong2020approximate,reddi2019stochastic}. 




\textbf{Missing label bias:} 
XC training datasets are known to suffer from missing label bias \cite{jain2016extreme}, owing to the presence of
selection \cite{marlin2012collaborative}, position \cite{collins2018study}, exposure \cite{lee2023uctrl, zhang2020causaldebiasing} and inductive biases \cite{chen2023bias} in retrieval applications. Missing labels cause relevant documents to go missing from the observed training data \cite{schultheis2022missing} resulting to inaccurate model training.

In XC, the most commonly adopted solution to the missing label problem is propensity-based learning \cite{jain2016extreme,qaraei2021convex,wydmuch2021propensity,wei2021towards}. Propensity score estimates the likelihood of a query-document pair being omitted from dataset, despite being relevant. \citet{jain2016extreme} proposed propensity based debiasing of the training and evaluation process for XC models. However, we demonstrate that propensity-based learning cannot recover the "missing knowledge" in training dataset. 
The missing labels problem has also been explored in some closely related domains: recommendation \cite{saito2020unbiased,wang2019doubly,schnabel2016recommendations} and positive unlabeled learning (PUL) \cite{bekker2018learning,bekker2020learning,jaskie2019positive,kato2019learning}. More details about these works can be found in the Appendix \ref{sec:appendix-rworks}.

\textbf{Teacher Models and Data Augmentation:}
Another stream of literature aims to train performant XC or retrieval models by augmenting \cite{dai2022promptagatorfewshotdenseretrieval, saadfalcon2023udapdr, jeronymo2023inparsv2largelanguagemodels, bonifacio2022inparsdataaugmentationinformation,mittal2024graphregularizedencodertraining, mohan2024oak,qu2021rocketqa,aggarwal2023semsupxc} the training dataset with additional resources like teacher models, and query or document meta-data. LEVER \cite{buvanesh2024enhancing} and Gandalf \cite{kharbanda2023gandalf} attempts to impute the missing labels in the training dataset by using auxiliary models. However, their continued reliance on the biased training data prevents them from solving the missing knowledge problem.

\textbf{Large Language Models:} Recently, LLMs have been used for generating data for training task-specific models \cite{lee2023making}, newer SLMs \cite{gunasekar2023textbooksneedphi,maini2024rephrasing,hartvigsen2022toxigen,rosenbaum2022clasp,rosenbaum2022linguist}, and multimodal models \cite{li2024datacomp}. This approach to training models using LLM-generated datasets help in combating problems like data scarcity in low-resource settings, or lack of unbiased and clean datasets.
The strategies used for such synthetic dataset creation include quality filtering \cite{rae2021scaling,muennighoff2024scaling, fang2024data, abdin2024phi}, deduplication \cite{abbas2023semdedup}, data mixing \cite{xie2024doremi, albalak2023efficient, du2022glam}, synthetic data generation \cite{gunasekar2023textbooksneedphi, maini2024rephrasing}, data augmentation \cite{whitehouse2023llm}, or dataset transformation \cite{gandhi2024better}. There have also been some efforts to employ SLMs directly for retrieval, by finetuning them \cite{mohankumar2023unified,muennighoff2024generative,ma2024finetuningllama}, or rephrasing queries/documents \cite{wang2023query2doc}.


In contrast, we propose a way for \textit{completing} existing task specific dataset by addressing missing knowledge problems in the dataset. 
While existing work mainly focus on creating task-specific datasets from scratch or augmenting to extend the scale of these datasets, it is possible that the biases in the source dataset / LLM may just get scaled up in this synthetic data creation process. 
\begin{figure*}[h]
	\centering
\includegraphics[width=\linewidth]{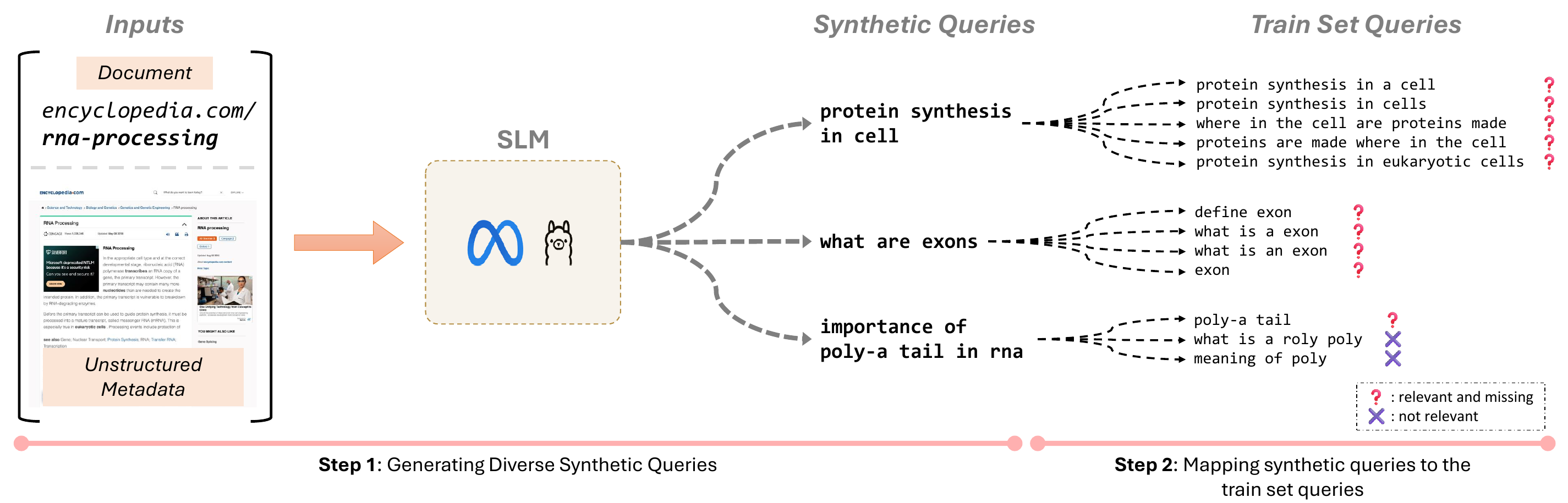}
		\caption{
  Steps of the \alg algorithm depicting how we bridge the missing knowledge in biased training datasets. In Step 1, for a URL (document) ``\textit{\href{https://www.encyclopedia.com/science-and-technology/biology-and-genetics/genetics-and-genetic-engineering/rna-processing}{encyclopedia.com/.../rna-processing}}'', the finetuned SLM generates diverse synthetic queries spanning concepts like \textit{protein synthesis}, \textit{exons}, \textit{poly-a tail} etc using the available unstructured meta-data (see Figure \ref{fig:slm_rephrasing_metadata_figure} in appendix). 
  In Step 2, a retriever is used to increase the coverage of the chosen document to relevant train queries through these synthetic queries, e.g. synthetic query ``\textit{what are exons}'' is mapped to similar train queries like ``\textit{define exon}'' and ``\textit{exon}'' which were missing for the document. 
  The retriever additionally filters out irrelevant train queries for the document using the similarity threshold $\tau$. 
  }. 
		\label{fig:alg_figure}
\end{figure*}

\section{Theory: Necessity of External Knowledge}
\label{sec:theory}

This section studies the phenomenon of {\it systematic} missing label bias in XC. It characterizes the nature, origin and properties of the systematic missing labels, and theoretically demonstrates 
that the classical debiasing techniques of XC such as propensity scoring and naive imputation are unable to mitigate such biases.

\subsection{Preliminaries}
\label{sec:preliminaries}

\header{Extreme Classification} Extreme Classification or XC involves learning a mathematical model $\func M$ that takes a query $\vec x \in \func X$ as input and predicts relevant subset of documents $\vec y \in \{0,1\}^L$ from a set of $L$ (a large number) documents with the $l$th document deemed relevant to $\vec x$ if $y_l=1$ and irrelevant otherwise. The model $\func M$ is trained by supervised learning on a large training set $\func D = \{\{\vec x_i,\vec y_i\}_{i=1}^{N},\{\vec z_l\}_{l=1}^L\}$ where $\vec x_i, \vec z_l$ are the textual features of $i$th query and $l$th document respectively and $\vec y_i$ are the ground truth documents for $\vec x_i$. Often, $\{\vec y_i\}_{i=1}^N$ contains many false negatives where a document $l$ is indeed relevant to query $i$ but is labelled as irrelevant ({\it i.e. $y_{il} = 0$}). This is the well-known problem of missing label bias in XC~\citep{jain2016extreme, schultheis2022missing}. Training $\func M$ on data with missing labels can lead to inferior model quality and yield inaccurate predictions when deployed in real-world applications. We revisit the missing label bias from the perspective of world knowledge.

\header{Training data collection} As the number of queries and documents in a typical XC application can range in millions or more,  human annotation of documents becomes infeasible. Consequently, most applications rely on implicit feedback from users to generate training data cost-effectively at scale. For example in a search application, when a user asks a query $\vec x_i$ sampled from a distribution $P_x = \op P(\vec x)$, a serving algorithm processes it and displays some results which are browsed by the user and clicked on in case they are found to be relevant to the query. User clicks are sparse; typically, a single document may be clicked for each user query. Let the probability that a document $\vec z$ gets clicked for a query $\vec x$ be denoted as $P_{z|x} = \op P_{\text{click}}(\vec z | \vec x)$. User clicks are not accurate and are subject to numerous biases such as exposure bias where, due to the inaccuracies of the serving algorithm, some relevant documents are not displayed for a query, and position bias where users click less on relevant but lower ranked documents, and so on. Such biases manifest as missing labels in the training data. For simplicity, we assume a uniform value of $1-B$ for exposure bias. Given this, the true probability that the document $\vec z$ is relevant to query $\vec x$  becomes $\op P(y_{xz}=1 | \vec x, \vec z) = P_{z|x}/B$. The data is aggregated over a specified time period, and the training dataset $\func D$ is constructed from the set of all user-entered queries, available documents and the user clicks from this period. 


\subsection{Systematic Missing Label Bias}

\header{World knowledge} In many XC tasks, the query and document texts, on their own, do not contain sufficient information to make relevance decisions. Therefore, users draw upon additional world knowledge for their clicks. 
Typically, the users derive such knowledge either from their personal world knowledge or from reading additional meta-data provided with the documents. 
An important property of world knowledge is that it is easily compressible across training samples. For example, in Figure~\ref{fig:alg_figure}, a knowledge statement that "\textit{rna} is related to \textit{protein-synthesis}" does not directly entail the statement "\textit{exon} is related to \textit{rna} or \textit{genes}". Due to this, world knowledge required for solving a large-scale XC task typically requires vast amounts of long-tail knowledge. For precise understanding, we formalize such intuitions about world knowledge into mathematical statements below.

We assume that the world knowledge required in an XC task is a large and discrete set of knowledge items $\func K  = \{k_m\}_{m=1}^K$. These items can correspond to statements such as those in the previous para, but the following analysis is independent of the design choice for $\func K$. The set $\func K$ is assumed to satisfy the following:

\begin{postulate} \label{postulate:externality}
\textbf{Externality of $\func K$}: The world knowledge is not derivable from and is independent of query and document texts: \\ $\{k_m\}_{m=1}^K \indep \{x_i\}_{i=1}^N, \{z_l\}_{l=1}^L$
\end{postulate}

\begin{postulate} \label{postulate:Incompressibility}
\textbf{Incompressibility of $\func K$}: The world knowledge items are not derivable from and are independent of each other: \\ $k_m \indep \func K \setminus k_m \;\;\; \forall m \in [1,\cdots, K]$
\end{postulate}

Typically, as in the showcased examples, users utilize only a small amount of knowledge for judging any given query-document pair. We formalize this as follows:

\begin{postulate} \label{postulate:sparsity}
\textbf{Sparsity of $\func K$}: The relevance of any query-document pair can be fully determined by using  exactly one item from the knowledge set: \\
$y_{il} = R(\vec x_i, \vec z_l, k_{\vec x_i, \vec z_l}) \in \{0,1\}$ where $R$ is the deterministic  true relevance function and $k_{\vec x_i, \vec z_l}$ is the knowledge item utilized by the pair. For notational convenience, let $D_m = \{(\vec x, \vec z) \in \func X \times \func Z : k_{\vec x,\vec z} = k_m\}$ denote the space of query-doc pairs which depend on the item $m$.
\end{postulate}

Note that the three postulates mentioned above do not hold fully in real applications. More realistic but complex settings can be explored in the future.

\header{Systematic missing labels} The necessity of world knowledge for judging relevance applies not only to the users but also to the serving algorithms. However, whereas humans can naturally store vast internal knowledge and possess reasoning ability to glean more knowledge from the meta-data, XC models do not inherently have these capabilities. As a result, XC models need to acquire such knowledge by  supervised training on the dataset $\func D$. For this, the training data needs to have an adequate coverage of $\func K$ among its clicked ({\it i.e.} positive/relevant) query-document pairs. Conversely, if some knowledge is completely missing in the training data, it cannot be learned by the supervised model, no matter how large the training set size is. This is proved by the following lemma:

\begin{lemma}
If $y_{il}=0 \eqnspace \forall (\vec x_i,\vec z_l) \in D_m$, then for any test pair $(\vec x,\vec z) \sim D_m$, $R(\vec x,\vec z,k_m) \indep \func D$ where $\func D = \{\{\vec x_i,\vec y_i\}_{i=1}^{N},\{\vec z_l\}_{l=1}^L\} $ is the training dataset.
\end{lemma}

We refer to such a missing label phenomenon, where all the relevant documents corresponding to some knowledge items go missing in the training ground truth, as the problem of \textbf{systematic missing labels}. Systematic missing labels lead to loss of crucial knowledge which cannot be re-constructed by any supervised model (including extreme classifiers) trained on $\func D$. As a result, such models will be unable to reliably predict any relevant query-document pair that depends on the missing knowledge $k_m$, as proven next:

\begin{corollary}
If $y_{il}=0 \eqnspace \forall (\vec x_i,\vec z_l) \in D_m$, then for any test pair $(\vec x,\vec z) \sim D_m$, $R(\vec x,\vec z,k_m) \indep \func M$; where $\func M $ is a predictive model trained deterministically from $\{\{\vec x_i,\vec y_i\}_{i=1}^{N},\{\vec z_l\}_{l=1}^L\}$.
\end{corollary}

When the missing $k_m$ belongs to the long-tail of $\func K$, the above result also holds true for propensity-scored approaches~\citep{jain2016extreme,qaraei2021convex}, as they typically train on just $\func D$ and, a small unbiased test set $\func D'$ which is unlikely to contain pairs from $k_m$. Similarly, this also applies to naive imputation techniques~\citep{kharbanda2023gandalf,buvanesh2024enhancing} which rely on either the ground truth or another supervised model to impute a target extreme classifier.

\subsection{Lower Bound on Model Performance}
While the previous part presents systematic missing labels as a theoretical construct, they are indeed real and frequent in many practical XC applications requiring vast and long-tail knowledge (for example see Figure \ref{fig:long-tail-in-xc-applications} in Appendix).
There are two key reasons behind this. First, long-tail labels cannot be consistently covered even with a large training set due to inherent stochasticities of the data sampling process. Second, user clicks are affected by exposure bias which further reduces the amount of clicks on the long-tail.

During click-based training set generation, the marginal probability of sampling a relevant pair $(\vec x,\vec z) \in D_m$ associated with a knowledge item $m$ is $p_m = \sum_{(\vec x, \vec z) \in D_m} P_x P_{z|x}$. Note that $\sum_{m=1}^K p_m = 1$. {\it W.l.o.g} assume that the indices $\{1,\cdots,K\}$ of knowledge items are sorted according to decreasing marginals $p_m$. Also, denote the survival function $\bar{F}_m = \sum_{m'=1}^m p_{m'}$. Now, a long-tailed distribution is typically expected to have a significant portion of aggregate density in the tail portion, {\it i.e.} for some large $m$, $\bar{F}_m$ is reasonably large even though $p_m$ is very small. The below theorem establishes a lower bound on the amount of systematic missing labels introduced in the training data as a function of the amount of knowledge tail and exposure bias.






\begin{theorem}
Let $n$ be the number of clicked query-doc pairs sampled for the training data, $B$ be the exposure bias while sampling, and $\bar{F}_m,p_m$ be defined as above. Then, for any $0 \le \delta \le 1$, with probability at least $1-\delta$, at least $max_{m=1}^K \frac{\bar{F}_m}{B} e^{-B. p_m. n} - \frac{1}{2 \sqrt{2n}} (\log{\frac{2K}{\delta}}^{\frac{3}{2}})$ fraction of the true relevant label distribution will correspond to the systematic missing label region and therefore will be irrecoverable by supervised training.
\end{theorem}

Note, from the above theorem, that (1) the error can be significant when $p_m$ is of the order $\frac{1}{n}$ which can happen when the knowledge grows with the training set size, (2) the error increases with the exposure bias $B$. Proof is in the Appendix.

Therefore, systematic missing labels are a serious concern because not only are they hard to mitigate with biased training data but they can also lead to vicious feedback loops in retrieval systems where errors pertaining to systematic missing labels are retained and amplified over a period of time.

\section{Method}
In the previous section, we established that the observed training datasets alone are insufficient for bridging the existing knowledge gaps, which are crucial for addressing systematic missing labels. 
To tackle this issue, we propose a novel algorithm in the following subsection 
which integrates external world knowledge beyond what is available in the training dataset.


\subsection{\alg: \underline{S}calable \underline{K}nowledge \underline{I}nfusion for \underline{M}issing Labels}

\textit{\textbf{Overview.}} With the recent advancements in LLMs where they can act as knowledge bases \cite{he2024can} and can accurately predict user search preferences \cite{thomas2024large}, they offer a natural solution to mitigate knowledge gaps, possibly by annotating every query-document pair in the dataset. 
However, this may take years and is prohibitively expensive, requiring $O(t_{LLM} \cdot N \cdot L)$ calls where both $N$ (number of queries) and $L$ (number of documents) is in millions in XC datasets.
Instead, we propose to use an alternative strategy to directly generate relevant synthetic queries for a document that brings down the cost to $O(t_{LLM} \cdot L)$, sidestepping annotating every query-document pair possible. However, this is still not scalable for XC datasets. One could possibly make this practical by using a SLM, that takes $O(t_{SLM} \cdot L)$ time, however, the generation quality might suffer due to limited parametric knowledge (Fig. \ref{fig:gpt4_pt_ft_llama} in Appendix).

\textbf{\textit{SKIM.}} Therefore, to tackle the issue of (i) limited knowledge and (ii) scalability at once, we propose \alg, that stands for \textbf{S}calable \textbf{K}nowledge \textbf{I}nfusion for
\textbf{M}issing Labels. 
SKIM utilizes a SLM for scalability, but incorporates abundant unstructured metadata text available in XC datasets to circumvent limited parametric knowledge.
At a high level, \alg can be described in the following steps:
(i) generating diverse synthetic queries that are representative of the missing knowledge for a document,
and (ii) systematically incorporating these synthetic queries into the training data by mapping them to multiple similar train set queries. 
Generating diverse synthetic queries in Step (i) ensures we directly add different relevant knowledge items that might be missing in the training dataset, 
while Step (ii) ensures that each synthetic query can result in multiple train query and document pairs for each knowledge item being introduced to reinforce it in the dataset.
This yields \alg both highly accurate and scalable. Additionally, since \alg adds new (train query, document) pairs directly in the training dataset, this makes it usable with any XC model, making it generally applicable.
First, we outline the core components of \alg that utilizes a LLM for clarity. 
In the subsequent section, we then discuss a scalable modification that replaces LLM with SLM designed to handle large-scale XC datasets (see Figure~\ref{fig:alg_figure}).

\textbf{Step 1. Generating diverse synthetic queries.}
We prompt a LLM to generate diverse synthetic queries given a document. 
Given that knowledge items are often incompressible, diversity in query generation is important to ensure coverage of various knowledge items.
Additionally, we choose to generate queries for a document since XC models consider every document as an independent classifier (that being said, this step can easily be reversed i.e. for every query, LLM can generate diverse documents).

Due to the extensive parametric world knowledge of LLMs, they are capable of generating queries that can address existing knowledge gaps in typical XC training datasets.
For instance, in the \orcas dataset, which involves matching user queries to relevant URLs, a LLM like GPT4 generates synthetic queries for a URL (document) such as \textit{\href{https://www.encyclopedia.com/science-and-technology/biology-and-genetics/genetics-and-genetic-engineering/rna-processing}{encyclopedia.org/.../rna-processing}} that span diverse knowledge items relating to \textit{genes}, \textit{exons} and, \textit{DNA}. 
One may choose to include available metadata with the document. 
%
We instruct the LLM using a well-crafted prompt that describes the task and provides high-quality in-context examples ensuring diversity and accuracy in synthetic query generation. 
For our implementation, we use GPT-4 \cite{achiam2023gpt} as the LLM. 
The prompts used are provided in the Appendix \ref{sec:gpt4-llama-prompts}. 
This step is repeated for every document, resulting in an asymptotic time complexity of $O(t_{LLM} \cdot L)$.
In practice, however, this process would be costly and cannot scale to XC datasets containing millions of documents as noted earlier. 
Therefore, we discuss a scalable version of Step 1 in the subsequent subsection \ref{sec:scaling_skim}, that scales to XC datasets and is accurate at the same time.

\textbf{Step 2. Mapping synthetic queries to train set queries.}
While the synthetic queries are relevant and diverse, they may not match the distribution of real users' queries. 
Assuming that the number of queries is vast, to enforce the true distribution over queries, we map them to real queries. 
Specifically, we use the generated synthetic queries from the previous step and map them onto the train queries, that is, for every synthetic query corresponding to a document, we take the nearest neighboring train queries in the train set and consider them as relevant train queries for this document. 
This step additionally helps increasing the coverage of each generated synthetic query. 
For instance, a synthetic query like "\textit{what are exons}"  generated in the previous step, after a nearest neighbor search in the train query space, yields neighbors like "\textit{define exon}", "\textit{what is an exon}", and "\textit{exon}", which are minor variations of the synthetic query. 
To ensure that only high-quality (train query, document) pairs are retained, we only consider neighboring queries for a synthetic query that have a cosine similarity greater than $\tau$, which is tuned as a hyper-parameter. 
For nearest neighbor search, we employ an Approximate Nearest Neighbor Search (ANNS) algorithm \cite{malkov2018efficienthnsw}, which can be run efficiently on CPUs and can scale to large datasets. 
We use a finetuned dual encoder \cite{dahiya2023ngame} on the given training dataset as the embedding space in which we calculate nearest neighbors. 
Once we have these new synthetic (train query, document) pairs, we combine them with the ones already in the training dataset to create a new training dataset and finally train any XC model on this.

The above two steps complete the basic components of our method. As noted previously, Step 1 is computationally expensive, whereas Step 2 can be performed efficiently. Thus, we now discuss how to scale Step 1 for practical XC workloads.

\subsection{Scaling synthetic query generation in \alg}
\label{sec:scaling_skim}
In order to scale synthetic query generation for large XC datasets, we employ SLMs for query generation.
Directly prompting a pretrained SLM results in inferior generation quality (see ablation in Appendix \ref{sec:appendix-ablations}, Figure \ref{fig:gpt4_pt_ft_llama}).
To maintain similar generation quality as LLMs, one could possibly try to distill/finetune SLM using an LLM for query generation task, but due to SLM's limited number of parameters, it may never be able to store the the diverse world knowledge across multiple domains.
We thus face two challenges: (i) a lack of parametric world knowledge compared to LLMs, and (ii) a pre-trained SLM is not be well-suited for the specific task at hand (see Fig. \ref{fig:gpt4_pt_ft_llama} in appendix). 
To address the first issue, we supplement the SLM with auxiliary data that can help in accurate and reliable query generation. 
Fortunately, for retrieval tasks, our key insight is that the documents have readily available unstructured metadata text associated with them, containing knowledge to recover the missing labels. 
For example, in \orcas where documents are URLs, we can use webpage text as associated metadata; or in query-ad keyword retrieval datasets, we can use the ad landing page for the keyword; or for Wikipedia titles, we can use their article content etc.
For the second issue, we distill the specific task-solving ability from the LLM into the SLM.
To be concrete, the specific task we want to distill to the SLM is: \textit{Given a document and unstructured metadata text about the document, generate diverse synthetic queries that can fill the knowledge gaps in the dataset.}

\textbf{\textit{Task-specific distillation.}} To collect distillation data for SLM, we instruct the large language model (LLM) using the same prompt developed in Step 1. Since the task will ultimately be performed by an SLM, we explicitly instruct the LLM to \textit{only} utilize the knowledge present in the metadata for synthetic query generation, and to \textit{not} use its own parametric world knowledge. This ensures that the SLM can solve this task using the available metadata only. Additionally, this inherently demands the LLM (and eventually the distilled SLM) to filter out non-relevant text from the unstructured metadata text (refer to Figure \ref{fig:slm_rephrasing_metadata_figure} in the appendix).

For curation of distillation data, we use GPT-4, collecting approximately 50K responses per dataset, generating $\sim$10 synthetic queries per example/document. 
We then distill this particular task to an SLM via vanilla supervised fine-tuning (SFT) using Low-Rank Adaptation (LoRA).
The SLM achieves close to the LLM generation quality after finetuning (see Fig. \ref{fig:gpt4_pt_ft_llama} in appendix).
Our experiments use two SLMs: Llama2-7B \cite{touvron2023llama} for the primary experiments, and Phi1.5-1.3B \cite{javaheripi2023phi} for ablation studies. 
Detailed information on SFT and LoRA hyper-parameters, distillation time, and hardware used is provided in Appendix \ref{sec:skim-impl-details}.



\textbf{\textit{Large-scale inference.}} Once we have the distilled SLM, we perform large-scale inference on all the documents (alongside their associated metadata) to generate synthetic queries. 
With these generated synthetic queries for each document, we can then proceed to the original Step 2 discussed in the previous section. Refer to Appendix \ref{sec:skim-impl-details} for details on SLM inference time and hardware. 
In this way, we are able to perform Step 1 in \alg accurately at scale. This reduces the generation time from several weeks/months to under a day for XC datasets, and at the same time, maintaining high-quality of generation similar to that of LLMs (see Fig. \ref{fig:gpt4_pt_ft_llama} in appendix).

\section{Experimental Results}



We now discuss the experimental setup, including datasets, evaluation and key observations. We also assess the real-world applicability of \alg through offline and online tests on a search engine. 

\subsection{Setup}
\textbf{Datasets}: 
We benchmark \alg on two public XC datasets, \orcas \cite{Dahiya23bDEXA, craswell2020orcas} and \wiki \cite{buvanesh2024enhancing}.

 \textit{\orcas}: This dataset involves mapping short user search queries to URLs of web pages that answer those queries. The training set is curated from click logs of the Bing Search engine \cite{craswell2020orcas}, thus making it prone to missing labels. To ensure unbiased evaluation, we build a test set using human-labelled queries  from the TREC-19 and TREC-20 Deep Learning competitions \cite{craswell2020overview}. Web-page text for a URL (document) is used as unstructured metadata.

 \textit{\wiki}: This dataset involves mapping Wiki-pedia titles to their categories. 
Since Wikipedia categories are human-annotated, the degree of missing labels is relatively less severe. 
Therefore, we assume this dataset to be relatively unbiased and simulate a controlled click bias (only in training set) using a pre-trained \texttt{msmarco} model \footnote{\href{https://huggingface.co/sentence-transformers/msmarco-distilbert-base-v4}{https://huggingface.co/sentence-transformers/msmarco-distilbert-base-v4}}. 
Specifically, only those documents that appear in the top-K ($K=200$) predictions of the pre-trained \texttt{msmarco} model are considered in the biased training set.
Note that our design choices are informed by production settings where only some documents predicted by some deployed model are shown to the user.
The article text corresponding to the Wikipedia Title (query) is used as metadata.
More details related to datasets and simulation are in Appendix ~\ref{sec:curations-stats}.

    


\textbf{Baselines}: We compare \alg against two classes of baselines: those that do not use external knowledge and those that do. Within baselines without access to external knowledge, we compare to the propensity-based approach: Inverse Propensity Scoring (IPS) \cite{jain2016extreme, qaraei2021convex}, and imputation methods like Gandalf \cite{kharbanda2023gandalf} and LEVER \cite{buvanesh2024enhancing}. To explore a combination of propensity and imputation methods, similar to doubly robust techniques \cite{li2023stabledr}, we compare against a baseline that combines IPS with LEVER. 
Finally, for methods having access to external knowledge, we evaluate \alg against a slightly modified version of UDAPDR \cite{saadfalcon2023udapdr}, an SLM-based rewriting method tailored for our use case. 
The above methods are applied on two state-of-the-art base XC models: one-vs-all extreme classifier \renee \cite{jain2023renee} and a dual encoder \dexml \cite{gupta2024dualencoders}.
More implementation details are in App.~\ref{sec:baseline-details}.



\textbf{Evaluation}: Since unbiased evaluation is a hard problem as pointed in \citet{schultheis2022missing}, we employ a mixture of evaluation methods to reliably judge model performance. 
Our main test sets correspond to an unbiased setting:  evaluation on human annotated test sets in case of \orcas and evaluation on the original test set in \wiki.
Since we focus on improving retrieval performance, we evaluate our method and baselines using $Recall@K (R@k)$  with $K=25,100$ i.e. higher values of K. 

\subsection{Main Results }
\label{sec:main-results}
Now we discuss main observations from our results in Table \ref{tab:main-table}. 




\textbf{\alg outperforms state-of-the-art XC methods.} On both public datasets, \alg, when used with any XC model, outperforms all baselines on unbiased test sets. In \orcas, where the training dataset exhibits real-world missing label bias, \alg  outperforms the closest baseline by an average of \textbf{4.68} absolute points in $R@100$ and \textbf{3.60} absolute points in $R@25$. Similarly, in \wiki, \alg outperforms the closest competitor by \textbf{8.27} absolute points in $R@100$.

\textbf{Importance of World Knowledge.} All baseline methods (except UDAPDR) do not incorporate world knowledge, and thus suffer from inferior performance, with their results being in similar ranges (they have $R@100$ around 40\% in \orcas).
Methods such as LEVER \cite{buvanesh2024enhancing} and Gandalf \cite{kharbanda2023gandalf} rely on biased training data itself for imputation, thereby lacking world knowledge, crucial for predicting missing knowledge items. 
While UDAPDR \cite{saadfalcon2023udapdr} attempts to incorporate external knowledge by using an SLM, the observed improvements are sub par due to its limited parametric knowledge.

\textbf{Incorrect Conclusions from Standard Offline Test Sets.}
Table \ref{tab:wiki-suppl-results} (found in appendix) reports the results on the biased test set on the dataset \wiki. We observe that on \dexml, methods that do not incorporate external knowledge have similar $R@100$ ($\sim$43) on the biased test set, while \alg has a $R@100$ of 37.64. However, on the unbiased test set, \alg outperforms the nearest baseline by $>8$ absolute points.
Similar observations hold true in case of baselines using \renee, where \renee + LEVER achieves the highest biased $R@100$, but has unbiased $R@100$ significantly inferior to \renee + \alg (more than 13 points).
This can lead to erroneous conclusions about the performance of a method. 
More results on the biased test sets are Tables \ref{tab:orcas-suppl-results} and \ref{tab:wiki-suppl-results} (Appendix).

\begin{table}[t]
\caption{Recall on unbiased test sets 
using two base XC models: \renee and \dexml. 
When compared to methods (Gandalf, LEVER, LEVER + IPS) that do not use any external information, \alg outperforms the closest baseline by 10.16 points in $R@100$. 
``-'' indicates the method is not supported. }
\centering
\footnotesize
\begin{tabular}{lcccccc}
\toprule
 & \multicolumn{3}{c}{\textbf{\orcas}} & \multicolumn{3}{c}{\textbf{LF-WikiHT-2M}} \\
\cmidrule(r){2-4} \cmidrule{5-7}
 \textbf{Method} & \textbf{R@10} & \textbf{R@25} & \textbf{R@100} & \textbf{R@10} &  \textbf{R@25} & \textbf{R@100} \\
\midrule
 \dexml & 21.77 & 28.32 & 39.08 &  8.27 & 10.87 & 15.25 \\
+ Gandalf  &22.39 & 29.87 & 40.52 &  8.23 & 10.73 & 15.01 \\
+ LEVER  &  21.70 & 29.26 & 40.45 &  7.57 & 9.89 & 13.78 \\
+ IPS   & 22.31 & 29.14 & 39.42 &  8.10 & 10.92 & 15.92 \\
+ LEVER + IPS  &  23.47 & 30.92 & 41.04 &  6.82 & 9.03 & 12.92 \\
+ \slmaug  & 26.42 & 33.71 & 47.70 &  6.83 & 10.27 & 17.15  \\
+ \alg (Ours)  & \textbf{26.99} & \textbf{36.90} & \textbf{49.60} &  \textbf{11.13} & \textbf{16.59} & \textbf{25.42} \\
\midrule
\renee   & 21.75 & 28.22 & 38.42 &  7.89 & 9.95 & 13.07 \\
+ Gandalf  & 20.89 & 27.58 & 34.66 &  8.46 & 11.08 & 15.30  \\
+ LEVER  &  22.58 & 28.40 & 39.52 &  8.59 & 11.15 & 15.98\\
+ IPS  & 19.67 & 24.38 & 31.08 &  9.12 & 12.01 & 16.75 \\
+ LEVER + IPS &  21.52 & 29.24 & 37.10 &  9.23 & 12.63 & 19.28 \\

+ \slmaug    & 24.71 & 34.02 & 44.93 &  - & - & - \\
+ \alg (Ours) & \textbf{27.20} & \textbf{38.04} & \textbf{52.39} &  \textbf{11.18} & \textbf{17.55} & \textbf{28.32}  \\

\bottomrule
\end{tabular}
\label{tab:main-table}
\end{table}

\subsection{Application to Sponsored Search}
\label{sec:sponsored_search_main}
We demonstrate the large-scale real-world applicability of \alg in sponsored search, where the task is to match user queries to a subset of relevant advertiser bid keywords, from potentially billions of keywords. 
This is a challenging problem which is further exacerbated by the systematic missing label bias in click logs, making it ideal for testing \alg.
To reliably evaluate \alg, we do offline evaluation using a proprietary filter model and conduct online A/B tests on live search engine traffic.


\textbf{Offline evaluation.} Offline experiments are conducted on the proprietary LF-Query2Keyword-10M dataset, curated from clicks on a popular search engine. This dataset contains around 140M user queries and 10M ad keywords.
The full landing page content corresponding to the ad keyword was treated as the unstructured metadata in \alg. 
We compare \renee + \alg with the deployed production model (Prod-Model), \renee \cite{jain2023renee} and \renee + LEVER \cite{buvanesh2024enhancing}.
We report $Hits@K$ where $K=50,100$, which is the number of pairs judged relevant among the top-K predictions by the proprietary filter model, averaged over all test queries.
\alg outperformed the closest baseline by $6$ points in $Hits@100$ and $3$ points in $Hits@50$. This goes on to show not just the effectiveness of \alg, but also its large-scale real-world applicability for industry applications. 
More details about the proprietary filter model can be found in the Appendix \ref{sec:sponsored_search_appendix}.


\begin{table}[tbp]
\centering
\caption{Offline evaluation results on proprietary query-keyword dataset LF-Query2Keyword-10M. \alg outperforms LEVER by more than $6$ absoulte points in $Hits@100$.}
\footnotesize
\begin{tabular}{lcc}
\toprule
Method & \textbf{Hits@50} & \textbf{Hits@100} \\
\midrule
Prod-Model & 27.50 & 46.10 \\
\renee & 31.80 & 52.40 \\
\renee + LEVER & 31.20 & 52.20 \\
\renee + \alg & 34.50 & 58.80 \\
\bottomrule
\end{tabular}
\label{tab:results}
\end{table}

\textbf{Online A/B test.} For live deployment, \alg was trained on a larger dataset similar to LF-Query2Keyword-10M with around 180M advertiser bid keywords and 170M user queries mined from click logs. 
A \renee + \alg model was deployed on live traffic on a popular search engine to conduct an A/B study. 
\alg increased the keyword density (number of predicted keywords which pass the various relevance and business filters in the sponsored search stack) by 16\% indicating its ability to bring in good diverse keywords previously missed by an ensemble of state-of-the art dense retrievers, generative models, and XC algorithms.
\alg improved the impression-yield (avg. number of ads shown per query) by 1.02\% and the click-yield (avg. number of ads clicked per query) by 1.23\%. All improvements are   statistically significant with p-value  < 0.001. 


\subsection{Ablations}
\label{sec:main-ablations}


We conduct ablation studies to justify the design choices in \alg. 
Specifically, we (i) demonstrate the importance of metadata in Step 1 of \alg, (ii) show that retrieval augmentation (RA) \cite{asai2024reliable} may be ineffective with biased training data, (iii) impact of SLM size on \alg, and (iv) ineffectiveness of propensity-based methods.  
Further ablations on the the number of effect of new synthetic training pairs generated by \alg, and the use of a pre-trained SLM for Step 1 of \alg are provided in Appendix \ref{sec:appendix-ablations}.

\textbf{Significance of unstructured metadata}: We generate synthetic documents/queries using metadata and without metadata, and then compare the performance of the downstream trained XC models (see Table \ref{tab:meta-vs-no-meta}). The results reveal a significant gap in performance between the two scenarios, underscoring the critical role of metadata in recovering missing knowledge, even in unstructured form. 
Interestingly, without metadata, we see a marginal gain (compared to vanilla training on biased data) owing to \textit{some} parametric world knowledge present in the SLM.

\begin{table}[tbp]
    \caption{Effect of providing meta-data to the SLM during the task-specific-distillation and large scale inference step. Providing meta-data compensates for lack of world-knowledge of the SLM and improves $R@100$ by 2.7 points on average.
    }
    \label{tab:meta-vs-no-meta}
    \centering
    \footnotesize
    \begin{tabular}{cccccc}
    \toprule
     &  & \multicolumn{2}{c}{\textbf{\orcas}} & \multicolumn{2}{c}{\textbf{LF-WikiHT-2M}} \\
    \cmidrule(r){3-4} \cmidrule(r){5-6}
    \textbf{Model}& \textbf{Metadata} & \textbf{R@25} & \textbf{R@100} & \textbf{R@25} & \textbf{R@100} \\
    \midrule
    \multirow{2}{*}{\dexml} & \xmark & 35.23 & 46.35 & 14.46 & 24.03 \\
     & \cmark & 36.90 & 49.60 & 16.60 & 25.40 \\
    \multirow{2}{*}{\renee} & \xmark & 34.54 & 48.80 & 13.58 & 24.29 \\
     & \cmark  & 38.00 & 52.40 & 17.60 & 28.30 \\
    \bottomrule
    \end{tabular}    
\end{table}

\textbf{Metadata-Augmented XC does not work}:
Recently, retrieval-augmentation (RA) works have improved performance in situations involving limited knowledge \cite{asai2024reliable}, that is, directly providing relevant knowledge as input to the retriever/generator improves performance.
In this ablation (Table \ref{tab:wiki_ra_doesnt_work}), we test what happens if one employs metadata augmentation, i.e., directly provide the relevant knowledge (metadata text) concatenated with the query text as input to the XC model.
In a way, we are placing an upper bound on RA since we always provide the relevant information as input alongside the query.
Surprisingly, we observe that metadata augmented (or retrieval-augmented) XC model underperforms.
This highlights an important point that one should de-bias the training dataset for RA to work. 
However, if we train using synthetic data generated by \alg, we see a significant improvement ($+17$ points) in $R@100$.
Interestingly, the XC model starts generalizing using the metadata provided as input and gives better performance as compared to the model trained using \alg but without metadata-augmentation (compare rows 2 and 4 in Table \ref{tab:wiki_ra_doesnt_work}).
This  shows that RA-like methods would be complementary to our approach and RA might fail if one does not de-bias the training dataset effectively.
For more details see Appendix \ref{sec:appendix-ablations}.

\begin{table}[tb]
\caption{Effect of Retrieval Aug. (RA) alongside \alg on unbiased performance on \wiki dataset. For each configuration, we train a \renee model. When using RA, metadata is passed as input alongside the query during both train and test time.}
\centering
\footnotesize
\begin{tabular}{cccc}

\toprule
\textbf{Retrieval Aug.} &\textbf{ \alg }&  \textbf{R@25} & \textbf{R@100} \\
\midrule
\xmark & \xmark & 10.00 & 13.10 \\
\xmark & \cmark & 17.60 & 28.30 \\
\cmark & \xmark & 15.57 & 20.60 \\
\cmark & \cmark & 27.72 & 45.98 \\
\bottomrule
\end{tabular}

\label{tab:wiki_ra_doesnt_work}
\end{table}

\textbf{Size of the SLM used in \alg}: Table \ref{tab:slm_performance}, shows the effect of the size of SLM used in Step 1 of \alg algorithm. We observe that even when the size of the SLM is decreased from 7B model to a 1.3B model the drop in $R@100$ is marginal drop ($<1\%$). This makes \alg extremely practical in limited compute.

\begin{table}[h!]
\caption{Effect of the size of the SLM used in \alg's synthetic query generation step (i.e. Step 1) on the downstream unbiased performance of \renee + \alg. We use \orcas dataset for this ablation. This is only a minor drop ($< 1\%$) when decreasing the size of the SLM from 7B to 1.3B parameters, making \alg extremely practical even in low training compute regimes.}
\centering
\footnotesize

\begin{tabular}{lcc}
\toprule
\textbf{SLM} & \textbf{R@25} & \textbf{R@100} \\
\midrule
Phi-1.5 (1.3B) & 37.28 & 51.66 \\
Llama2 (7B) & 38.00 & 52.4 \\
\bottomrule
\end{tabular}

\label{tab:slm_performance}
\end{table}

\textbf{Adding missing labels is more effective than training with accurate propensity estimates:}
Propensity-based methods are the most common way of addressing missing label bias in XC. 
To examine their limitations, we simulate a data collection process from real world applications where a deployed model retrieves a shortlist of documents and users click on them. 
We study the performance  of different models under varying degree  of missing labels, when we have access to accurate propensities.
This simulation is performed on \wiki dataset, where the degree of exposure of the deployed model is controlled by adjusting the top-K parameter. 
After top-K thresholding, a stochastic missing mechanism was applied to simulate user clicks, where a relevant query-document pair occurs (or clicked) in the dataset with probability $p_{ql} = \sigma(s(q,l))$, with $\sigma(.)$ as the sigmoid function and $s(q,l)$ as the score assigned by the pretrained \texttt{msmarco} model. 
Note that this simulation process is similar to click models considered in \citet{saito2020unbiased}, where the probability of a click is modeled as the product of exposure and relevance probabilities (MNAR setting).
Once the simulations are done, a dataset is created for every K. For every such dataset, we train \renee, \renee using IPS that make use of ground-truth $p_{ql}$ as propensities, and finally, \renee + \alg. 
Fig. \ref{fig:propensity_ablation} in the Appendix shows the effect of varying $K$ on different models trained on that particular dataset. 
As seen from the graph, reweighing the \renee loss function using $p(q, l)$ (ground-truth propensities) improves the $R@100$ of the \renee model by 2\%, but filling in the missing knowledge (or missing labels) using \alg further improves $R@100$ by 13 points when $K=200$. 
To further highlight the importance of having unbiased training data, we consider a training dataset based on randomly sampling an equivalent number of (query, documents) pairs from all the relevant pairs known in \wiki (MAR setting).
We see that even when 2.8\% of the total data is exposed using random sampling on relevant pairs, the $R@100$ is around 37, compared to 12 when biased data is used.

\begin{figure}[htbp]  
    \centering  
    \includegraphics[width=0.5\textwidth]{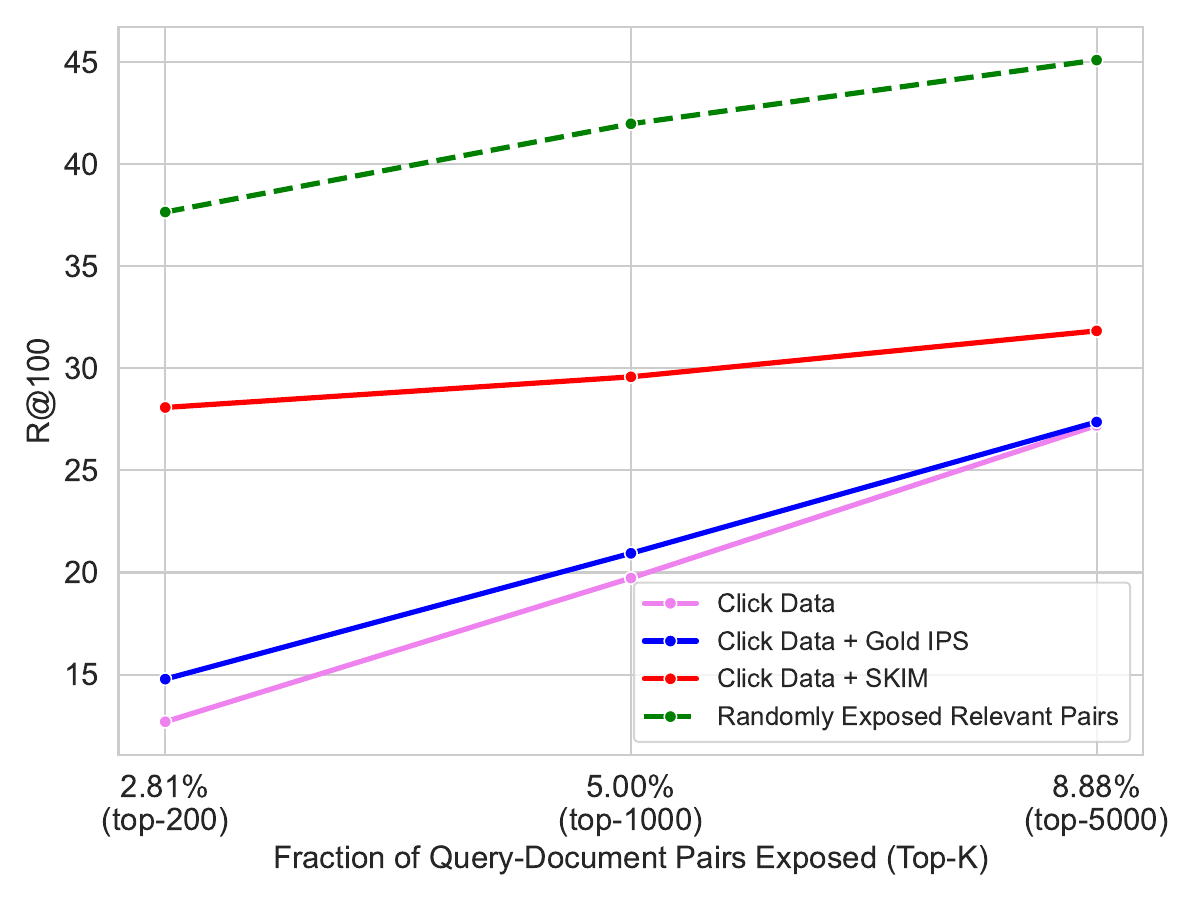}  
    \caption{X axis represents increasing fraction of relevant pairs or clicks, and in brackets, we show the top-K that was used in simulation to collect that fraction of clicks. We compare \renee models trained with (a) only click data (MNAR), (b) click data + using IPS with \textit{golden} propensities, (c) click data + \alg and finally, (d) click data created using sampling relevant pairs uniformly at random (MAR). All four settings are compared across different fraction of relevant pairs being exposed.}  
    \label{fig:propensity_ablation}  
\end{figure}

\section{Takeaways}
Our findings motivate the following observations for practitioners working with real-world XC problems or large-scale retrieval in general.

\textit{1. More data is not more knowledge.} In real world applications, more data is typically collected by increasing the time period in which clicks are collected using a deployed model.
While this may lead to more query-document pairs, this doesn't guarantee additional knowledge in the newly collected dataset since the predictions of the deployed model is limited by the knowledge present in its training dataset. Thus, this cycle continues without adding any new beneficial knowledge, and consequently does not add new document labels representative of missing knowledge. As noted earlier (see Table \ref{tab:main-table}, Figure \ref{fig:propensity_ablation} in the appendix), exposure to new knowledge is important to break this cycle, and \alg takes the first step in this direction.

\textit{2. Building better offline evaluation.} As discussed in Sec \ref{sec:main-results} relying on existing biased test sets for evaluation leads to misleading conclusions. \citet{schultheis2022missing} suggest the use of human labelled sets for reliable evaluation, but they are expensive to collect. However, with the advancements in LLMs, using them as reliable evaluation tools \cite{llms_test1, llms_test2, thomas2024large} can be a promising option. That being said, care should be taken while evaluating to not fit to the biases present in LLMs.



\section{Conclusion}


This paper studied a novel connection between the missing label issue in XC and the availability of world knowledge in the training data. The standard debiasing techniques such as propensity-scoring and imputation were found to be ineffective in mitigating knowledge-dependent systematic missing labels. To address this, it presented a novel approach, \alg, a debiasing algorithm that addresses the missing knowledge problem at scale by using an SLM coupled with meta-data. Experimental results on multiple XC tasks showed strong improvements in Recall metrics when \alg was applied to leading XC models. The real-world applicability of \alg was demonstrated using an online A/B test on a popular search engine.

\section{Acknowledgements}
We would like to thank Neelabh Madan for his useful suggestions and feedback during paper writing.
\clearpage
\bibliographystyle{acmrefstyle}
\balance
\bibliography{main}


\begin{thebibliography}{103}


\ifx \showCODEN    \undefined \def \showCODEN     #1{\unskip}     \fi
\ifx \showDOI      \undefined \def \showDOI       #1{#1}\fi
\ifx \showISBNx    \undefined \def \showISBNx     #1{\unskip}     \fi
\ifx \showISBNxiii \undefined \def \showISBNxiii  #1{\unskip}     \fi
\ifx \showISSN     \undefined \def \showISSN      #1{\unskip}     \fi
\ifx \showLCCN     \undefined \def \showLCCN      #1{\unskip}     \fi
\ifx \shownote     \undefined \def \shownote      #1{#1}          \fi
\ifx \showarticletitle \undefined \def \showarticletitle #1{#1}   \fi
\ifx \showURL      \undefined \def \showURL       {\relax}        \fi
\providecommand\bibfield[2]{#2}
\providecommand\bibinfo[2]{#2}
\providecommand\natexlab[1]{#1}
\providecommand\showeprint[2][]{arXiv:#2}

\bibitem[Abbas et~al\mbox{.}(2023)]%
        {abbas2023semdedup}
\bibfield{author}{\bibinfo{person}{Amro Abbas}, \bibinfo{person}{Kushal Tirumala}, \bibinfo{person}{D{\'a}niel Simig}, \bibinfo{person}{Surya Ganguli}, {and} \bibinfo{person}{Ari~S Morcos}.} \bibinfo{year}{2023}\natexlab{}.
\newblock \showarticletitle{Semdedup: Data-efficient learning at web-scale through semantic deduplication}.
\newblock \bibinfo{journal}{\emph{arXiv preprint arXiv:2303.09540}} (\bibinfo{year}{2023}).
\newblock


\bibitem[Abdin et~al\mbox{.}(2024)]%
        {abdin2024phi}
\bibfield{author}{\bibinfo{person}{Marah Abdin}, \bibinfo{person}{Sam~Ade Jacobs}, \bibinfo{person}{Ammar~Ahmad Awan}, \bibinfo{person}{Jyoti Aneja}, \bibinfo{person}{Ahmed Awadallah}, \bibinfo{person}{Hany Awadalla}, \bibinfo{person}{Nguyen Bach}, \bibinfo{person}{Amit Bahree}, \bibinfo{person}{Arash Bakhtiari}, \bibinfo{person}{Harkirat Behl}, {et~al\mbox{.}}} \bibinfo{year}{2024}\natexlab{}.
\newblock \showarticletitle{Phi-3 technical report: A highly capable language model locally on your phone}.
\newblock \bibinfo{journal}{\emph{arXiv preprint arXiv:2404.14219}} (\bibinfo{year}{2024}).
\newblock


\bibitem[Achiam et~al\mbox{.}(2023)]%
        {achiam2023gpt}
\bibfield{author}{\bibinfo{person}{Josh Achiam}, \bibinfo{person}{Steven Adler}, \bibinfo{person}{Sandhini Agarwal}, \bibinfo{person}{Lama Ahmad}, \bibinfo{person}{Ilge Akkaya}, \bibinfo{person}{Florencia~Leoni Aleman}, \bibinfo{person}{Diogo Almeida}, \bibinfo{person}{Janko Altenschmidt}, \bibinfo{person}{Sam Altman}, \bibinfo{person}{Shyamal Anadkat}, {et~al\mbox{.}}} \bibinfo{year}{2023}\natexlab{}.
\newblock \showarticletitle{Gpt-4 technical report}.
\newblock \bibinfo{journal}{\emph{arXiv preprint arXiv:2303.08774}} (\bibinfo{year}{2023}).
\newblock


\bibitem[Adler et~al\mbox{.}(2024)]%
        {adler2024nemotron}
\bibfield{author}{\bibinfo{person}{Bo Adler}, \bibinfo{person}{Niket Agarwal}, \bibinfo{person}{Ashwath Aithal}, \bibinfo{person}{Dong~H Anh}, \bibinfo{person}{Pallab Bhattacharya}, \bibinfo{person}{Annika Brundyn}, \bibinfo{person}{Jared Casper}, \bibinfo{person}{Bryan Catanzaro}, \bibinfo{person}{Sharon Clay}, \bibinfo{person}{Jonathan Cohen}, {et~al\mbox{.}}} \bibinfo{year}{2024}\natexlab{}.
\newblock \showarticletitle{Nemotron-4 340B Technical Report}.
\newblock \bibinfo{journal}{\emph{arXiv preprint arXiv:2406.11704}} (\bibinfo{year}{2024}).
\newblock


\bibitem[Aggarwal et~al\mbox{.}(2023)]%
        {aggarwal2023semsupxc}
\bibfield{author}{\bibinfo{person}{Pranjal Aggarwal}, \bibinfo{person}{Ameet Deshpande}, {and} \bibinfo{person}{Karthik Narasimhan}.} \bibinfo{year}{2023}\natexlab{}.
\newblock \showarticletitle{SemSup-XC: semantic supervision for zero and few-shot extreme classification}. In \bibinfo{booktitle}{\emph{Proceedings of the 40th International Conference on Machine Learning}} (Honolulu, Hawaii, USA) \emph{(\bibinfo{series}{ICML'23})}. \bibinfo{publisher}{JMLR.org}, Article \bibinfo{articleno}{11}, \bibinfo{numpages}{20}~pages.
\newblock


\bibitem[Agrawal et~al\mbox{.}(2013)]%
        {agrawal2013multi}
\bibfield{author}{\bibinfo{person}{Rahul Agrawal}, \bibinfo{person}{Archit Gupta}, \bibinfo{person}{Yashoteja Prabhu}, {and} \bibinfo{person}{Manik Varma}.} \bibinfo{year}{2013}\natexlab{}.
\newblock \showarticletitle{Multi-label learning with millions of labels: Recommending advertiser bid phrases for web pages}. In \bibinfo{booktitle}{\emph{Proceedings of the 22nd international conference on World Wide Web}}. \bibinfo{pages}{13--24}.
\newblock


\bibitem[Albalak et~al\mbox{.}(2023)]%
        {albalak2023efficient}
\bibfield{author}{\bibinfo{person}{Alon Albalak}, \bibinfo{person}{Liangming Pan}, \bibinfo{person}{Colin Raffel}, {and} \bibinfo{person}{William~Yang Wang}.} \bibinfo{year}{2023}\natexlab{}.
\newblock \showarticletitle{Efficient online data mixing for language model pre-training}. In \bibinfo{booktitle}{\emph{R0-FoMo: Robustness of Few-shot and Zero-shot Learning in Large Foundation Models}}.
\newblock


\bibitem[Asai et~al\mbox{.}(2024)]%
        {asai2024reliable}
\bibfield{author}{\bibinfo{person}{Akari Asai}, \bibinfo{person}{Zexuan Zhong}, \bibinfo{person}{Danqi Chen}, \bibinfo{person}{Pang~Wei Koh}, \bibinfo{person}{Luke Zettlemoyer}, \bibinfo{person}{Hannaneh Hajishirzi}, {and} \bibinfo{person}{Wen-tau Yih}.} \bibinfo{year}{2024}\natexlab{}.
\newblock \showarticletitle{Reliable, adaptable, and attributable language models with retrieval}.
\newblock \bibinfo{journal}{\emph{arXiv preprint arXiv:2403.03187}} (\bibinfo{year}{2024}).
\newblock


\bibitem[Babbar and Sch{\"o}lkopf(2017)]%
        {babbar2017dismec}
\bibfield{author}{\bibinfo{person}{Rohit Babbar} {and} \bibinfo{person}{Bernhard Sch{\"o}lkopf}.} \bibinfo{year}{2017}\natexlab{}.
\newblock \showarticletitle{Dismec: Distributed sparse machines for extreme multi-label classification}. In \bibinfo{booktitle}{\emph{Proceedings of the tenth ACM international conference on web search and data mining}}. \bibinfo{pages}{721--729}.
\newblock


\bibitem[Babbar and Sch{\"o}lkopf(2019)]%
        {babbar2019data}
\bibfield{author}{\bibinfo{person}{Rohit Babbar} {and} \bibinfo{person}{Bernhard Sch{\"o}lkopf}.} \bibinfo{year}{2019}\natexlab{}.
\newblock \showarticletitle{Data scarcity, robustness and extreme multi-label classification}.
\newblock \bibinfo{journal}{\emph{Machine Learning}} \bibinfo{volume}{108}, \bibinfo{number}{8} (\bibinfo{year}{2019}), \bibinfo{pages}{1329--1351}.
\newblock


\bibitem[Barezi et~al\mbox{.}(2019)]%
        {barezi2019submodular}
\bibfield{author}{\bibinfo{person}{Elham~J Barezi}, \bibinfo{person}{Ian~D Wood}, \bibinfo{person}{Pascale Fung}, {and} \bibinfo{person}{Hamid~R Rabiee}.} \bibinfo{year}{2019}\natexlab{}.
\newblock \showarticletitle{A submodular feature-aware framework for label subset selection in extreme classification problems}. In \bibinfo{booktitle}{\emph{Proceedings of the 2019 Conference of the North American Chapter of the Association for Computational Linguistics: Human Language Technologies, Volume 1 (Long and Short Papers)}}. \bibinfo{pages}{1009--1018}.
\newblock


\bibitem[Bekker and Davis(2018)]%
        {bekker2018learning}
\bibfield{author}{\bibinfo{person}{Jessa Bekker} {and} \bibinfo{person}{Jesse Davis}.} \bibinfo{year}{2018}\natexlab{}.
\newblock \showarticletitle{Learning from positive and unlabeled data under the selected at random assumption}. In \bibinfo{booktitle}{\emph{Second International Workshop on Learning with Imbalanced Domains: Theory and Applications}}. PMLR, \bibinfo{pages}{8--22}.
\newblock


\bibitem[Bekker and Davis(2020)]%
        {bekker2020learning}
\bibfield{author}{\bibinfo{person}{Jessa Bekker} {and} \bibinfo{person}{Jesse Davis}.} \bibinfo{year}{2020}\natexlab{}.
\newblock \showarticletitle{Learning from positive and unlabeled data: A survey}.
\newblock \bibinfo{journal}{\emph{Machine Learning}} \bibinfo{volume}{109}, \bibinfo{number}{4} (\bibinfo{year}{2020}), \bibinfo{pages}{719--760}.
\newblock


\bibitem[Bhatia et~al\mbox{.}(2016)]%
        {bhatia2016extreme}
\bibfield{author}{\bibinfo{person}{K Bhatia}, \bibinfo{person}{K Dahiya}, \bibinfo{person}{H Jain}, \bibinfo{person}{A Mittal}, \bibinfo{person}{Y Prabhu}, {and} \bibinfo{person}{M Varma}.} \bibinfo{year}{2016}\natexlab{}.
\newblock \showarticletitle{The extreme classification repository: Multi-label datasets and code. url: http://manikvarma. org/downloads/XC}.
\newblock \bibinfo{journal}{\emph{XMLRepository. html}}  \bibinfo{volume}{132} (\bibinfo{year}{2016}).
\newblock


\bibitem[Bhatia et~al\mbox{.}(2015)]%
        {bhatia2015sparse}
\bibfield{author}{\bibinfo{person}{Kush Bhatia}, \bibinfo{person}{Himanshu Jain}, \bibinfo{person}{Purushottam Kar}, \bibinfo{person}{Manik Varma}, {and} \bibinfo{person}{Prateek Jain}.} \bibinfo{year}{2015}\natexlab{}.
\newblock \showarticletitle{Sparse local embeddings for extreme multi-label classification}.
\newblock \bibinfo{journal}{\emph{Advances in neural information processing systems}}  \bibinfo{volume}{28} (\bibinfo{year}{2015}).
\newblock


\bibitem[Bonifacio et~al\mbox{.}(2022)]%
        {bonifacio2022inparsdataaugmentationinformation}
\bibfield{author}{\bibinfo{person}{Luiz Bonifacio}, \bibinfo{person}{Hugo Abonizio}, \bibinfo{person}{Marzieh Fadaee}, {and} \bibinfo{person}{Rodrigo Nogueira}.} \bibinfo{year}{2022}\natexlab{}.
\newblock \bibinfo{title}{InPars: Data Augmentation for Information Retrieval using Large Language Models}.
\newblock
\newblock
\showeprint[arxiv]{2202.05144}~[cs.CL]
\urldef\tempurl%
\url{https://arxiv.org/abs/2202.05144}
\showURL{%
\tempurl}


\bibitem[Buvanesh et~al\mbox{.}(2024)]%
        {buvanesh2024enhancing}
\bibfield{author}{\bibinfo{person}{Anirudh Buvanesh}, \bibinfo{person}{Rahul Chand}, \bibinfo{person}{Jatin Prakash}, \bibinfo{person}{Bhawna Paliwal}, \bibinfo{person}{Mudit Dhawan}, \bibinfo{person}{Neelabh Madan}, \bibinfo{person}{Deepesh Hada}, \bibinfo{person}{Vidit Jain}, \bibinfo{person}{Sonu Mehta}, \bibinfo{person}{Yashoteja Prabhu}, \bibinfo{person}{Manish Gupta}, \bibinfo{person}{Ramachandran Ramjee}, {and} \bibinfo{person}{Manik Varma}.} \bibinfo{year}{2024}\natexlab{}.
\newblock \showarticletitle{Enhancing Tail Performance in Extreme Classifiers by Label Variance Reduction}. In \bibinfo{booktitle}{\emph{The Twelfth International Conference on Learning Representations}}.
\newblock
\urldef\tempurl%
\url{https://openreview.net/forum?id=6ARlSgun7J}
\showURL{%
\tempurl}


\bibitem[Chen et~al\mbox{.}(2023)]%
        {chen2023bias}
\bibfield{author}{\bibinfo{person}{Jiawei Chen}, \bibinfo{person}{Hande Dong}, \bibinfo{person}{Xiang Wang}, \bibinfo{person}{Fuli Feng}, \bibinfo{person}{Meng Wang}, {and} \bibinfo{person}{Xiangnan He}.} \bibinfo{year}{2023}\natexlab{}.
\newblock \showarticletitle{Bias and debias in recommender system: A survey and future directions}.
\newblock \bibinfo{journal}{\emph{ACM Transactions on Information Systems}} \bibinfo{volume}{41}, \bibinfo{number}{3} (\bibinfo{year}{2023}), \bibinfo{pages}{1--39}.
\newblock


\bibitem[Collins et~al\mbox{.}(2018)]%
        {collins2018study}
\bibfield{author}{\bibinfo{person}{Andrew Collins}, \bibinfo{person}{Dominika Tkaczyk}, \bibinfo{person}{Akiko Aizawa}, {and} \bibinfo{person}{Joeran Beel}.} \bibinfo{year}{2018}\natexlab{}.
\newblock \showarticletitle{A study of position bias in digital library recommender systems}.
\newblock \bibinfo{journal}{\emph{arXiv preprint arXiv:1802.06565}} (\bibinfo{year}{2018}).
\newblock


\bibitem[Craswell et~al\mbox{.}(2020a)]%
        {craswell2020orcas}
\bibfield{author}{\bibinfo{person}{Nick Craswell}, \bibinfo{person}{Daniel Campos}, \bibinfo{person}{Bhaskar Mitra}, \bibinfo{person}{Emine Yilmaz}, {and} \bibinfo{person}{Bodo Billerbeck}.} \bibinfo{year}{2020}\natexlab{a}.
\newblock \showarticletitle{Orcas: 18 million clicked query-document pairs for analyzing search}. In \bibinfo{booktitle}{\emph{Proceedings of the 29th ACM International Conference on Information \& Knowledge Management}}. \bibinfo{pages}{2983--2989}.
\newblock


\bibitem[Craswell et~al\mbox{.}(2020b)]%
        {craswell2020overview}
\bibfield{author}{\bibinfo{person}{Nick Craswell}, \bibinfo{person}{Bhaskar Mitra}, \bibinfo{person}{Emine Yilmaz}, \bibinfo{person}{Daniel Campos}, {and} \bibinfo{person}{Ellen~M Voorhees}.} \bibinfo{year}{2020}\natexlab{b}.
\newblock \showarticletitle{Overview of the TREC 2019 deep learning track}.
\newblock \bibinfo{journal}{\emph{arXiv preprint arXiv:2003.07820}} (\bibinfo{year}{2020}).
\newblock


\bibitem[Dahiya et~al\mbox{.}(2021a)]%
        {dahiya2021siamesexml}
\bibfield{author}{\bibinfo{person}{Kunal Dahiya}, \bibinfo{person}{Ananye Agarwal}, \bibinfo{person}{Deepak Saini}, \bibinfo{person}{Gururaj K}, \bibinfo{person}{Jian Jiao}, \bibinfo{person}{Amit Singh}, \bibinfo{person}{Sumeet Agarwal}, \bibinfo{person}{Purushottam Kar}, {and} \bibinfo{person}{Manik Varma}.} \bibinfo{year}{2021}\natexlab{a}.
\newblock \showarticletitle{SiameseXML: Siamese Networks meet Extreme Classifiers with 100M Labels}. In \bibinfo{booktitle}{\emph{Proceedings of the 38th International Conference on Machine Learning}} \emph{(\bibinfo{series}{Proceedings of Machine Learning Research}, Vol.~\bibinfo{volume}{139})}, \bibfield{editor}{\bibinfo{person}{Marina Meila} {and} \bibinfo{person}{Tong Zhang}} (Eds.). \bibinfo{publisher}{PMLR}, \bibinfo{pages}{2330--2340}.
\newblock
\urldef\tempurl%
\url{https://proceedings.mlr.press/v139/dahiya21a.html}
\showURL{%
\tempurl}


\bibitem[Dahiya et~al\mbox{.}(2023a)]%
        {dahiya2023ngame}
\bibfield{author}{\bibinfo{person}{Kunal Dahiya}, \bibinfo{person}{Nilesh Gupta}, \bibinfo{person}{Deepak Saini}, \bibinfo{person}{Akshay Soni}, \bibinfo{person}{Yajun Wang}, \bibinfo{person}{Kushal Dave}, \bibinfo{person}{Jian Jiao}, \bibinfo{person}{Gururaj K}, \bibinfo{person}{Prasenjit Dey}, \bibinfo{person}{Amit Singh}, {et~al\mbox{.}}} \bibinfo{year}{2023}\natexlab{a}.
\newblock \showarticletitle{Ngame: Negative mining-aware mini-batching for extreme classification}. In \bibinfo{booktitle}{\emph{Proceedings of the Sixteenth ACM International Conference on Web Search and Data Mining}}. \bibinfo{pages}{258--266}.
\newblock


\bibitem[Dahiya et~al\mbox{.}(2021b)]%
        {dahiya2021deepxml}
\bibfield{author}{\bibinfo{person}{Kunal Dahiya}, \bibinfo{person}{Deepak Saini}, \bibinfo{person}{Anshul Mittal}, \bibinfo{person}{Ankush Shaw}, \bibinfo{person}{Kushal Dave}, \bibinfo{person}{Akshay Soni}, \bibinfo{person}{Himanshu Jain}, \bibinfo{person}{Sumeet Agarwal}, {and} \bibinfo{person}{Manik Varma}.} \bibinfo{year}{2021}\natexlab{b}.
\newblock \showarticletitle{Deepxml: A deep extreme multi-label learning framework applied to short text documents}. In \bibinfo{booktitle}{\emph{Proceedings of the 14th ACM international conference on web search and data mining}}. \bibinfo{pages}{31--39}.
\newblock


\bibitem[Dahiya et~al\mbox{.}(2023b)]%
        {dahiya2023deep}
\bibfield{author}{\bibinfo{person}{Kunal Dahiya}, \bibinfo{person}{Sachin Yadav}, \bibinfo{person}{Sushant Sondhi}, \bibinfo{person}{Deepak Saini}, \bibinfo{person}{Sonu Mehta}, \bibinfo{person}{Jian Jiao}, \bibinfo{person}{Sumeet Agarwal}, \bibinfo{person}{Purushottam Kar}, {and} \bibinfo{person}{Manik Varma}.} \bibinfo{year}{2023}\natexlab{b}.
\newblock \showarticletitle{Deep encoders with auxiliary parameters for extreme classification}. In \bibinfo{booktitle}{\emph{In Proceedings of the ACM SIGKDD Conference on Knowledge Discovery and Data Mining, Long Beach, California}}.
\newblock
\urldef\tempurl%
\url{https://www.microsoft.com/en-us/research/publication/deep-encoders-with-auxiliary-parameters-for-extreme-classification/}
\showURL{%
\tempurl}


\bibitem[Dahiya et~al\mbox{.}(2023c)]%
        {Dahiya23bDEXA}
\bibfield{author}{\bibinfo{person}{K. Dahiya}, \bibinfo{person}{S. Yadav}, \bibinfo{person}{S. Sondhi}, \bibinfo{person}{D. Saini}, \bibinfo{person}{S. Mehta}, \bibinfo{person}{J. Jiao}, \bibinfo{person}{S. Agarwal}, \bibinfo{person}{P. Kar}, {and} \bibinfo{person}{M. Varma}.} \bibinfo{year}{2023}\natexlab{c}.
\newblock \showarticletitle{Deep encoders with auxiliary parameters for extreme classification}. In \bibinfo{booktitle}{\emph{KDD}}.
\newblock


\bibitem[Dai et~al\mbox{.}(2022)]%
        {dai2022promptagatorfewshotdenseretrieval}
\bibfield{author}{\bibinfo{person}{Zhuyun Dai}, \bibinfo{person}{Vincent~Y. Zhao}, \bibinfo{person}{Ji Ma}, \bibinfo{person}{Yi Luan}, \bibinfo{person}{Jianmo Ni}, \bibinfo{person}{Jing Lu}, \bibinfo{person}{Anton Bakalov}, \bibinfo{person}{Kelvin Guu}, \bibinfo{person}{Keith~B. Hall}, {and} \bibinfo{person}{Ming-Wei Chang}.} \bibinfo{year}{2022}\natexlab{}.
\newblock \bibinfo{title}{Promptagator: Few-shot Dense Retrieval From 8 Examples}.
\newblock
\newblock
\showeprint[arxiv]{2209.11755}~[cs.CL]
\urldef\tempurl%
\url{https://arxiv.org/abs/2209.11755}
\showURL{%
\tempurl}


\bibitem[Du et~al\mbox{.}(2022)]%
        {du2022glam}
\bibfield{author}{\bibinfo{person}{Nan Du}, \bibinfo{person}{Yanping Huang}, \bibinfo{person}{Andrew~M Dai}, \bibinfo{person}{Simon Tong}, \bibinfo{person}{Dmitry Lepikhin}, \bibinfo{person}{Yuanzhong Xu}, \bibinfo{person}{Maxim Krikun}, \bibinfo{person}{Yanqi Zhou}, \bibinfo{person}{Adams~Wei Yu}, \bibinfo{person}{Orhan Firat}, {et~al\mbox{.}}} \bibinfo{year}{2022}\natexlab{}.
\newblock \showarticletitle{Glam: Efficient scaling of language models with mixture-of-experts}. In \bibinfo{booktitle}{\emph{International Conference on Machine Learning}}. PMLR, \bibinfo{pages}{5547--5569}.
\newblock


\bibitem[Fang et~al\mbox{.}(2024)]%
        {fang2024data}
\bibfield{author}{\bibinfo{person}{Alex Fang}, \bibinfo{person}{Albin~Madappally Jose}, \bibinfo{person}{Amit Jain}, \bibinfo{person}{Ludwig Schmidt}, \bibinfo{person}{Alexander~T Toshev}, {and} \bibinfo{person}{Vaishaal Shankar}.} \bibinfo{year}{2024}\natexlab{}.
\newblock \showarticletitle{Data Filtering Networks}. In \bibinfo{booktitle}{\emph{NeurIPS 2023 Workshop on Distribution Shifts: New Frontiers with Foundation Models}}.
\newblock
\urldef\tempurl%
\url{https://openreview.net/forum?id=ZKtZ7KQ6G5}
\showURL{%
\tempurl}


\bibitem[Gandhi et~al\mbox{.}(2024)]%
        {gandhi2024better}
\bibfield{author}{\bibinfo{person}{Saumya Gandhi}, \bibinfo{person}{Ritu Gala}, \bibinfo{person}{Vijay Viswanathan}, \bibinfo{person}{Tongshuang Wu}, {and} \bibinfo{person}{Graham Neubig}.} \bibinfo{year}{2024}\natexlab{}.
\newblock \showarticletitle{Better Synthetic Data by Retrieving and Transforming Existing Datasets}.
\newblock \bibinfo{journal}{\emph{arXiv preprint arXiv:2404.14361}} (\bibinfo{year}{2024}).
\newblock


\bibitem[Gunasekar et~al\mbox{.}(2023)]%
        {gunasekar2023textbooksneedphi}
\bibfield{author}{\bibinfo{person}{Suriya Gunasekar}, \bibinfo{person}{Yi Zhang}, \bibinfo{person}{Jyoti Aneja}, \bibinfo{person}{Caio César~Teodoro Mendes}, \bibinfo{person}{Allie~Del Giorno}, \bibinfo{person}{Sivakanth Gopi}, \bibinfo{person}{Mojan Javaheripi}, \bibinfo{person}{Piero Kauffmann}, \bibinfo{person}{Gustavo de Rosa}, \bibinfo{person}{Olli Saarikivi}, \bibinfo{person}{Adil Salim}, \bibinfo{person}{Shital Shah}, \bibinfo{person}{Harkirat~Singh Behl}, \bibinfo{person}{Xin Wang}, \bibinfo{person}{Sébastien Bubeck}, \bibinfo{person}{Ronen Eldan}, \bibinfo{person}{Adam~Tauman Kalai}, \bibinfo{person}{Yin~Tat Lee}, {and} \bibinfo{person}{Yuanzhi Li}.} \bibinfo{year}{2023}\natexlab{}.
\newblock \bibinfo{title}{Textbooks Are All You Need}.
\newblock
\newblock
\showeprint[arxiv]{2306.11644}~[cs.CL]
\urldef\tempurl%
\url{https://arxiv.org/abs/2306.11644}
\showURL{%
\tempurl}


\bibitem[Guo et~al\mbox{.}(2019)]%
        {guo2019breaking}
\bibfield{author}{\bibinfo{person}{Chuan Guo}, \bibinfo{person}{Ali Mousavi}, \bibinfo{person}{Xiang Wu}, \bibinfo{person}{Daniel~N Holtmann-Rice}, \bibinfo{person}{Satyen Kale}, \bibinfo{person}{Sashank Reddi}, {and} \bibinfo{person}{Sanjiv Kumar}.} \bibinfo{year}{2019}\natexlab{}.
\newblock \showarticletitle{Breaking the glass ceiling for embedding-based classifiers for large output spaces}.
\newblock \bibinfo{journal}{\emph{Advances in Neural Information Processing Systems}}  \bibinfo{volume}{32} (\bibinfo{year}{2019}).
\newblock


\bibitem[Gupta et~al\mbox{.}(2022)]%
        {gupta2022elias}
\bibfield{author}{\bibinfo{person}{Nilesh Gupta}, \bibinfo{person}{Patrick Chen}, \bibinfo{person}{Hsiang-Fu Yu}, \bibinfo{person}{Cho-Jui Hsieh}, {and} \bibinfo{person}{Inderjit Dhillon}.} \bibinfo{year}{2022}\natexlab{}.
\newblock \showarticletitle{ELIAS: End-to-End Learning to Index and Search in Large Output Spaces}.
\newblock \bibinfo{journal}{\emph{Advances in Neural Information Processing Systems}}  \bibinfo{volume}{35} (\bibinfo{year}{2022}), \bibinfo{pages}{19798--19809}.
\newblock


\bibitem[Gupta et~al\mbox{.}(2024a)]%
        {gupta2024dualencoders}
\bibfield{author}{\bibinfo{person}{Nilesh Gupta}, \bibinfo{person}{Fnu Devvrit}, \bibinfo{person}{Ankit~Singh Rawat}, \bibinfo{person}{Srinadh Bhojanapalli}, \bibinfo{person}{Prateek Jain}, {and} \bibinfo{person}{Inderjit~S Dhillon}.} \bibinfo{year}{2024}\natexlab{a}.
\newblock \showarticletitle{Dual-Encoders for Extreme Multi-label Classification}. In \bibinfo{booktitle}{\emph{The Twelfth International Conference on Learning Representations}}.
\newblock
\urldef\tempurl%
\url{https://openreview.net/forum?id=dNe1T0Ahby}
\showURL{%
\tempurl}


\bibitem[Gupta et~al\mbox{.}(2024b)]%
        {gupta2024efficacy}
\bibfield{author}{\bibinfo{person}{Nilesh Gupta}, \bibinfo{person}{Fnu Devvrit}, \bibinfo{person}{Ankit~Singh Rawat}, \bibinfo{person}{Srinadh Bhojanapalli}, \bibinfo{person}{Prateek Jain}, {and} \bibinfo{person}{Inderjit~S Dhillon}.} \bibinfo{year}{2024}\natexlab{b}.
\newblock \showarticletitle{Efficacy of Dual-Encoders for Extreme Multi-label Classification}. In \bibinfo{booktitle}{\emph{The Twelfth International Conference on Learning Representations}}.
\newblock
\urldef\tempurl%
\url{https://openreview.net/forum?id=dNe1T0Ahby}
\showURL{%
\tempurl}


\bibitem[Hartvigsen et~al\mbox{.}(2022)]%
        {hartvigsen2022toxigen}
\bibfield{author}{\bibinfo{person}{Thomas Hartvigsen}, \bibinfo{person}{Saadia Gabriel}, \bibinfo{person}{Hamid Palangi}, \bibinfo{person}{Maarten Sap}, \bibinfo{person}{Dipankar Ray}, {and} \bibinfo{person}{Ece Kamar}.} \bibinfo{year}{2022}\natexlab{}.
\newblock \showarticletitle{Toxigen: A large-scale machine-generated dataset for adversarial and implicit hate speech detection}.
\newblock \bibinfo{journal}{\emph{arXiv preprint arXiv:2203.09509}} (\bibinfo{year}{2022}).
\newblock


\bibitem[He et~al\mbox{.}(2024)]%
        {he2024can}
\bibfield{author}{\bibinfo{person}{Qiyuan He}, \bibinfo{person}{Yizhong Wang}, {and} \bibinfo{person}{Wenya Wang}.} \bibinfo{year}{2024}\natexlab{}.
\newblock \showarticletitle{Can Language Models Act as Knowledge Bases at Scale?}
\newblock \bibinfo{journal}{\emph{arXiv preprint arXiv:2402.14273}} (\bibinfo{year}{2024}).
\newblock


\bibitem[Jain et~al\mbox{.}(2019)]%
        {jain2019slice}
\bibfield{author}{\bibinfo{person}{Himanshu Jain}, \bibinfo{person}{Venkatesh Balasubramanian}, \bibinfo{person}{Bhanu Chunduri}, {and} \bibinfo{person}{Manik Varma}.} \bibinfo{year}{2019}\natexlab{}.
\newblock \showarticletitle{Slice: Scalable linear extreme classifiers trained on 100 million labels for related searches}. In \bibinfo{booktitle}{\emph{Proceedings of the twelfth ACM international conference on web search and data mining}}. \bibinfo{pages}{528--536}.
\newblock


\bibitem[Jain et~al\mbox{.}(2016)]%
        {jain2016extreme}
\bibfield{author}{\bibinfo{person}{Himanshu Jain}, \bibinfo{person}{Yashoteja Prabhu}, {and} \bibinfo{person}{Manik Varma}.} \bibinfo{year}{2016}\natexlab{}.
\newblock \showarticletitle{Extreme multi-label loss functions for recommendation, tagging, ranking \& other missing label applications}. In \bibinfo{booktitle}{\emph{Proceedings of the 22nd ACM SIGKDD international conference on knowledge discovery and data mining}}. \bibinfo{pages}{935--944}.
\newblock


\bibitem[Jain et~al\mbox{.}(2023)]%
        {jain2023renee}
\bibfield{author}{\bibinfo{person}{Vidit Jain}, \bibinfo{person}{Jatin Prakash}, \bibinfo{person}{Deepak Saini}, \bibinfo{person}{Jian Jiao}, \bibinfo{person}{Ramachandran Ramjee}, {and} \bibinfo{person}{Manik Varma}.} \bibinfo{year}{2023}\natexlab{}.
\newblock \showarticletitle{Renee: End-to-end training of extreme classification models}. In \bibinfo{booktitle}{\emph{Proceedings of Machine Learning and Systems}}. \bibinfo{pages}{To appear}.
\newblock


\bibitem[Jaskie and Spanias(2019)]%
        {jaskie2019positive}
\bibfield{author}{\bibinfo{person}{Kristen Jaskie} {and} \bibinfo{person}{Andreas Spanias}.} \bibinfo{year}{2019}\natexlab{}.
\newblock \showarticletitle{Positive and unlabeled learning algorithms and applications: A survey}. In \bibinfo{booktitle}{\emph{2019 10th International Conference on Information, Intelligence, Systems and Applications (IISA)}}. IEEE, \bibinfo{pages}{1--8}.
\newblock


\bibitem[Javaheripi et~al\mbox{.}(2023)]%
        {javaheripi2023phi}
\bibfield{author}{\bibinfo{person}{Mojan Javaheripi}, \bibinfo{person}{S{\'e}bastien Bubeck}, \bibinfo{person}{Marah Abdin}, \bibinfo{person}{Jyoti Aneja}, \bibinfo{person}{Sebastien Bubeck}, \bibinfo{person}{Caio C{\'e}sar~Teodoro Mendes}, \bibinfo{person}{Weizhu Chen}, \bibinfo{person}{Allie Del~Giorno}, \bibinfo{person}{Ronen Eldan}, \bibinfo{person}{Sivakanth Gopi}, {et~al\mbox{.}}} \bibinfo{year}{2023}\natexlab{}.
\newblock \showarticletitle{Phi-2: The surprising power of small language models}.
\newblock \bibinfo{journal}{\emph{Microsoft Research Blog}} (\bibinfo{year}{2023}).
\newblock


\bibitem[Jeronymo et~al\mbox{.}(2023)]%
        {jeronymo2023inparsv2largelanguagemodels}
\bibfield{author}{\bibinfo{person}{Vitor Jeronymo}, \bibinfo{person}{Luiz Bonifacio}, \bibinfo{person}{Hugo Abonizio}, \bibinfo{person}{Marzieh Fadaee}, \bibinfo{person}{Roberto Lotufo}, \bibinfo{person}{Jakub Zavrel}, {and} \bibinfo{person}{Rodrigo Nogueira}.} \bibinfo{year}{2023}\natexlab{}.
\newblock \bibinfo{title}{InPars-v2: Large Language Models as Efficient Dataset Generators for Information Retrieval}.
\newblock
\newblock
\showeprint[arxiv]{2301.01820}~[cs.IR]
\urldef\tempurl%
\url{https://arxiv.org/abs/2301.01820}
\showURL{%
\tempurl}


\bibitem[Jiang et~al\mbox{.}(2021)]%
        {jiang2021lightxml}
\bibfield{author}{\bibinfo{person}{Ting Jiang}, \bibinfo{person}{Deqing Wang}, \bibinfo{person}{Leilei Sun}, \bibinfo{person}{Huayi Yang}, \bibinfo{person}{Zhengyang Zhao}, {and} \bibinfo{person}{Fuzhen Zhuang}.} \bibinfo{year}{2021}\natexlab{}.
\newblock \showarticletitle{Lightxml: Transformer with dynamic negative sampling for high-performance extreme multi-label text classification}. In \bibinfo{booktitle}{\emph{Proceedings of the AAAI conference on artificial intelligence}}, Vol.~\bibinfo{volume}{35}. \bibinfo{pages}{7987--7994}.
\newblock


\bibitem[Joachims et~al\mbox{.}(2017)]%
        {joachims2017unbiased}
\bibfield{author}{\bibinfo{person}{Thorsten Joachims}, \bibinfo{person}{Adith Swaminathan}, {and} \bibinfo{person}{Tobias Schnabel}.} \bibinfo{year}{2017}\natexlab{}.
\newblock \showarticletitle{Unbiased learning-to-rank with biased feedback}. In \bibinfo{booktitle}{\emph{Proceedings of the tenth ACM international conference on web search and data mining}}. \bibinfo{pages}{781--789}.
\newblock


\bibitem[Kandpal et~al\mbox{.}(2023)]%
        {kandpal2023large}
\bibfield{author}{\bibinfo{person}{Nikhil Kandpal}, \bibinfo{person}{Haikang Deng}, \bibinfo{person}{Adam Roberts}, \bibinfo{person}{Eric Wallace}, {and} \bibinfo{person}{Colin Raffel}.} \bibinfo{year}{2023}\natexlab{}.
\newblock \showarticletitle{Large language models struggle to learn long-tail knowledge}. In \bibinfo{booktitle}{\emph{International Conference on Machine Learning}}. PMLR, \bibinfo{pages}{15696--15707}.
\newblock


\bibitem[Kato et~al\mbox{.}(2019)]%
        {kato2019learning}
\bibfield{author}{\bibinfo{person}{Masahiro Kato}, \bibinfo{person}{Takeshi Teshima}, {and} \bibinfo{person}{Junya Honda}.} \bibinfo{year}{2019}\natexlab{}.
\newblock \showarticletitle{Learning from positive and unlabeled data with a selection bias}. In \bibinfo{booktitle}{\emph{International conference on learning representations}}.
\newblock


\bibitem[Khandagale et~al\mbox{.}(2020)]%
        {khandagale2020bonsai}
\bibfield{author}{\bibinfo{person}{Sujay Khandagale}, \bibinfo{person}{Han Xiao}, {and} \bibinfo{person}{Rohit Babbar}.} \bibinfo{year}{2020}\natexlab{}.
\newblock \showarticletitle{Bonsai: diverse and shallow trees for extreme multi-label classification}.
\newblock \bibinfo{journal}{\emph{Machine Learning}} \bibinfo{volume}{109}, \bibinfo{number}{11} (\bibinfo{year}{2020}), \bibinfo{pages}{2099--2119}.
\newblock


\bibitem[Kharbanda et~al\mbox{.}(2022)]%
        {kharbanda2022cascadexml}
\bibfield{author}{\bibinfo{person}{Siddhant Kharbanda}, \bibinfo{person}{Atmadeep Banerjee}, \bibinfo{person}{Erik Schultheis}, {and} \bibinfo{person}{Rohit Babbar}.} \bibinfo{year}{2022}\natexlab{}.
\newblock \showarticletitle{CascadeXML: Rethinking Transformers for End-to-end Multi-resolution Training in Extreme Multi-label Classification}. In \bibinfo{booktitle}{\emph{Advances in Neural Information Processing Systems}}, \bibfield{editor}{\bibinfo{person}{S.~Koyejo}, \bibinfo{person}{S.~Mohamed}, \bibinfo{person}{A.~Agarwal}, \bibinfo{person}{D.~Belgrave}, \bibinfo{person}{K.~Cho}, {and} \bibinfo{person}{A.~Oh}} (Eds.), Vol.~\bibinfo{volume}{35}. \bibinfo{publisher}{Curran Associates, Inc.}, \bibinfo{pages}{2074--2087}.
\newblock
\urldef\tempurl%
\url{https://proceedings.neurips.cc/paper_files/paper/2022/file/0e0157ce5ea15831072be4744cbd5334-Paper-Conference.pdf}
\showURL{%
\tempurl}


\bibitem[Kharbanda et~al\mbox{.}(2024)]%
        {kharbanda2023gandalf}
\bibfield{author}{\bibinfo{person}{Siddhant Kharbanda}, \bibinfo{person}{Devaansh Gupta}, \bibinfo{person}{Erik Schultheis}, \bibinfo{person}{Atmadeep Banerjee}, \bibinfo{person}{Vikas Verma}, {and} \bibinfo{person}{Rohit Babbar}.} \bibinfo{year}{2024}\natexlab{}.
\newblock \bibinfo{title}{Gandalf: Learning label correlations in Extreme Multi-label Classification via Label Features}.
\newblock
\newblock
\urldef\tempurl%
\url{https://openreview.net/forum?id=JuyFppXzh2}
\showURL{%
\tempurl}


\bibitem[Kweon and Yu(2024)]%
        {Kweon2024doublycaliberated}
\bibfield{author}{\bibinfo{person}{Wonbin Kweon} {and} \bibinfo{person}{Hwanjo Yu}.} \bibinfo{year}{2024}\natexlab{}.
\newblock \showarticletitle{Doubly Calibrated Estimator for Recommendation on Data Missing Not at Random}. In \bibinfo{booktitle}{\emph{Proceedings of the ACM on Web Conference 2024}} (Singapore, Singapore) \emph{(\bibinfo{series}{WWW '24})}. \bibinfo{publisher}{Association for Computing Machinery}, \bibinfo{address}{New York, NY, USA}, \bibinfo{pages}{3810–3820}.
\newblock
\showISBNx{9798400701719}
\urldef\tempurl%
\url{https://doi.org/10.1145/3589334.3645617}
\showDOI{\tempurl}


\bibitem[Lee et~al\mbox{.}(2023b)]%
        {lee2023making}
\bibfield{author}{\bibinfo{person}{Dong-Ho Lee}, \bibinfo{person}{Jay Pujara}, \bibinfo{person}{Mohit Sewak}, \bibinfo{person}{Ryen White}, {and} \bibinfo{person}{Sujay Jauhar}.} \bibinfo{year}{2023}\natexlab{b}.
\newblock \showarticletitle{Making Large Language Models Better Data Creators}. In \bibinfo{booktitle}{\emph{Proceedings of the 2023 Conference on Empirical Methods in Natural Language Processing}}. \bibinfo{pages}{15349--15360}.
\newblock


\bibitem[Lee et~al\mbox{.}(2023a)]%
        {lee2023uctrl}
\bibfield{author}{\bibinfo{person}{Jae-woong Lee}, \bibinfo{person}{Seongmin Park}, \bibinfo{person}{Mincheol Yoon}, {and} \bibinfo{person}{Jongwuk Lee}.} \bibinfo{year}{2023}\natexlab{a}.
\newblock \showarticletitle{uCTRL: Unbiased Contrastive Representation Learning via Alignment and Uniformity for Collaborative Filtering}. In \bibinfo{booktitle}{\emph{Proceedings of the 46th International ACM SIGIR Conference on Research and Development in Information Retrieval}} (Taipei, Taiwan) \emph{(\bibinfo{series}{SIGIR '23})}. \bibinfo{publisher}{Association for Computing Machinery}, \bibinfo{address}{New York, NY, USA}, \bibinfo{pages}{2456–2460}.
\newblock
\showISBNx{9781450394086}
\urldef\tempurl%
\url{https://doi.org/10.1145/3539618.3592076}
\showDOI{\tempurl}


\bibitem[Li et~al\mbox{.}(2023a)]%
        {li2023propensitymatters}
\bibfield{author}{\bibinfo{person}{Haoxuan Li}, \bibinfo{person}{Yanghao Xiao}, \bibinfo{person}{Chunyuan Zheng}, \bibinfo{person}{Peng Wu}, {and} \bibinfo{person}{Peng Cui}.} \bibinfo{year}{2023}\natexlab{a}.
\newblock \showarticletitle{Propensity matters: measuring and enhancing balancing for recommendation}. In \bibinfo{booktitle}{\emph{Proceedings of the 40th International Conference on Machine Learning}} (Honolulu, Hawaii, USA) \emph{(\bibinfo{series}{ICML'23})}. \bibinfo{publisher}{JMLR.org}, Article \bibinfo{articleno}{831}, \bibinfo{numpages}{13}~pages.
\newblock


\bibitem[Li et~al\mbox{.}(2023b)]%
        {li2023stabledr}
\bibfield{author}{\bibinfo{person}{Haoxuan Li}, \bibinfo{person}{Chunyuan Zheng}, {and} \bibinfo{person}{Peng Wu}.} \bibinfo{year}{2023}\natexlab{b}.
\newblock \showarticletitle{Stable{DR}: Stabilized Doubly Robust Learning for Recommendation on Data Missing Not at Random}. In \bibinfo{booktitle}{\emph{The Eleventh International Conference on Learning Representations}}.
\newblock
\urldef\tempurl%
\url{https://openreview.net/forum?id=3VO1y5N7K1H}
\showURL{%
\tempurl}


\bibitem[Li et~al\mbox{.}(2024)]%
        {li2024datacomp}
\bibfield{author}{\bibinfo{person}{Jeffrey Li}, \bibinfo{person}{Alex Fang}, \bibinfo{person}{Georgios Smyrnis}, \bibinfo{person}{Maor Ivgi}, \bibinfo{person}{Matt Jordan}, \bibinfo{person}{Samir Gadre}, \bibinfo{person}{Hritik Bansal}, \bibinfo{person}{Etash Guha}, \bibinfo{person}{Sedrick Keh}, \bibinfo{person}{Kushal Arora}, {et~al\mbox{.}}} \bibinfo{year}{2024}\natexlab{}.
\newblock \showarticletitle{DataComp-LM: In search of the next generation of training sets for language models}.
\newblock \bibinfo{journal}{\emph{arXiv preprint arXiv:2406.11794}} (\bibinfo{year}{2024}).
\newblock


\bibitem[Li et~al\mbox{.}(2023c)]%
        {li2023synthetic}
\bibfield{author}{\bibinfo{person}{Zhuoyan Li}, \bibinfo{person}{Hangxiao Zhu}, \bibinfo{person}{Zhuoran Lu}, {and} \bibinfo{person}{Ming Yin}.} \bibinfo{year}{2023}\natexlab{c}.
\newblock \showarticletitle{Synthetic Data Generation with Large Language Models for Text Classification: Potential and Limitations}. In \bibinfo{booktitle}{\emph{Proceedings of the 2023 Conference on Empirical Methods in Natural Language Processing}}. \bibinfo{pages}{10443--10461}.
\newblock


\bibitem[Lu et~al\mbox{.}(2021)]%
        {lu2021less}
\bibfield{author}{\bibinfo{person}{Shuqi Lu}, \bibinfo{person}{Di He}, \bibinfo{person}{Chenyan Xiong}, \bibinfo{person}{Guolin Ke}, \bibinfo{person}{Waleed Malik}, \bibinfo{person}{Zhicheng Dou}, \bibinfo{person}{Paul Bennett}, \bibinfo{person}{Tie-Yan Liu}, {and} \bibinfo{person}{Arnold Overwijk}.} \bibinfo{year}{2021}\natexlab{}.
\newblock \showarticletitle{Less is More: Pretrain a Strong Siamese Encoder for Dense Text Retrieval Using a Weak Decoder}. In \bibinfo{booktitle}{\emph{Proceedings of the 2021 Conference on Empirical Methods in Natural Language Processing}}. \bibinfo{pages}{2780--2791}.
\newblock


\bibitem[Lu et~al\mbox{.}(2020)]%
        {lu2020twinbert}
\bibfield{author}{\bibinfo{person}{Wenhao Lu}, \bibinfo{person}{Jian Jiao}, {and} \bibinfo{person}{Ruofei Zhang}.} \bibinfo{year}{2020}\natexlab{}.
\newblock \showarticletitle{Twinbert: Distilling knowledge to twin-structured compressed bert models for large-scale retrieval}. In \bibinfo{booktitle}{\emph{Proceedings of the 29th ACM International Conference on Information \& Knowledge Management}}. \bibinfo{pages}{2645--2652}.
\newblock


\bibitem[Ma et~al\mbox{.}(2024)]%
        {ma2024finetuningllama}
\bibfield{author}{\bibinfo{person}{Xueguang Ma}, \bibinfo{person}{Liang Wang}, \bibinfo{person}{Nan Yang}, \bibinfo{person}{Furu Wei}, {and} \bibinfo{person}{Jimmy Lin}.} \bibinfo{year}{2024}\natexlab{}.
\newblock \showarticletitle{Fine-Tuning LLaMA for Multi-Stage Text Retrieval}. In \bibinfo{booktitle}{\emph{Proceedings of the 47th International ACM SIGIR Conference on Research and Development in Information Retrieval}} (Washington DC, USA) \emph{(\bibinfo{series}{SIGIR '24})}. \bibinfo{publisher}{Association for Computing Machinery}, \bibinfo{address}{New York, NY, USA}, \bibinfo{pages}{2421–2425}.
\newblock
\showISBNx{9798400704314}
\urldef\tempurl%
\url{https://doi.org/10.1145/3626772.3657951}
\showDOI{\tempurl}


\bibitem[Maini et~al\mbox{.}(2024)]%
        {maini2024rephrasing}
\bibfield{author}{\bibinfo{person}{Pratyush Maini}, \bibinfo{person}{Skyler Seto}, \bibinfo{person}{He Bai}, \bibinfo{person}{David Grangier}, \bibinfo{person}{Yizhe Zhang}, {and} \bibinfo{person}{Navdeep Jaitly}.} \bibinfo{year}{2024}\natexlab{}.
\newblock \showarticletitle{Rephrasing the web: A recipe for compute and data-efficient language modeling}.
\newblock \bibinfo{journal}{\emph{arXiv preprint arXiv:2401.16380}} (\bibinfo{year}{2024}).
\newblock


\bibitem[Malkov and Yashunin(2018)]%
        {malkov2018efficienthnsw}
\bibfield{author}{\bibinfo{person}{Yu~A Malkov} {and} \bibinfo{person}{Dmitry~A Yashunin}.} \bibinfo{year}{2018}\natexlab{}.
\newblock \showarticletitle{Efficient and robust approximate nearest neighbor search using hierarchical navigable small world graphs}.
\newblock \bibinfo{journal}{\emph{IEEE transactions on pattern analysis and machine intelligence}} \bibinfo{volume}{42}, \bibinfo{number}{4} (\bibinfo{year}{2018}), \bibinfo{pages}{824--836}.
\newblock


\bibitem[Marlin et~al\mbox{.}(2012)]%
        {marlin2012collaborative}
\bibfield{author}{\bibinfo{person}{Benjamin Marlin}, \bibinfo{person}{Richard~S Zemel}, \bibinfo{person}{Sam Roweis}, {and} \bibinfo{person}{Malcolm Slaney}.} \bibinfo{year}{2012}\natexlab{}.
\newblock \showarticletitle{Collaborative filtering and the missing at random assumption}.
\newblock \bibinfo{journal}{\emph{arXiv preprint arXiv:1206.5267}} (\bibinfo{year}{2012}).
\newblock


\bibitem[McAuley et~al\mbox{.}(2015)]%
        {mcauley2015inferring}
\bibfield{author}{\bibinfo{person}{Julian McAuley}, \bibinfo{person}{Rahul Pandey}, {and} \bibinfo{person}{Jure Leskovec}.} \bibinfo{year}{2015}\natexlab{}.
\newblock \showarticletitle{Inferring networks of substitutable and complementary products}. In \bibinfo{booktitle}{\emph{Proceedings of the 21th ACM SIGKDD international conference on knowledge discovery and data mining}}. \bibinfo{pages}{785--794}.
\newblock


\bibitem[Medini et~al\mbox{.}(2019)]%
        {medini2019extreme}
\bibfield{author}{\bibinfo{person}{Tharun Kumar~Reddy Medini}, \bibinfo{person}{Qixuan Huang}, \bibinfo{person}{Yiqiu Wang}, \bibinfo{person}{Vijai Mohan}, {and} \bibinfo{person}{Anshumali Shrivastava}.} \bibinfo{year}{2019}\natexlab{}.
\newblock \showarticletitle{Extreme Classification in Log Memory using Count-Min Sketch: A Case Study of Amazon Search with 50M Products}. In \bibinfo{booktitle}{\emph{Advances in Neural Information Processing Systems}}. \bibinfo{pages}{13244--13254}.
\newblock


\bibitem[Mittal et~al\mbox{.}(2024)]%
        {mittal2024graphregularizedencodertraining}
\bibfield{author}{\bibinfo{person}{Anshul Mittal}, \bibinfo{person}{Shikhar Mohan}, \bibinfo{person}{Deepak Saini}, \bibinfo{person}{Suchith~C. Prabhu}, \bibinfo{person}{Jain jiao}, \bibinfo{person}{Sumeet Agarwal}, \bibinfo{person}{Soumen Chakrabarti}, \bibinfo{person}{Purushottam Kar}, {and} \bibinfo{person}{Manik Varma}.} \bibinfo{year}{2024}\natexlab{}.
\newblock \bibinfo{title}{Graph Regularized Encoder Training for Extreme Classification}.
\newblock
\newblock
\showeprint[arxiv]{2402.18434}~[cs.LG]
\urldef\tempurl%
\url{https://arxiv.org/abs/2402.18434}
\showURL{%
\tempurl}


\bibitem[Mohan et~al\mbox{.}(2024)]%
        {mohan2024oak}
\bibfield{author}{\bibinfo{person}{Shikhar Mohan}, \bibinfo{person}{Deepak Saini}, \bibinfo{person}{Anshul Mittal}, \bibinfo{person}{Sayak~Ray Chowdhury}, \bibinfo{person}{Bhawna Paliwal}, \bibinfo{person}{Jian Jiao}, \bibinfo{person}{Manish Gupta}, {and} \bibinfo{person}{Manik Varma}.} \bibinfo{year}{2024}\natexlab{}.
\newblock \showarticletitle{{OAK}: Enriching Document Representations using Auxiliary Knowledge for Extreme Classification}. In \bibinfo{booktitle}{\emph{Forty-first International Conference on Machine Learning}}.
\newblock
\urldef\tempurl%
\url{https://openreview.net/forum?id=Cbacx90Wkt}
\showURL{%
\tempurl}


\bibitem[Mohankumar et~al\mbox{.}(2023)]%
        {mohankumar2023unified}
\bibfield{author}{\bibinfo{person}{Akash~Kumar Mohankumar}, \bibinfo{person}{Bhargav Dodla}, \bibinfo{person}{Gururaj K}, {and} \bibinfo{person}{Amit Singh}.} \bibinfo{year}{2023}\natexlab{}.
\newblock \showarticletitle{Unified Generative \& Dense Retrieval for Query Rewriting in Sponsored Search}. In \bibinfo{booktitle}{\emph{Proceedings of the 32nd ACM International Conference on Information and Knowledge Management}}. \bibinfo{pages}{4745--4751}.
\newblock


\bibitem[Mohri et~al\mbox{.}(2012)]%
        {Mohri12}
\bibfield{author}{\bibinfo{person}{M. Mohri}, \bibinfo{person}{A. Rostamizadeh}, {and} \bibinfo{person}{A. Talwalkar}.} \bibinfo{year}{2012}\natexlab{}.
\newblock \bibinfo{booktitle}{\emph{{F}oundations of {M}achine {L}earning}}.
\newblock \bibinfo{publisher}{MIT Press}.
\newblock


\bibitem[Muennighoff et~al\mbox{.}({[n.\,d.]})]%
        {muennighoff2024generative}
\bibfield{author}{\bibinfo{person}{Niklas Muennighoff}, \bibinfo{person}{SU Hongjin}, \bibinfo{person}{Liang Wang}, \bibinfo{person}{Nan Yang}, \bibinfo{person}{Furu Wei}, \bibinfo{person}{Tao Yu}, \bibinfo{person}{Amanpreet Singh}, {and} \bibinfo{person}{Douwe Kiela}.} \bibinfo{year}{[n.\,d.]}\natexlab{}.
\newblock \showarticletitle{Generative Representational Instruction Tuning}. In \bibinfo{booktitle}{\emph{ICLR 2024 Workshop: How Far Are We From AGI}}.
\newblock


\bibitem[Muennighoff et~al\mbox{.}(2024)]%
        {muennighoff2024scaling}
\bibfield{author}{\bibinfo{person}{Niklas Muennighoff}, \bibinfo{person}{Alexander Rush}, \bibinfo{person}{Boaz Barak}, \bibinfo{person}{Teven Le~Scao}, \bibinfo{person}{Nouamane Tazi}, \bibinfo{person}{Aleksandra Piktus}, \bibinfo{person}{Sampo Pyysalo}, \bibinfo{person}{Thomas Wolf}, {and} \bibinfo{person}{Colin~A Raffel}.} \bibinfo{year}{2024}\natexlab{}.
\newblock \showarticletitle{Scaling data-constrained language models}.
\newblock \bibinfo{journal}{\emph{Advances in Neural Information Processing Systems}}  \bibinfo{volume}{36} (\bibinfo{year}{2024}).
\newblock


\bibitem[Prabhu et~al\mbox{.}(2018a)]%
        {prabhu2018extreme}
\bibfield{author}{\bibinfo{person}{Yashoteja Prabhu}, \bibinfo{person}{Anil Kag}, \bibinfo{person}{Shilpa Gopinath}, \bibinfo{person}{Kunal Dahiya}, \bibinfo{person}{Shrutendra Harsola}, \bibinfo{person}{Rahul Agrawal}, {and} \bibinfo{person}{Manik Varma}.} \bibinfo{year}{2018}\natexlab{a}.
\newblock \showarticletitle{Extreme multi-label learning with label features for warm-start tagging, ranking \& recommendation}. In \bibinfo{booktitle}{\emph{Proceedings of the eleventh ACM international conference on web search and data mining}}. \bibinfo{pages}{441--449}.
\newblock


\bibitem[Prabhu et~al\mbox{.}(2018b)]%
        {prabhu2018parabel}
\bibfield{author}{\bibinfo{person}{Yashoteja Prabhu}, \bibinfo{person}{Anil Kag}, \bibinfo{person}{Shrutendra Harsola}, \bibinfo{person}{Rahul Agrawal}, {and} \bibinfo{person}{Manik Varma}.} \bibinfo{year}{2018}\natexlab{b}.
\newblock \showarticletitle{Parabel: Partitioned label trees for extreme classification with application to dynamic search advertising}. In \bibinfo{booktitle}{\emph{Proceedings of the 2018 World Wide Web Conference}}. \bibinfo{pages}{993--1002}.
\newblock


\bibitem[Prabhu and Varma(2014)]%
        {prabhu2014fastxml}
\bibfield{author}{\bibinfo{person}{Yashoteja Prabhu} {and} \bibinfo{person}{Manik Varma}.} \bibinfo{year}{2014}\natexlab{}.
\newblock \showarticletitle{Fastxml: A fast, accurate and stable tree-classifier for extreme multi-label learning}. In \bibinfo{booktitle}{\emph{Proceedings of the 20th ACM SIGKDD international conference on Knowledge discovery and data mining}}. \bibinfo{pages}{263--272}.
\newblock


\bibitem[Qaraei et~al\mbox{.}(2021)]%
        {qaraei2021convex}
\bibfield{author}{\bibinfo{person}{Mohammadreza Qaraei}, \bibinfo{person}{Erik Schultheis}, \bibinfo{person}{Priyanshu Gupta}, {and} \bibinfo{person}{Rohit Babbar}.} \bibinfo{year}{2021}\natexlab{}.
\newblock \showarticletitle{Convex Surrogates for Unbiased Loss Functions in Extreme Classification With Missing Labels}. In \bibinfo{booktitle}{\emph{Proceedings of the Web Conference}}. \bibinfo{pages}{3711–3720}.
\newblock


\bibitem[Qu et~al\mbox{.}(2020)]%
        {qu2020rocketqa}
\bibfield{author}{\bibinfo{person}{Yingqi Qu}, \bibinfo{person}{Yuchen Ding}, \bibinfo{person}{Jing Liu}, \bibinfo{person}{Kai Liu}, \bibinfo{person}{Ruiyang Ren}, \bibinfo{person}{Wayne~Xin Zhao}, \bibinfo{person}{Daxiang Dong}, \bibinfo{person}{Hua Wu}, {and} \bibinfo{person}{Haifeng Wang}.} \bibinfo{year}{2020}\natexlab{}.
\newblock \showarticletitle{RocketQA: An optimized training approach to dense passage retrieval for open-domain question answering}.
\newblock \bibinfo{journal}{\emph{arXiv preprint arXiv:2010.08191}} (\bibinfo{year}{2020}).
\newblock


\bibitem[Rae et~al\mbox{.}(2021)]%
        {rae2021scaling}
\bibfield{author}{\bibinfo{person}{Jack~W Rae}, \bibinfo{person}{Sebastian Borgeaud}, \bibinfo{person}{Trevor Cai}, \bibinfo{person}{Katie Millican}, \bibinfo{person}{Jordan Hoffmann}, \bibinfo{person}{Francis Song}, \bibinfo{person}{John Aslanides}, \bibinfo{person}{Sarah Henderson}, \bibinfo{person}{Roman Ring}, \bibinfo{person}{Susannah Young}, {et~al\mbox{.}}} \bibinfo{year}{2021}\natexlab{}.
\newblock \showarticletitle{Scaling language models: Methods, analysis \& insights from training gopher}.
\newblock \bibinfo{journal}{\emph{arXiv preprint arXiv:2112.11446}} (\bibinfo{year}{2021}).
\newblock


\bibitem[Rahmani et~al\mbox{.}(2024a)]%
        {llms_test1}
\bibfield{author}{\bibinfo{person}{Hossein~A Rahmani}, \bibinfo{person}{Nick Craswell}, \bibinfo{person}{Emine Yilmaz}, \bibinfo{person}{Bhaskar Mitra}, {and} \bibinfo{person}{Daniel Campos}.} \bibinfo{year}{2024}\natexlab{a}.
\newblock \showarticletitle{Synthetic Test Collections for Retrieval Evaluation}. In \bibinfo{booktitle}{\emph{Proceedings of the 47th International ACM SIGIR Conference on Research and Development in Information Retrieval}}. \bibinfo{pages}{2647--2651}.
\newblock


\bibitem[Rahmani et~al\mbox{.}(2024b)]%
        {llms_test2}
\bibfield{author}{\bibinfo{person}{Hossein~A Rahmani}, \bibinfo{person}{Clemencia Siro}, \bibinfo{person}{Mohammad Aliannejadi}, \bibinfo{person}{Nick Craswell}, \bibinfo{person}{Charles~LA Clarke}, \bibinfo{person}{Guglielmo Faggioli}, \bibinfo{person}{Bhaskar Mitra}, \bibinfo{person}{Paul Thomas}, {and} \bibinfo{person}{Emine Yilmaz}.} \bibinfo{year}{2024}\natexlab{b}.
\newblock \showarticletitle{LLM4Eval: Large Language Model for Evaluation in IR}. In \bibinfo{booktitle}{\emph{Proceedings of the 47th International ACM SIGIR Conference on Research and Development in Information Retrieval}}. \bibinfo{pages}{3040--3043}.
\newblock


\bibitem[Rawat et~al\mbox{.}(2021)]%
        {rawat2021disentangling}
\bibfield{author}{\bibinfo{person}{Ankit~Singh Rawat}, \bibinfo{person}{Aditya~K Menon}, \bibinfo{person}{Wittawat Jitkrittum}, \bibinfo{person}{Sadeep Jayasumana}, \bibinfo{person}{Felix Yu}, \bibinfo{person}{Sashank Reddi}, {and} \bibinfo{person}{Sanjiv Kumar}.} \bibinfo{year}{2021}\natexlab{}.
\newblock \showarticletitle{Disentangling sampling and labeling bias for learning in large-output spaces}. In \bibinfo{booktitle}{\emph{International conference on machine learning}}. PMLR, \bibinfo{pages}{8890--8901}.
\newblock


\bibitem[Reddi et~al\mbox{.}(2019)]%
        {reddi2019stochastic}
\bibfield{author}{\bibinfo{person}{Sashank~J Reddi}, \bibinfo{person}{Satyen Kale}, \bibinfo{person}{Felix Yu}, \bibinfo{person}{Daniel Holtmann-Rice}, \bibinfo{person}{Jiecao Chen}, {and} \bibinfo{person}{Sanjiv Kumar}.} \bibinfo{year}{2019}\natexlab{}.
\newblock \showarticletitle{Stochastic negative mining for learning with large output spaces}. In \bibinfo{booktitle}{\emph{The 22nd International Conference on Artificial Intelligence and Statistics}}. PMLR, \bibinfo{pages}{1940--1949}.
\newblock


\bibitem[Rosenbaum et~al\mbox{.}(2022a)]%
        {rosenbaum2022clasp}
\bibfield{author}{\bibinfo{person}{Andy Rosenbaum}, \bibinfo{person}{Saleh Soltan}, \bibinfo{person}{Wael Hamza}, \bibinfo{person}{Marco Damonte}, \bibinfo{person}{Isabel Groves}, {and} \bibinfo{person}{Amir Saffari}.} \bibinfo{year}{2022}\natexlab{a}.
\newblock \showarticletitle{CLASP: Few-Shot Cross-Lingual Data Augmentation for Semantic Parsing}. In \bibinfo{booktitle}{\emph{Proceedings of the 2nd Conference of the Asia-Pacific Chapter of the Association for Computational Linguistics and the 12th International Joint Conference on Natural Language Processing (Volume 2: Short Papers)}}. \bibinfo{pages}{444--462}.
\newblock


\bibitem[Rosenbaum et~al\mbox{.}(2022b)]%
        {rosenbaum2022linguist}
\bibfield{author}{\bibinfo{person}{Andy Rosenbaum}, \bibinfo{person}{Saleh Soltan}, \bibinfo{person}{Wael Hamza}, \bibinfo{person}{Yannick Versley}, {and} \bibinfo{person}{Markus Boese}.} \bibinfo{year}{2022}\natexlab{b}.
\newblock \showarticletitle{LINGUIST: Language Model Instruction Tuning to Generate Annotated Utterances for Intent Classification and Slot Tagging}. In \bibinfo{booktitle}{\emph{Proceedings of the 29th International Conference on Computational Linguistics}}. \bibinfo{pages}{218--241}.
\newblock


\bibitem[Saad-Falcon et~al\mbox{.}(2023)]%
        {saadfalcon2023udapdr}
\bibfield{author}{\bibinfo{person}{Jon Saad-Falcon}, \bibinfo{person}{Omar Khattab}, \bibinfo{person}{Keshav Santhanam}, \bibinfo{person}{Radu Florian}, \bibinfo{person}{Martin Franz}, \bibinfo{person}{Salim Roukos}, \bibinfo{person}{Avirup Sil}, \bibinfo{person}{Md~Arafat Sultan}, {and} \bibinfo{person}{Christopher Potts}.} \bibinfo{year}{2023}\natexlab{}.
\newblock \bibinfo{title}{UDAPDR: Unsupervised Domain Adaptation via LLM Prompting and Distillation of Rerankers}.
\newblock
\newblock
\showeprint[arxiv]{2303.00807}~[cs.IR]


\bibitem[Saito et~al\mbox{.}(2020)]%
        {saito2020unbiased}
\bibfield{author}{\bibinfo{person}{Yuta Saito}, \bibinfo{person}{Suguru Yaginuma}, \bibinfo{person}{Yuta Nishino}, \bibinfo{person}{Hayato Sakata}, {and} \bibinfo{person}{Kazuhide Nakata}.} \bibinfo{year}{2020}\natexlab{}.
\newblock \showarticletitle{Unbiased recommender learning from missing-not-at-random implicit feedback}. In \bibinfo{booktitle}{\emph{Proceedings of the 13th International Conference on Web Search and Data Mining}}. \bibinfo{pages}{501--509}.
\newblock


\bibitem[Schnabel et~al\mbox{.}(2016a)]%
        {schnabel2016recommendations}
\bibfield{author}{\bibinfo{person}{Tobias Schnabel}, \bibinfo{person}{Adith Swaminathan}, \bibinfo{person}{Ashudeep Singh}, \bibinfo{person}{Navin Chandak}, {and} \bibinfo{person}{Thorsten Joachims}.} \bibinfo{year}{2016}\natexlab{a}.
\newblock \showarticletitle{Recommendations as treatments: debiasing learning and evaluation}. In \bibinfo{booktitle}{\emph{Proceedings of the 33rd International Conference on International Conference on Machine Learning-Volume 48}}. \bibinfo{pages}{1670--1679}.
\newblock


\bibitem[Schnabel et~al\mbox{.}(2016b)]%
        {tobias2016treatments}
\bibfield{author}{\bibinfo{person}{Tobias Schnabel}, \bibinfo{person}{Adith Swaminathan}, \bibinfo{person}{Ashudeep Singh}, \bibinfo{person}{Navin Chandak}, {and} \bibinfo{person}{Thorsten Joachims}.} \bibinfo{year}{2016}\natexlab{b}.
\newblock \showarticletitle{Recommendations as treatments: debiasing learning and evaluation} \emph{(\bibinfo{series}{ICML'16})}. \bibinfo{publisher}{JMLR.org}, \bibinfo{pages}{1670–1679}.
\newblock


\bibitem[Schultheis et~al\mbox{.}(2022)]%
        {schultheis2022missing}
\bibfield{author}{\bibinfo{person}{Erik Schultheis}, \bibinfo{person}{Marek Wydmuch}, \bibinfo{person}{Rohit Babbar}, {and} \bibinfo{person}{Krzysztof Dembczynski}.} \bibinfo{year}{2022}\natexlab{}.
\newblock \showarticletitle{On missing labels, long-tails and propensities in extreme multi-label classification}. In \bibinfo{booktitle}{\emph{Proceedings of the 28th ACM SIGKDD Conference on Knowledge Discovery and Data Mining}}. \bibinfo{pages}{1547--1557}.
\newblock


\bibitem[Thomas et~al\mbox{.}(2024)]%
        {thomas2024large}
\bibfield{author}{\bibinfo{person}{Paul Thomas}, \bibinfo{person}{Seth Spielman}, \bibinfo{person}{Nick Craswell}, {and} \bibinfo{person}{Bhaskar Mitra}.} \bibinfo{year}{2024}\natexlab{}.
\newblock \showarticletitle{Large language models can accurately predict searcher preferences}. In \bibinfo{booktitle}{\emph{Proceedings of the 47th International ACM SIGIR Conference on Research and Development in Information Retrieval}}. \bibinfo{pages}{1930--1940}.
\newblock


\bibitem[Touvron et~al\mbox{.}(2023)]%
        {touvron2023llama}
\bibfield{author}{\bibinfo{person}{Hugo Touvron}, \bibinfo{person}{Louis Martin}, \bibinfo{person}{Kevin Stone}, \bibinfo{person}{Peter Albert}, \bibinfo{person}{Amjad Almahairi}, \bibinfo{person}{Yasmine Babaei}, \bibinfo{person}{Nikolay Bashlykov}, \bibinfo{person}{Soumya Batra}, \bibinfo{person}{Prajjwal Bhargava}, \bibinfo{person}{Shruti Bhosale}, {et~al\mbox{.}}} \bibinfo{year}{2023}\natexlab{}.
\newblock \showarticletitle{Llama 2: Open foundation and fine-tuned chat models}.
\newblock \bibinfo{journal}{\emph{arXiv preprint arXiv:2307.09288}} (\bibinfo{year}{2023}).
\newblock


\bibitem[Wang et~al\mbox{.}(2023)]%
        {wang2023query2doc}
\bibfield{author}{\bibinfo{person}{Liang Wang}, \bibinfo{person}{Nan Yang}, {and} \bibinfo{person}{Furu Wei}.} \bibinfo{year}{2023}\natexlab{}.
\newblock \showarticletitle{Query2doc: Query Expansion with Large Language Models}.
\newblock \bibinfo{journal}{\emph{arXiv preprint arXiv:2303.07678}} (\bibinfo{year}{2023}).
\newblock


\bibitem[Wang et~al\mbox{.}(2019)]%
        {wang2019doubly}
\bibfield{author}{\bibinfo{person}{Xiaojie Wang}, \bibinfo{person}{Rui Zhang}, \bibinfo{person}{Yu Sun}, {and} \bibinfo{person}{Jianzhong Qi}.} \bibinfo{year}{2019}\natexlab{}.
\newblock \showarticletitle{Doubly robust joint learning for recommendation on data missing not at random}. In \bibinfo{booktitle}{\emph{International Conference on Machine Learning}}. PMLR, \bibinfo{pages}{6638--6647}.
\newblock


\bibitem[Wang et~al\mbox{.}(2021)]%
        {wang2021combatingselection}
\bibfield{author}{\bibinfo{person}{Xiaojie Wang}, \bibinfo{person}{Rui Zhang}, \bibinfo{person}{Yu Sun}, {and} \bibinfo{person}{Jianzhong Qi}.} \bibinfo{year}{2021}\natexlab{}.
\newblock \showarticletitle{Combating Selection Biases in Recommender Systems with a Few Unbiased Ratings}. In \bibinfo{booktitle}{\emph{Proceedings of the 14th ACM International Conference on Web Search and Data Mining}} (Virtual Event, Israel) \emph{(\bibinfo{series}{WSDM '21})}. \bibinfo{publisher}{Association for Computing Machinery}, \bibinfo{address}{New York, NY, USA}, \bibinfo{pages}{427–435}.
\newblock
\showISBNx{9781450382977}
\urldef\tempurl%
\url{https://doi.org/10.1145/3437963.3441799}
\showDOI{\tempurl}


\bibitem[Wei et~al\mbox{.}(2021)]%
        {wei2021towards}
\bibfield{author}{\bibinfo{person}{Tong Wei}, \bibinfo{person}{Wei-Wei Tu}, \bibinfo{person}{Yu-Feng Li}, {and} \bibinfo{person}{Guo-Ping Yang}.} \bibinfo{year}{2021}\natexlab{}.
\newblock \showarticletitle{Towards robust prediction on tail labels}. In \bibinfo{booktitle}{\emph{Proceedings of the 27th ACM SIGKDD Conference on Knowledge Discovery \& Data Mining}}. \bibinfo{pages}{1812--1820}.
\newblock


\bibitem[Whitehouse et~al\mbox{.}(2023)]%
        {whitehouse2023llm}
\bibfield{author}{\bibinfo{person}{Chenxi Whitehouse}, \bibinfo{person}{Monojit Choudhury}, {and} \bibinfo{person}{Alham Aji}.} \bibinfo{year}{2023}\natexlab{}.
\newblock \showarticletitle{LLM-powered Data Augmentation for Enhanced Cross-lingual Performance}. In \bibinfo{booktitle}{\emph{Proceedings of the 2023 Conference on Empirical Methods in Natural Language Processing}}. \bibinfo{pages}{671--686}.
\newblock


\bibitem[Wydmuch et~al\mbox{.}(2021)]%
        {wydmuch2021propensity}
\bibfield{author}{\bibinfo{person}{Marek Wydmuch}, \bibinfo{person}{Kalina Jasinska-Kobus}, \bibinfo{person}{Rohit Babbar}, {and} \bibinfo{person}{Krzysztof Dembczynski}.} \bibinfo{year}{2021}\natexlab{}.
\newblock \showarticletitle{Propensity-scored probabilistic label trees}. In \bibinfo{booktitle}{\emph{Proceedings of the 44th International ACM SIGIR Conference on Research and Development in Information Retrieval}}. \bibinfo{pages}{2252--2256}.
\newblock


\bibitem[Xie et~al\mbox{.}(2024)]%
        {xie2024doremi}
\bibfield{author}{\bibinfo{person}{Sang~Michael Xie}, \bibinfo{person}{Hieu Pham}, \bibinfo{person}{Xuanyi Dong}, \bibinfo{person}{Nan Du}, \bibinfo{person}{Hanxiao Liu}, \bibinfo{person}{Yifeng Lu}, \bibinfo{person}{Percy~S Liang}, \bibinfo{person}{Quoc~V Le}, \bibinfo{person}{Tengyu Ma}, {and} \bibinfo{person}{Adams~Wei Yu}.} \bibinfo{year}{2024}\natexlab{}.
\newblock \showarticletitle{Doremi: Optimizing data mixtures speeds up language model pretraining}.
\newblock \bibinfo{journal}{\emph{Advances in Neural Information Processing Systems}}  \bibinfo{volume}{36} (\bibinfo{year}{2024}).
\newblock


\bibitem[Xiong et~al\mbox{.}(2020)]%
        {xiong2020approximate}
\bibfield{author}{\bibinfo{person}{Lee Xiong}, \bibinfo{person}{Chenyan Xiong}, \bibinfo{person}{Ye Li}, \bibinfo{person}{Kwok-Fung Tang}, \bibinfo{person}{Jialin Liu}, \bibinfo{person}{Paul Bennett}, \bibinfo{person}{Junaid Ahmed}, {and} \bibinfo{person}{Arnold Overwijk}.} \bibinfo{year}{2020}\natexlab{}.
\newblock \showarticletitle{Approximate nearest neighbor negative contrastive learning for dense text retrieval}.
\newblock \bibinfo{journal}{\emph{arXiv preprint arXiv:2007.00808}} (\bibinfo{year}{2020}).
\newblock


\bibitem[Ye et~al\mbox{.}(2020)]%
        {ye2020pretrained}
\bibfield{author}{\bibinfo{person}{Hui Ye}, \bibinfo{person}{Zhiyu Chen}, \bibinfo{person}{Da-Han Wang}, {and} \bibinfo{person}{Brian Davison}.} \bibinfo{year}{2020}\natexlab{}.
\newblock \showarticletitle{Pretrained generalized autoregressive model with adaptive probabilistic label clusters for extreme multi-label text classification}. In \bibinfo{booktitle}{\emph{International Conference on Machine Learning}}. PMLR, \bibinfo{pages}{10809--10819}.
\newblock


\bibitem[Yingqi~Qu and Wang(2021)]%
        {qu2021rocketqa}
\bibfield{author}{\bibinfo{person}{Jing Liu Kai Liu Ruiyang Ren Wayne Xin Zhao Daxiang Dong Hua~Wu Yingqi~Qu, Yuchen~Ding} {and} \bibinfo{person}{Haifeng Wang}.} \bibinfo{year}{2021}\natexlab{}.
\newblock \showarticletitle{RocketQA: An Optimized Training Approach to Dense Passage Retrieval for Open-Domain Question Answering}. In \bibinfo{booktitle}{\emph{In Proceedings of NAACL}}.
\newblock


\bibitem[You et~al\mbox{.}(2019)]%
        {you2019attentionxml}
\bibfield{author}{\bibinfo{person}{Ronghui You}, \bibinfo{person}{Zihan Zhang}, \bibinfo{person}{Ziye Wang}, \bibinfo{person}{Suyang Dai}, \bibinfo{person}{Hiroshi Mamitsuka}, {and} \bibinfo{person}{Shanfeng Zhu}.} \bibinfo{year}{2019}\natexlab{}.
\newblock \showarticletitle{Attentionxml: Label tree-based attention-aware deep model for high-performance extreme multi-label text classification}.
\newblock \bibinfo{journal}{\emph{Advances in neural information processing systems}}  \bibinfo{volume}{32} (\bibinfo{year}{2019}).
\newblock


\bibitem[Zhang et~al\mbox{.}(2021)]%
        {zhang2021fast}
\bibfield{author}{\bibinfo{person}{Jiong Zhang}, \bibinfo{person}{Wei-Cheng Chang}, \bibinfo{person}{Hsiang-Fu Yu}, {and} \bibinfo{person}{Inderjit Dhillon}.} \bibinfo{year}{2021}\natexlab{}.
\newblock \showarticletitle{Fast multi-resolution transformer fine-tuning for extreme multi-label text classification}.
\newblock \bibinfo{journal}{\emph{Advances in Neural Information Processing Systems}}  \bibinfo{volume}{34} (\bibinfo{year}{2021}), \bibinfo{pages}{7267--7280}.
\newblock


\bibitem[Zhang et~al\mbox{.}(2020)]%
        {zhang2020causaldebiasing}
\bibfield{author}{\bibinfo{person}{Wenhao Zhang}, \bibinfo{person}{Wentian Bao}, \bibinfo{person}{Xiao-Yang Liu}, \bibinfo{person}{Keping Yang}, \bibinfo{person}{Quan Lin}, \bibinfo{person}{Hong Wen}, {and} \bibinfo{person}{Ramin Ramezani}.} \bibinfo{year}{2020}\natexlab{}.
\newblock \showarticletitle{Large-scale Causal Approaches to Debiasing Post-click Conversion Rate Estimation with Multi-task Learning}. In \bibinfo{booktitle}{\emph{Proceedings of The Web Conference 2020}} (Taipei, Taiwan) \emph{(\bibinfo{series}{WWW '20})}. \bibinfo{publisher}{Association for Computing Machinery}, \bibinfo{address}{New York, NY, USA}, \bibinfo{pages}{2775–2781}.
\newblock
\showISBNx{9781450370233}
\urldef\tempurl%
\url{https://doi.org/10.1145/3366423.3380037}
\showDOI{\tempurl}


\end{thebibliography}
\onecolumn
\begin{appendix}

\section{Theoretical Proofs}

\begingroup
\addtocounter{theorem}{-3}

\begin{lemma}
If $y_{il}=0 \eqnspace \forall (\vec x_i,\vec z_l) \in D_m$, then for any test pair $(\vec x,\vec z) \sim D_m$, $R(\vec x,\vec z,k_m) \indep \func D$ where $\func D = \{\{\vec x_i,\vec y_i\}_{i=1}^{N},\{\vec z_l\}_{l=1}^L\} $ is the training dataset.
\end{lemma}
\begin{proof}
As $R(\vec x,\vec z,k_{\vec x, \vec z})$ is a deterministic function of its parameters, the following conditional independence holds true:
\begin{equation}
\label{eqn:lemma_eqn_1}
\{\{\vec x_i,\vec y_i\}_{i=1}^{N},\{\vec z_l\}_{l=1}^L\} \indep \vec x, \vec z, k_m | (\func K \setminus k_m)
\end{equation}
The independence w.r.t $\vec x, \vec z$ is due to sampling independence. The independence w.r.t $k_m$ arises because, as since $y_{il}=0 \eqnspace \forall (\vec x_i,\vec z_l) \in D_m$, the rest of the documents in $\func D$ depend only on $\func K \setminus k_m$ owing to Postulates ~\ref{postulate:externality} and ~\ref{postulate:sparsity}.

Additionally, the following also holds true owing to Postulate~\ref{postulate:Incompressibility} and independent sampling:
\begin{equation}
\label{eqn:lemma_eqn_2}
\vec x,\vec z,k_m \indep (\func K \setminus k_m)
\end{equation}

By combining (\ref{eqn:lemma_eqn_1}) and (\ref{eqn:lemma_eqn_2}) through the contraction lemma of independence, and then using the determinism of $R(\vec x,\vec z,k_{\vec x, \vec z})$ we get the desired:

\begin{equation}
\label{eqn:lemma_eqn_3}
R(\vec x,\vec z,k_m) \indep \{\{\vec x_i,\vec y_i\}_{i=1}^{N},\{\vec z_l\}_{l=1}^L\}.
\end{equation}
\end{proof}

\begin{corollary}
If $y_{il}=0 \eqnspace \forall (\vec x_i,\vec z_l) \in D_m$, then for any test pair $(\vec x,\vec z) \sim D_m$, $R(\vec x,\vec z,k_m) \indep \func M$; where $\func M $ is a predictive model trained deterministically from $\{\{\vec x_i,\vec y_i\}_{i=1}^{N},\{\vec z_l\}_{l=1}^L\}$.
\end{corollary}

\begin{proof}
As $\func M$ is a deterministic function of $\{\{\vec x_i,\vec y_i\}_{i=1}^{N},\{\vec z_l\}_{l=1}^L\}$, the following holds true:

\begin{equation}
\label{eqn:corollary_eqn_1}
\func M \indep \vec x,\vec z,k_m | \{\{\vec x_i,\vec y_i\}_{i=1}^{N},\{\vec z_l\}_{l=1}^L\}
\end{equation}

By combining (\ref{eqn:corollary_eqn_1}) with (\ref{eqn:lemma_eqn_3}) through the contraction lemma, we get the desired result.
\end{proof}

\begin{theorem}
Let $n$ be the number of clicked query-doc pair sampled for the training data, $B$ be the exposure bias while sampling, and $\bar{F}_m,p_m$ be defined as above. Then, for any $0 \le \delta \le 1$, with probability at least $1-\delta$, at least $max_{m=1}^K \frac{\bar{F}_m}{B} e^{-B. p_m. n} - \frac{1}{2 \sqrt{2n}} (\log{\frac{2K}{\delta}}^{\frac{3}{2}})$ fraction of the true relevant label distribution will correspond to systematic missing label region and therefore will be irrecoverable by supervised training.
\end{theorem}

\begin{proof}
Marginal probability of not sampling a knowledge item at all during training data generation is $(1-B.p_m)^n$. The associated probability of systematic missing labels then is $p_m (1-B.p_m)^n$

Now, the expected total loss due to all lost knowledge items, $\bar{E}$ is:
\begin{align}
\bar{E} = & \frac{1}{B} \op E \left[  \sum_{m=1}^K p_m (1-B.p_m)^n \right] \\
\ge & \frac{1}{B} \op E \left[ \sum_{m=m'}^K p_m (1-B.p_m)^n  \right] \\
\ge & \frac{1}{B} \max_{m=m'}^K \bar{F}_{m'} (1-B.p_{m'})^n  \\
\approx & \frac{1}{B} \max_{m=m'}^K \bar{F}_{m'} e^{-B. p_m'. n} 
\end{align}

where the last step uses the approximation that can be accurate when the value of $p_{m'}$ is small, {\it i.e.} on the long-tail.
\\
\\
Next, by using Extended McDiarmid's inequality ~\citep{Mohri12}, the value $E$ which is the probability of systematic missing label error for any random sample of dataset can be bounded as follows:

\begin{align}
\op P(E - \bar{E} \le \epsilon) \leq 2q + 2\exp \left\{ - 2 \frac{\left(\epsilon - q . np\right)^2}{np^2} \right\}
\end{align}

where for some $m'$, $p = p_{m'}$ and 

\begin{align}
q &= \sum_{m=1}^{m'} (1-p_m)^n \\
&\le m' (1-p_{m'})^n 
\end{align}

Plugging these back in Extended McDiarmid's inequality:

\begin{align}
\op P(E - \bar{E} \le \epsilon) \leq &\max_{m'=1}^K 2 m' (1-p_{m'})^n + 2 \exp \left\{ - 2 \frac{\left(\epsilon - q . np_{m'}\right)^2}{np_{m'}^2} \right\} \\
\le &\max_p 2 K (1-p)^n + 2 \exp \left\{ - 2 \frac{\left(\epsilon - q . np\right)^2} {np^2} \right\} \\
=& \max_p 2 K e^{n\log{(1-p)}} + 2 e^{\left\{ - 2 \frac{\left(\epsilon - q . np\right)^2} {np^2} \right\}} \\
\le& \max_p 2 K e^{-np} + 2 e^{\left\{ - 2 \frac{\epsilon^2} {np^2} \right\}} \\
\le& 2K e^{-2 n^{\frac{1}{3}} \epsilon^{\frac{2}{3}}}
\end{align}

As a result,

\begin{align}
\op P(E - \bar{E} \le \epsilon) \leq 2K e^{- n^{\frac{1}{3}} \epsilon^{\frac{2}{3}}}
\end{align}

In other words, for any small $\delta \le 1$, with probability at least $1-\delta$ on sampling the dataset $\func D$, the error $E$ of systematic missing labels is at least: $max_{m=1}^K \frac{\bar{F}_m}{B} e^{-B. p_m. n} - \frac{1}{2 \sqrt{2n}} (\log{\frac{2K}{\delta}}^{\frac{3}{2}})$. 

\end{proof}

\endgroup

\section{Long Tail of Knowledge required in XC Applications}

\begin{figure*}[ht]
    \centering
    \begin{subfigure}[b]{0.45\textwidth}
        \centering
        \includegraphics[width=\textwidth]{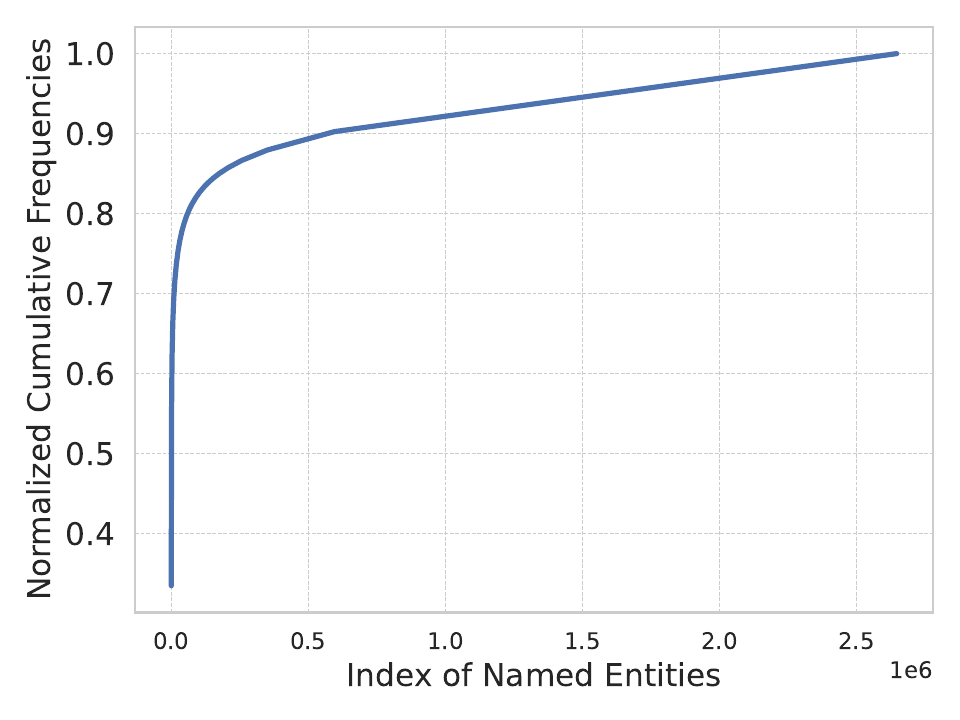}
        \caption{\orcas}
        \label{fig:figure1}
    \end{subfigure}
    \hfill
    \begin{subfigure}[b]{0.45\textwidth}
        \centering
        \includegraphics[width=\textwidth]{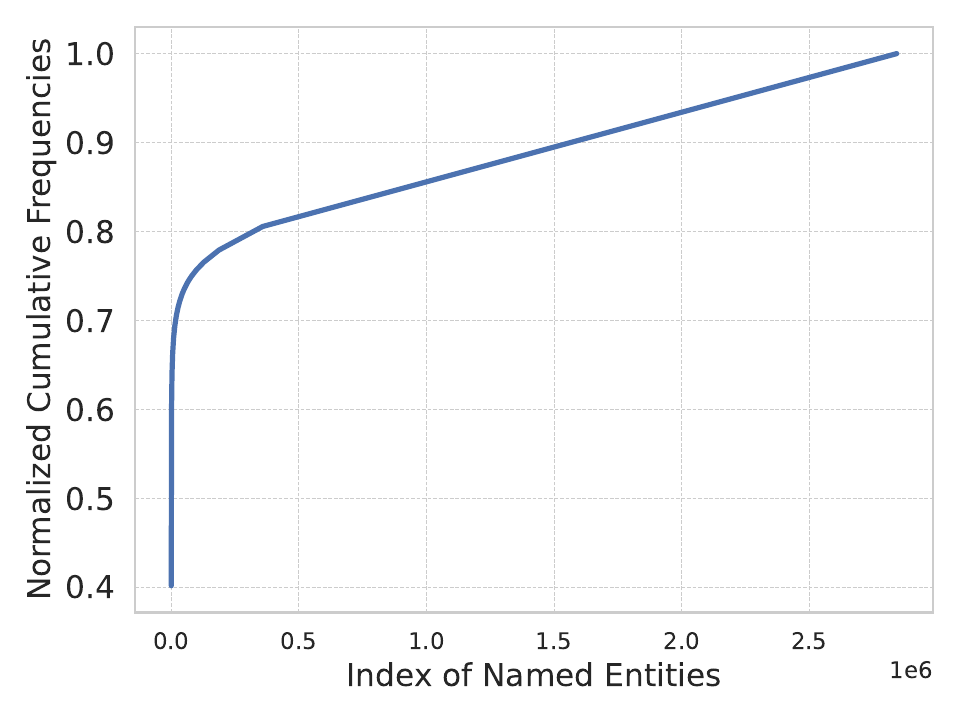}
        \caption{\wiki}
        \label{fig:figure2}
    \end{subfigure}
    
    \caption{We use named entity as a proxy for knowledge to show that XC applications require vast and long-tail knowledge. On the X-axis we plot the index of the named entities and on the Y-axis we show the normalized cumulative frequency of the named entities. In \orcas dataset \cite{Dahiya23bDEXA}, where the task is matching user queries to web-page URL, around 80\% of the knowledge is covered by 2.23\% entities. 
    While in the \wiki dataset where the task is to match Wikipedia titles to categories, around 11.28\% of the entities cover the same fraction of knowledge. 
    This highlights the extensive and long-tail knowledge required by XC applications.}
    \label{fig:long-tail-in-xc-applications}
\end{figure*}

\section{Dataset Curation and Statistics}
\label{sec:curations-stats}
This section describes the train-test splits used for training and evaluating different XC models. 
We benchmark on two public large-scale short-text XC datasets, namely \orcas and \wiki. We also evaluate \alg on the proprietary query to ad-keyword dataset curated from clicks logs of a popular search engine, we call it LF-Query2Keyword-10M. Below we outline the creation procedure and statistics for each dataset. \\

\noindent\textbf{\orcas:} We use the \orcas dataset released by \citet{dahiya2023deep} for XC, and originally proposed by \citet{craswell2020orcas}. The \orcas dataset models the task of mapping short user queries to web-page URLs. For example, a user query like "\textit{exons definition biology}" may be relevant to URLs like "\href{www.merriam-webster.com dictionary/exon}{\textit{www.merriam-webster.com dictionary/exon}}" and "\href{https://en.wikipedia.org/wiki/Exon}{\textit{en.wikipedia.org/wiki/Exon}}". For unbiased evaluation we curate a small unbiased test dataset using the human labelled test sets available on TREC-19 and TREC-20 Deep Learning competitions \cite{craswell2020overview}. Statistics of the dataset are provided in Table \ref{tab:dataset-stats}. \\

\noindent\textbf{\wiki:} This dataset was created by mining the Wikipedia dump \footnote{\href{https://dumps.wikimedia.org/enwiki/20240420/}{https://dumps.wikimedia.org/enwiki/20240420/}}. From the dump, we extract all Wikipedia titles and link each title to its associated categories and parent categories, as described in \cite{buvanesh2024enhancing}. We assume this dataset to be relatively unbiased since it was annotated by humans. To simulate click bias (controlled bias), we only include query-document pairs that appear in the top-200 predictions of a pre-trained \texttt{msmarco} retriever in the training set (refer Algorithm \ref{alg:dataset-curation}). This introduced bias mimics practical scenarios where users are shown a limited number of documents (or items) for a given query. Consequently, if an item is not in the top shortlist retrieved by the model, it will never be clicked and thus will not appear in the dataset curated from click logs. For unbiased evaluation, we use a test set that includes the complete ground truth without any simulation bias. Statistics of the dataset are provided in Table \ref{tab:dataset-stats}.

\begin{table}[h!]
\caption{Dataset statistics for train, biased test set, and unbiased test set. Due to lack of space, average points per document and average documents per point have been abbrevaiated as Av. PPL and Av. LPP respectively. \wiki has been shown as LF-WikiHT-2M. LF-Query2Keyword-10M is abbrevaiated as LF-Q2K-10M. Some details about the proprietary datasets has been redacted by ``-". We only provide approximate values for number of documents, train queries and test queries for LF-Q2K-10M.}
    \centering
    \begin{adjustbox}{max width=\textwidth}
    \begin{tabular}{lcccccccccc}
        \toprule
        Dataset & \# Documents & \multicolumn{3}{c}{Train Set} & \multicolumn{3}{c}{Biased Test Set} & \multicolumn{3}{c}{Unbiased Test Set} \\
        \cmidrule(lr){3-5} \cmidrule(lr){6-8} \cmidrule(lr){9-11}
         &   & \# Data Pts & Av. PPL & Av. LPP & \# Data Pts & Av. PPL & Av. LPP & \# Data Pts & Av. PPL & Av. LPP \\
        \midrule
        \orcas & 797322 & 7360881 & 16.133 & 1.747 & 2547702 & 5.676 & 1.776 & 88 & 1.009 & 27.750 \\
        LF-WikiHT-2M & 1966221 & 5418305 & 5.259 & 1.908 & 1355377  & 1.317 & 1.911 & 1355377 & 35.044 & 50.838 \\
        LF-Q2K-10M & $\sim$10000000 & $\sim$140000000 & - & - & $\sim$16000000  & - & - & $\sim$16000000 & - & - \\
        \bottomrule
    \end{tabular}
    \end{adjustbox}
    
    \label{tab:dataset-stats}
\end{table}

\begin{algorithm}[H]
\caption{Simulated bias dataset creation process for \wiki dataset.}
\label{alg:dataset-curation}
\textbf{Input:} Human annotated (\textit{unbiased}) dataset $\dataset$, Pre-trained \texttt{msmarco} model $\func M$, Parameter $k$ controlling the degree of exposure \\
\textbf{Ouput:} Biased Dataset $\dataset_{biased}$ \\

\For{\texttt{query} $q$ \texttt{in} $\dataset$}{
    $\func S_{q} \leftarrow$: Set of relevant documents for query $q$ in $\dataset$.\\
    $\func M_{q} \leftarrow$: Set of \texttt{top-k} predictions for query $q$ by model $\func M$.\\
    $\func T_{q} \leftarrow$: $\func S_q \cap \func M_q$

    \For{\texttt{doc} $d$ in $\func T_{q}$}{
        $\dataset_{biased} \leftarrow \dataset_{biased} \cup \{q, d\}$
    }
}
\end{algorithm}

\section{Baseline Implementation Details}
\label{sec:baseline-details}
This section describes the details of the baselines used in the main paper. We first describe both base XC methods, namely \renee and \dexml, and then move on to the baselines against which we compare \alg.

\subsection{Base XC Methods}
\noindent\textbf{\renee} \cite{jain2023renee}: This is a recent XC approach that makes use of one-versus-all (OvA) to represent each document. Query representations are derived from a 6-layer DistilBERT \footnote{\href{https://huggingface.co/docs/transformers/en/model_doc/distilbert}{https://huggingface.co/docs/transformers/en/model\_doc/distilbert}} encoder. We use the source code provided \href{https://github.com/microsoft/renee/tree/main}{here} for training \renee.

\noindent\textbf{\dexml} \cite{gupta2024efficacy}: This is a state-of-the-art dual encoder approach that derives the representations of both queries and documents using a 6-layer DistilBERT encoder. We use the provided source code \href{https://github.com/nilesh2797/DEXML}{here} for training \dexml. Due to the larger size of our datasets (800K documents in \orcas and 2M documents in \wiki), the hard negative mining configuration provided \href{https://github.com/nilesh2797/DEXML/blob/main/configs/LF-AmazonTitles-131K/dist-de-hnm_decoupled-softmax.yaml}{here} is used in all experiments.  

\subsection{Baselines}
We compare \alg with augmentation, propensity based and SLM-based rewriting approaches. Details of the baselines are provided below.

\noindent\textbf{Gandalf} \cite{kharbanda2023gandalf}: This is a graph-based data augmentation method that effectively models document-document correlations. As described in \cite{kharbanda2023gandalf}, we use a threshold of 0.1 to create the document-document augmentation matrix.

\noindent\textbf{LEVER} \cite{buvanesh2024enhancing}: This is a recent approach that improves the tail performance of any XC classifier by augmenting the ground truth using soft scores derived from a Siamese model. We use the source code provided \href{https://github.com/anirudhb11/LEVER}{here} for training LEVER.

\noindent\textbf{Inverse Propensity Scoring (IPS)} \cite{qaraei2021convex,jain2016extreme}: For propensity-based training, we derive per-document propensity estimates using the propensity model described in \cite{jain2016extreme}. For training \renee, we reweigh the positives using the factor described in \cite{qaraei2021convex}. In the case of \dexml, we simply use the propensity scores from \cite{jain2016extreme} to reweigh the positives.

\noindent\textbf{LEVER + IPS}: Drawing inspiration from the success of doubly robust methods \cite{li2023propensitymatters} that combine augmentation and propensity based methods, we create a baseline that combines reweighs the positives using IPS and augments the ground truth using LEVER. 

\noindent\textbf{\slmaug}: We used a slightly modified version of UDAPDR \cite{saadfalcon2023udapdr} as a baseline.
Specifically, we used a finetuned SLM on our generated GPT4 demonstrations but without metadata.
This is unlike \cite{saadfalcon2023udapdr, jeronymo2023inparsv2largelanguagemodels}, where the authors directly prompted a pretrained SLM for document/query generation. 
In fact, we found that directly prompting a pretrained SLM leads to drastically inferior document generation (see Figure \ref{fig:gpt4_pt_ft_llama}).
However, as done in \cite{saadfalcon2023udapdr, jeronymo2023inparsv2largelanguagemodels}, we do not employ a pretrained/finetuned re-ranker to filter the generated query-document pairs. 
We train the retriever on all the generated query-document pairs.
We provide the generated query-document pairs used in our experiments.

\section{\alg Implementation Details}
\label{sec:skim-impl-details}
This section discusses implementation details, hyperparameters used and the hardware employed for \alg. We cover: (i) task specific distillation of the SLM from LLM, (ii) Large-scale inference to generate synthetic queries (Step 1), and finally (iii) mapping synthetic queries to actual queries (Step 2) of \alg.

\textit{(i) Task-specific Distillation: } To create the distillation / finetuning data for the SLM, we use GPT4 as the LLM. We use the prompts mentioned in section \ref{sec:gpt4-llama-prompts}. We use the fast OpenAI access key to generate close to $\sim$40K-50K responses from GPT4. This takes a few hours to collect. Next, we use Low Rank Adaptation (LoRA) technique to finetune SLMs on GPT4 responses. We use the implementation provided here: \textcolor{blue}{\href{https://github.com/Lightning-AI/lit-gpt}{litgpt}} for finetuning SLM using LoRA. We finetune the SLM for a total of 2 epochs over the GPT4 responses for all datasets, using an effective batch-size of 64 and \texttt{bfloat16} weights, on H100 80GB GPU. We train \texttt{meta-llama/Llama-2-7b-hf} (Llama2 7B) for our main experiments, and train \texttt{microsoft/phi-1\_5} (Phi1.5 1.3B) for ablation.

\textit{(ii) Large-scale inference to generate synthetic queries:} Once we have the distilled / finetuned SLM, we run inference on all the (document, unstructured metadata) pairs present in our dataset to generate synthethic queries. For \orcas, we cap the tokens in the metadata to be 750 tokens, and limited the generated text to be at max 250 tokens. In \wiki, we cap the metadata to 1000, and the generations to 1000 tokens. We run large-scale inference of SLM to generate synthetic queries (Step 1 of \alg) again on RTX A6000 GPUs.

\textit{(iii) Mapping synthetic queries to actual queries (Step 2):} Once we have synthetic queries, we map them onto the actual queries using ANNS. We use the popular implementation of \citet{malkov2018efficienthnsw} provided here: \textcolor{blue}{\href{https://github.com/nmslib/hnswlib}{hnswlib}}. This takes a few minutes to calculate nearest neighbours on 10M document XC dataset on 96 core CPU machine without any GPUs. We linearly scale the normalized cosine score in $[0, 1]$, and then select the similarity threshold $\tau$ to filter the actual queries we add to the new training set for a document. For all datasets, we set default as $\tau=0.8$.

\subsection{Prompts for Synthetic Query (Document) Generation}
\label{sec:gpt4-llama-prompts}

We prompt GPT-4 and, subsequently, the fine-tuned LLAMA-2 on two datasets: \wiki and \orcas. For both datasets, we incorporate available metadata into the prompts. The GPT-4 prompts include a description of the task and manually curated in-context examples to ensure the generation of high-quality synthetic queries (or documents). The GPT-4 responses are then used to fine-tune the SLM. Once fine-tuned, the SLM inference is run by providing only the query (or document) and associated metadata, as through fine-tuning, the SLM has absorbed the task-specific information. The shorter SLM prompt also aids in scaling the to large-scale XC datasets.

\begin{figure}[htbp]
	\centering
\includegraphics[width=\linewidth]{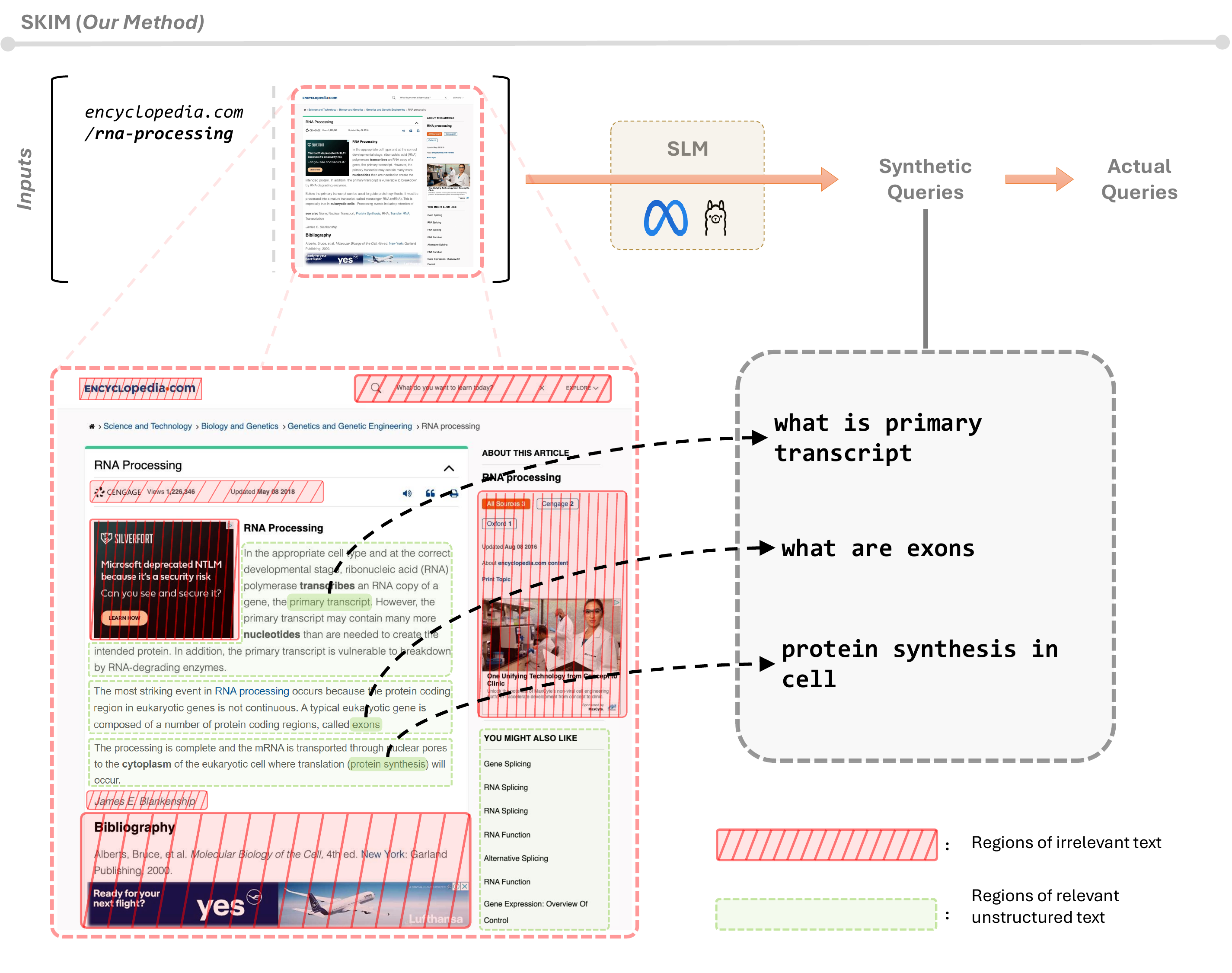}
	\caption{Intuitive visual explanation of how a finetuned SLM uses implicit knowledge contained in unstructured metadata to generate diverse synthetic queries that are representative of different knowledge concepts that could be absent in the training dataset. In other words, the finetuned SLM is able to \textit{skim} through the unstructured metadata, ignoring the non-relevant text, and generate only the relevant synthetic queries that are representative of different knowledge concepts about the document.}
\label{fig:slm_rephrasing_metadata_figure}
\end{figure}

\subsection{Time and Hardware requirements of \alg}

\begin{table}[h!]
\centering
\caption{Time taken and Hardware used in the different steps of \alg algorithm. In brackets, we mentioned the number of RTX A6000 48GB GPUs used for large-scale SLM inference. For SLM finetuning, we used H100 80GB GPU.}
\begin{tabular}{lcccc}
\toprule
 & Generating Examples from LLM (hrs) & Finetuning SLM (hrs) & SLM Inference (hrs) \\
\midrule
\orcas & 10 & 3 & 0.16 (128)  \\
LF-WikiHT-2M & 12 & 4 & 43.40 (128)  \\
\bottomrule
\end{tabular}
\label{tab:skim
-pipeline-time}
\end{table}

\section{Sponsored Search}
\label{sec:sponsored_search_appendix}
In this section we provide more details about the metadata and the proprietary filter model used in our sponsored search experiments.

\paragraph{Metadata used.} For every ad-keyword, we use the advertiser provided full landing page as unstructured metadata.

\paragraph{Proprietary Filter Model.} This model measures the relevance between a given query and an ad-keyword pair. This is implemented as a cross-encoder, initialized from BERT-Large, trained on large-scale human labelled (query, ad-keyword) pairs. 






\clearpage
\begin{tcolorbox}[lmprompt, title={\wiki GPT-4 Synthetic Query Generation Prompt.}]
\label{promptbox:gpt-wiki}
\#\#\# Task Instructions \eighttabs \fourtabs\tab\tab\texttt{/* \textbf{Describing the task} */} \\
You are given a Wikipedia article, specifically the title and body text of the article. Your job is to generate immediate Wikipedia categories as well as the Wikipedia categories for the generated immediate Wikipedia categories using the information provided in the body text of the article. 

When generating Wikipedia categories, make sure to follow below mentioned guidelines for a valid Wikipedia category:

1. Relevance: The categories you generate should be directly relevant to the topic of the Title. They should describe key aspects of the subject matter. For example, if given an Title about dogs, relevant categories might include "Mammals," "Pets," "Animal Behavior," or "Dog Breeds." Categories like "Astronomy" or "Cooking" would not be relevant to a dog-related Title.

2. Specificity: Wikipedia's category system is organized hierarchically, with broader categories containing more specific subcategories. Try to generate categories that aim to place the Title in the most specific category that applies.

First generate the immediate Wikipedia categories which are directly relevant to the article. After that, generate Wikipedia categories for the earlier generated immediate Wikipedia categories. Make sure to follow the guidelines that are mentioned above. 

Below are some examples that should provide more clarity about the task.

Example 1: \eighttabs\eighttabs \texttt{\textbf{/* In-Context Example */}}

\#\#\# Wikipedia Title\\
Sellankandal \eighttabs \eighttabs \tab \xspace \texttt{\textbf{/* Input query */}}

\#\#\# Wikipedia article begins \eighttabs \fourtabs \tab \tab \tab \texttt{\textbf{/* Meta-Data */}} \\
Sellankandal is a village situated 10 km inland from coastal Puttalam city in the North Western Province of Sri Lanka. It is the primary settlement of people of Black African descent in Sri Lanka called Kaffirs who until the 1930s spoke a Creole version of Portuguese. Most villages speak Sinhala and are found throughout the country as well as in the Middle East. The Baila type of music, very popular in Sri Lanka since the 1980s, originated centuries ago among this 'kaffir' community.\\
\#\#\# Wikipedia article ends \\

\#\#\# Task Output \eighttabs \fourtabs \tab  \texttt{\textbf{/* Synthetic Generated Documents */}} \\
\#\#\#\# Immediate Wikipedia Categories \\
Populated places in North Western Province, Sri Lanka \\
Populated places in Puttalam District \\
African diaspora in Sri Lanka \\

\#\#\#\# Wikipedia Categories for Immediate Wikipedia Categories \tab \texttt{\textbf{/* Example Category Hierarchy for this Query */}}  \\
Populated places in North Western Province, Sri Lanka \\
\tab Populated places in Sri Lanka by province \\
\tab Geography of North Western Province, Sri Lanka \\
Populated places in Puttalam District \\
\tab Populated places in North Western Province, Sri Lanka \\
\tab Populated places in Sri Lanka by district \\
\tab Geography of Puttalam District \\
African diaspora in Sri Lanka \\
\tab African diaspora in Asia \\
\tab Ethnic groups in Sri Lanka \\

Now perform the task for the following:

Please do not use any external knowledge to generate Wikipedia categories, the information must come only from the provided article, do not generate anything extra that is not there in the article.
**Make sure that the categories generated are actual Wikipedia Categories.**
Generate as much relevant immediate Wikipedia categories as possible.

\#\#\# Wikipedia Title\\
\{title\}

\#\#\# Wikipedia article begins\\
\{content\}\\
\#\#\# Wikipedia article ends

\#\#\# Task Output

\end{tcolorbox}

\begin{tcolorbox}[lmprompt, title={\orcas GPT-4 Synthetic Query Generation Prompt.}]
\label{promptbox:gpt4-orcas}
\# Task Instruction \eighttabs \fourtabs\tab\tab\phantom{xxx}\xspace\texttt{/* \textbf{Describing the task} */}

You are given a URL, along with some relevant search queries for this URL. Note that each of these relevant search queries is a "good" match for the given URL. Additionally, you are given some metadata in the form of the webpage text of this URL. You may use the relevant search queries and the webpage text to understand how the webpage text for this URL contains information that can answer or partially answer the relevant user queries. You need to solve the following task for this given URL. Please use only the information provided in the form of webpage text for the given URL, and the relevant search queries to solve the Task.

\#\# Query Generation from Webtext Task

Generate targeted search queries for the given URL using the webtext. The generated search queries contain tokens different from tokens of the URL. Generate upto 10 relevant search queries. Only generate new search queries, do not copy anything given in the relevant queries section.

Following are some examples for the task: \eighttabs \tab\tab\phantom{xxx} \texttt{\textbf{/* In-Context Example */}}

\#\#\# URL: \href{https://www.laserspineinstitute.com/back_problems/vertebrae/l3/}{https://www.laserspineinstitute.com/back\_problems/vertebrae/l3/} \fourtabs \tab\tab\texttt{\textbf{/* Input Document */}}

\#\#\# Relevant queries \eighttabs\tab\tab\phantom{xxx}\texttt{\textbf{/* Associated Queries in Training Data */}}\\
l3 compression fracture \\
l3-l4 symptoms \\
l3 and l4 vertebrae \\
l3 nerve root compression \\

\#\#\# Webpage text for URL begins \eighttabs\fourtabs\tab\tab \texttt{\textbf{/* Meta-Data */}}

Vertebrae Injury Vertebrae Fracture Back Vertebrae Vertebrae Pain Vertebrae Column Spinal Cord Vertebrae Between the Vertebrae Lumbar Spine Vertebrae Vertebrae Treatment Vertebrae Compression Vertebrae Compression Fracture Vertebrae Disc Vertebrae Nerve Spinal Column Vertebrae Vertebrae Surgery Neck Vertebrae Spine Vertebrae Spinal Vertebrae Compressed Vertebrae The L3 vertebra is located in the lumbar spine, which is in the lower back portion of the spinal column. The lumbar spine typically has five vertebrae, though some people range from four to six vertebrae in the lumbar region. The purpose of the lumbar spine is to stabilize and support the weight of the body, while still allowing the spine to move and bend freely. Because of the versatile nature of the lumbar spine, the vertebrae in this area are prone to injury and the development of spine conditions. The L3 vertebra is particularly susceptible to injury because it is the middle vertebra in the lumbar spine, which means it handles the most stress when the lumbar twists and bends.  The L3 vertebra holds most of the weight and stress of the body compared to the other vertebrae in the lumbar spine. Because of this, there are several spine conditions that can develop at the L3 vertebra and impact the surrounding nerve root, disc and/or joint. The most common spine conditions at the L3 vertebra include: Herniated disc Bulging disc Bone spurs Spondylosis Spondylolisthesis Arthritis of the spine There are several other conditions that may develop as a result of vertebral compression in the lumbar spine, which often impacts the disc and joints in the spine. If a spine condition occurs in the L3 vertebra, the symptoms will likely include chronic lower back pain and radiating pain in the buttock and leg of the impacted side. Additionally, the leg and foot might feel weak and numb due to the impacted nerve root being unable to send strong signals to the extremities. \\
\#\#\# Webpage text for URL ends

\#\#\# Output \\
\#\#\#\# Query Generation from Webtext Task \fourtabs\tab \texttt{\textbf{/* Useful Queries extracted from this web-page */}} \\ 
location of L3 \\ 
how many vertebrae in lumbar spine \\
causes for spondylosis \\
role of lumbar spine \\
chronic back pain \\ 
weak and numb feet \\

Now perform the Query Generation from Webtext task. Ensure that the generated query can be answered or contains information relevant to the webtext of the given URL. Try to generate new kinds of queries that do not overlap with the given relevant queries. Generate at most 10 search queries for the Query Generation from Webtext Task. Do not generate any query that has a text overlap with either the URL or the given relevant queries

\#\#\# URL
\{url\}

\#\#\# Relevant queries \\
\{relevant queries\} \\
\#\#\# Webpage text for URL begins \\
\{doc title\} \\
\{doc body\} \\
\#\#\# Webpage text for URL ends \\
\#\#\# Output 
\end{tcolorbox}



\begin{table*}[h!]
\caption{Results on biased and unbiased test sets for both XC base models: Renée and DEXML, comparing all baselines and \alg on \orcas}
\centering
\begin{tabular}{lGGGGcccc}
\toprule
& \multicolumn{8}{c}{\orcas} \\
\cmidrule(r){2-9}
&  \multicolumn{4}{G}{Biased} & \multicolumn{4}{c}{Unbiased}  \\
& P@5 & R@10 & R@25 & R@100 & P@5 & R@10 & R@25 & R@100  \\
\midrule
\dexml & 28.52 & 91.70 & 95.03 & 97.19 & 42.05 & 21.77 & 28.32 & 39.08\\
+ Gandalf & 28.38 & 91.60 & 95.09 & 97.29 & 44.09 & 22.39 & 29.87 & 40.52\\
+ LEVER & 28.15 & 91.19 & 94.81 & 97.12 & 43.64 & 21.70 & 29.26 & 40.45 \\
+ IPS & 26.60 & 88.72 & 93.50 & 96.59 & 45.45 & 22.31 & 29.14 & 39.42\\
+ LEVER + IPS & 27.30 & 89.93 & 94.21 & 96.92 & 46.14 & 23.47 & 30.92 & 41.04 \\
+ \slmaug & 27.28 & 90.18 & 94.56 & 97.24 & 46.36 & 26.42 & 33.71 & 47.70 \\
+ \alg (Ours) & 26.35 & 88.47 & 93.59 & 96.71 & 50.68 & 26.99 & 36.90 & 49.60\\
\midrule
\midrule
\renee & 27.85 & 90.44 & 94.19 & 96.64 & 45.68 & 21.75 & 28.22 & 38.42 \\
+ Gandalf & 28.78 & 91.98 & 94.97 & 96.94 & 44.09 & 20.89 & 27.58 & 34.66\\
+ LEVER & 29.13 & 92.95 & 95.96 & 97.76 & 45.45 & 22.58 & 28.40 & 39.52\\
+ IPS & 28.24 & 90.40 & 93.47 & 95.61 & 40.91 & 19.67 & 24.38 & 31.08 \\
+ LEVER + IPS & 28.61 & 91.74 & 94.85 & 96.88 & 43.64 & 21.52 & 29.24 & 37.10 \\
+ \alg (Ours) & 25.62 & 86.29 & 92.19 & 96.05 & 50.68 & 27.20 & 38.04 & 52.39 \\
\bottomrule
\end{tabular}
\label{tab:orcas-suppl-results}
\end{table*}

\begin{table*}[h!]
\caption{Results on biased and unbiased test sets for both XC base models: Renée and DEXML, comparing all baselines and \alg on \wiki}
\centering
\begin{tabular}{lGGGGcccc}
\toprule
& \multicolumn{8}{c}{\wiki} \\
\cmidrule(r){2-9}
&  \multicolumn{4}{G}{Biased} & \multicolumn{4}{c}{Unbiased}  \\
& P@5 & R@10 & R@25 & R@100 & P@5 & R@10 & R@25 & R@100  \\
\midrule
\dexml & 22.40 & 36.82 & 40.91 & 44.22 & 26.55 & 8.27 & 10.87 & 15.25 \\
+ Gandalf & 21.82 & 35.74 & 39.84 & 43.32 & 26.29 & 8.23 & 10.73 & 15.01 \\
+ LEVER  & 21.07 & 34.94 & 39.31 & 42.99 & 25.47 & 7.57 & 9.89 & 13.78 \\
+ IPS & 21.19 & 35.50 & 40.19 & 43.98 & 25.22 & 8.10 & 10.92 & 15.92\\
+ LEVER + IPS  & 19.01 & 32.24 & 37.50 & 42.35 & 22.23 & 6.82 & 9.03 & 12.92\\
+ UDAPDR & 15.70 & 26.54 & 31.47 & 37.16 & 22.05 & 6.83 & 10.27 & 17.15 \\
+ \alg (Ours) & 15.38 & 26.44 & 32.21 & 37.64 & 32.65 & 11.13 & 16.59 & 25.42 \\
\midrule
\midrule
\renee & 21.73 & 34.93 & 38.56 & 41.62 & 26.99 & 7.89 & 9.95 & 13.07 \\
+ Gandalf & 21.35 & 34.84 & 38.91 & 42.20 & 27.87 & 8.46 & 11.08 & 15.30\\
+ LEVER  & 21.76 & 35.78 & 40.05 & 43.86 & 27.83 & 8.59 & 11.15 & 15.98\\
+ IPS & 21.45 & 34.90 & 38.59 & 41.19 & 28.18 & 9.12 & 12.01 & 16.75\\
+ LEVER + IPS  & 21.00 & 35.01 & 39.48 & 43.54 & 27.46 & 9.23 & 12.63 & 19.28 \\
+ \alg (Ours) & 12.88 & 22.19 & 28.70 & 35.05 & 32.42 & 11.18 & 17.55 & 28.32\\
\bottomrule
\end{tabular}
\label{tab:wiki-suppl-results}
\end{table*}

\section{Ablation}
\label{sec:appendix-ablations}

\paragraph{Meta-data Augmented XC  ablation}: 
We provide more details for this ablation that was discussed in the main paper (See \ref{sec:main-ablations}). Just to recap, for this ablation, we wanted to test if directly using retrieval augmentation (RA) by providing available meta-data as part of the query would help in our scenario. To test this, we train a \renee model on the simulated biased training dataset of \wiki as used in our main experiments, however during training, we provide the metadata alongside the query. During test time too, we provide the metadata alongside the test query. We observe that this setting underperforms even when provided with the relevant knowledge for prediction. This goes to show that the \renee model memorizes the bias in the training dataset, and predicts the similar biased documents. This hints that RA like works might fail if used with biased training datasets. However, if we train the \renee model, again with metadata alongside the query, but on training dataset generated by \alg, we observe a sharp increase in performance. In this setting, the \renee + \alg model actually starts generalizing with the metadata present alongside the test query, and performs better as compared to the \renee + \alg model that uses only the test query during test time (see table \ref{tab:wiki_ra_doesnt_work}, R@100, in case of RA used with \alg, increases from 28.3 to 45.98 i.e. by more than 15 absolute points). This shows that RA works might be complementary to our proposed approach.

\paragraph{Effect of number of synthetic pairs (new training pairs) on unbiased performance}: In figure \ref{fig:orcas_split_vatiation}, we observe that as we increase the number of new training pairs or the synthetic pairs obtained using \alg, we see an increase in downstream unbiased performance of \renee model, as expected. We control the number of added pairs using the similarity threshold $\tau$. However, when we go pass a certain limit of number of added pairs, we see that the performance drops. This might be due to the fact that noisy pairs are being added in the dataset due to the low value of $\tau$. Thus, there is a sweet-spot when we obtain the maximum unbiased performance using \alg. This drop in performance could be mitigated by using a better retriever for mapping, or by increasing the number of synthetic queries that is being generated by the SLM during Step 1 of \alg.

\paragraph{Justification of design choices in SLM}: We perform an ablation to justify our design choices regarding (i) finetuning the SLM, and (ii) providing metadata to the SLM. In Figure \ref{fig:gpt4_pt_ft_llama}, we compare a pretrained SLM, finetuned SLM and the LLM performance on the task of good \textit{quality} synthetic document generation (measured using GT coverage @ 100, see figure \ref{fig:gpt4_pt_ft_llama}). SLM used is Llama2-7B and LLM is GPT4. We see a clear increase in performance when we move from pretrained SLM, to finetuned SLM and then finally to LLM. Additionally, we observe that metadata consistently improves generation quality when used with any model. Interestingly, finetuned SLM reaches close to LLM performance when provided with metadata. Thus, \alg is able to generate LLM-like quality of synthetic documents, with a much higher throughput, when finetuned and provided with relevant metadata, even though it is unstructured.

\paragraph{Size of the SLM used in \alg}: In table \ref{tab:slm_performance}, we see the effect of the size of the SLM used in Step 1 of \alg algorithm. We observe that even if we decrease the size of the SLM almost by 7x, we see marginal drop ($<1\%$) in downstream performance. This highlights extreme scalability of \alg even when used in limited compute regimes. How far can this be stretched is an interesting future work.

\begin{figure}[htbp]  
    \centering  
    \includegraphics[width=0.5\textwidth]{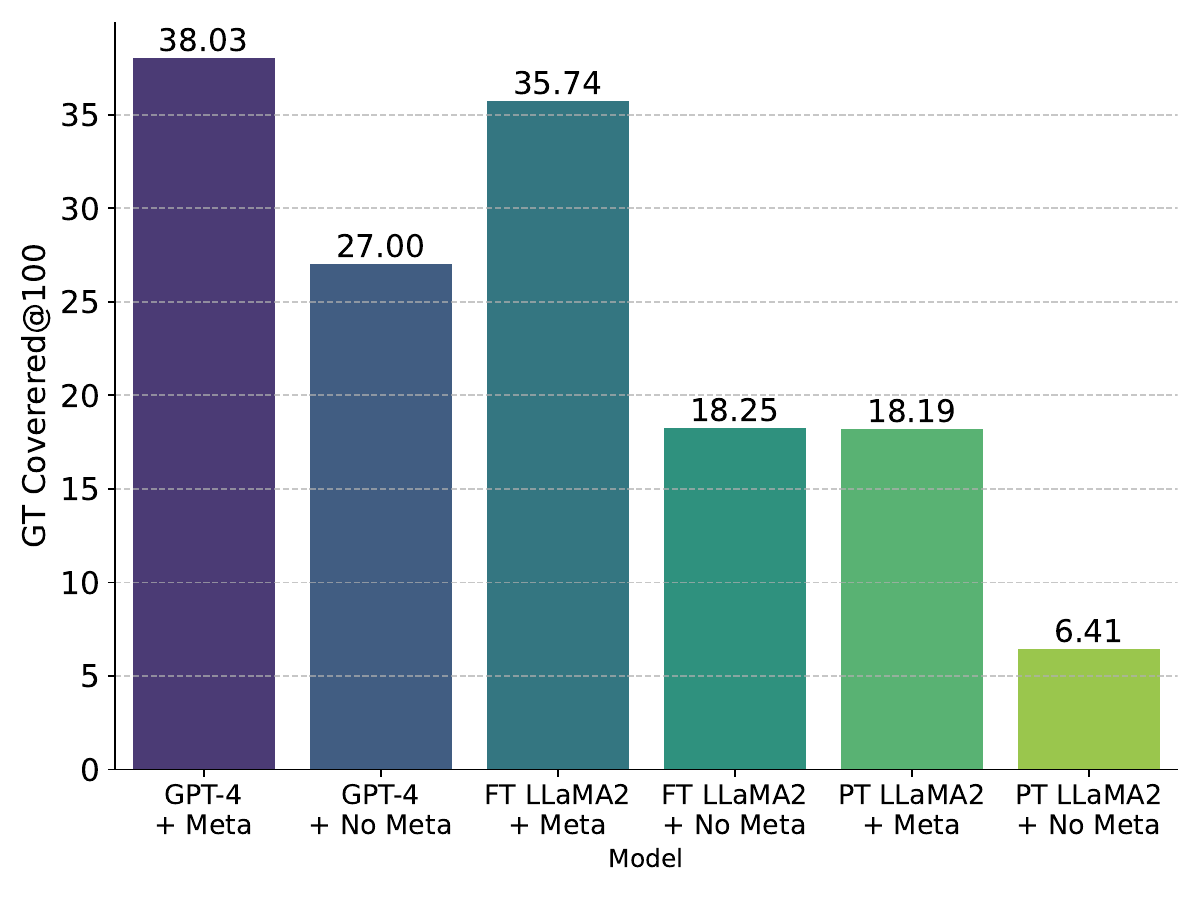}  
    \caption{
    Comparison of the effect of different language models and metadata on the quality of generated synthetic documents. GPT-4, a pretrained \llama-7B, and a finetuned \llama-7B are evaluated with and without metadata on the \wiki dataset where synthetic document generation is performed. The quality of synthetic documents is measured by GT covered @ 100, which is the average number of actual ground-truth (GT) documents retrieved using synthetic documents. For each synthetic document, the top 100 nearest documents are retrieved using a finetuned NGAME encoder and checked against the GT. The pretrained \llama~ has the lowest quality, while the finetuned \llama~ with metadata almost matches GPT-4 with metadata (\llama: 35.7 vs. GPT-4: 38.0).
    }  
    \label{fig:gpt4_pt_ft_llama}  
\end{figure}

\begin{figure}[htbp]  
    \centering  
    \includegraphics[width=0.5\textwidth]{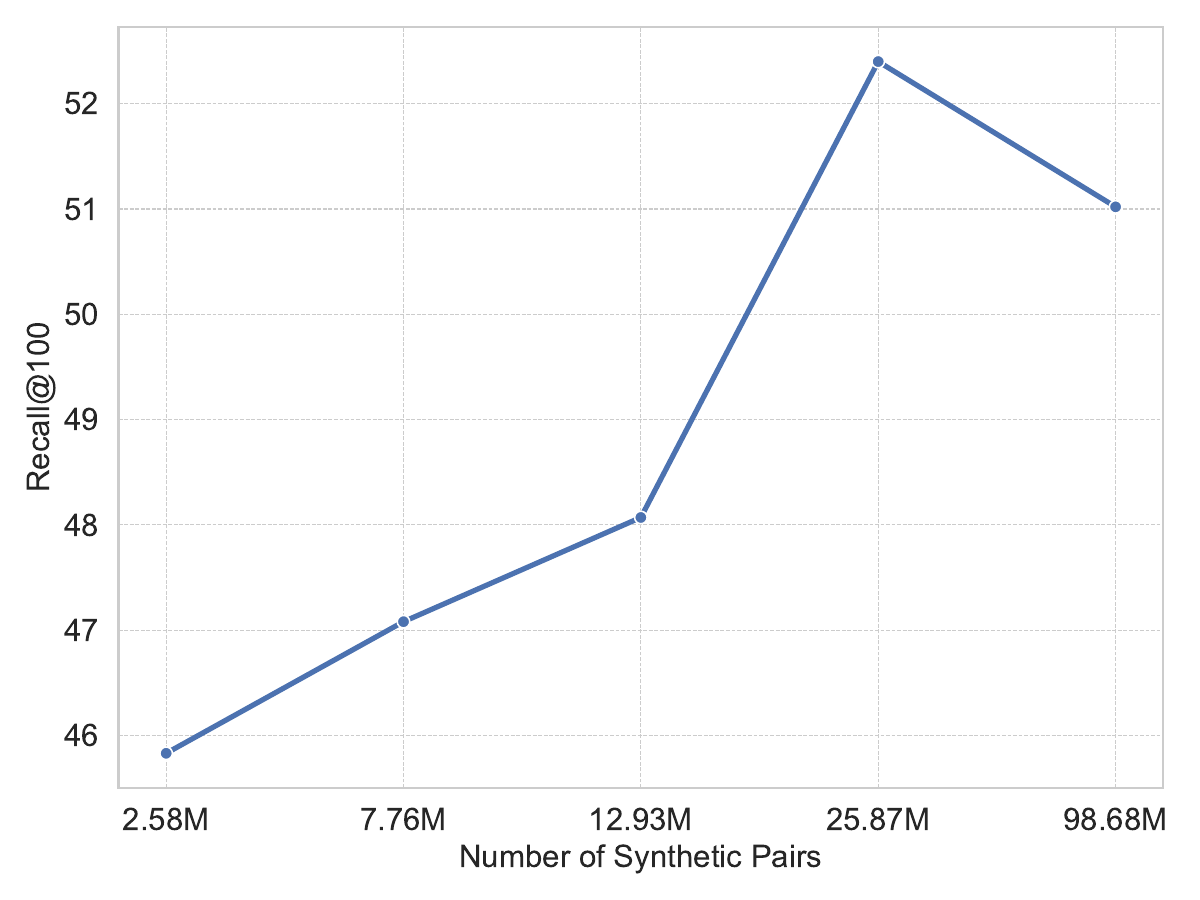}  
    \caption{Effect of the number of synthetic pairs added in Step 2 of \alg on the downstream retrieval performance of \renee on the \orcas dataset. We control the number of synthetic pairs added by tuning the similarity threshold $\tau$.}  
    \label{fig:orcas_split_vatiation}  
\end{figure}



\begin{table}[h!]
\centering
\small
\caption{Comparative analysis of the click-based ground truth (row 1), a state-of-the-art XC algorithm, \renee + LEVER (row 2), and our method, \renee + SKIM (row 3). Entries in black represent documents that have a click in the biased ground truth (row 1) or documents that are predicted by the models in top-100 predictions (rows 2 and 3). Grey cells indicate documents that are relevant to the query "Exon Definition" but are absent from both the ground truth (in row 1) or absent the top-100 predictions of the models (in row 2 and 3). All documents are relevant to the query "Exon Definition," according to the unbiased test set. Documents are grouped into four clusters based on their related concepts. The biased clicks in row 1  connect the query to document concepts such as "Exon" and "Exome Sequencing." However, there are no clicks linking "Exon Definition" to concepts like "Genes" or "RNA." The LEVER algorithm (row 2) retrieves documents from clusters similar to those seen during training (e.g., "Exon" and "Exome Sequencing") but struggles with other concepts. In contrast, \alg (row 3) successfully retrieves documents across all four clusters, emphasizing the value of integrating world knowledge. Document ranks are shown next to the URL in bold, and URLs have been shortened for brevity.} 
\begin{tabular}{|p{1.2cm}|p{1.2cm}|p{3cm}|p{3.25cm}|p{3.25cm}|p{3.25cm}|}
\hline
\multirow{2}{*}{} & \multirow{2}{*}{} & \multicolumn{4}{c|}{Document Concepts} \\ \cline{3-6} 
 &  Query & Exon & Exome  Sequencing & Genes & RNA \\ \hline
Biased Ground Truth & exon \newline definition & 
\textcircled{1}\texttt{merriam-webster.com dictionary/exon}, \newline \textcircled{2}\texttt{en.wikipedia.org wiki/Exon} \newline \textcolor{gray}{\textcircled{3}\texttt{net.science/diff-bw -exons-and-introns}}  & 

\textcircled{1}\texttt{en.wikipedia.org/ Exome-sequencing}, \newline 
\textcolor{gray}{\textcircled{2}\texttt{broadinstitute.org /what-exome-seq}}, \newline
\textcolor{gray}{\textcircled{3}\texttt{blogs.scientific american.com/10-things- exome-seq..}}
& 
\textcolor{gray}{\textcircled{1}\texttt{news-medical.net /What-are-Genes}} \newline
\textcolor{gray}{\textcircled{2}\texttt{www.genomenewsnetwork /whats a-genome/}} \newline
\textcolor{gray}{\textcircled{3}\texttt{en.wikipedia.org/wiki /Human-genome}} \newline
\textcolor{gray}{\textcircled{4}\texttt{en.wikipedia.org/wiki /Gene-expression}}
&  
\textcolor{gray}{\textcircled{1}\texttt{en.wikipedia.org/wiki /Messenger-RNA}} \newline
\textcolor{gray}{\textcircled{2}\texttt{en.wikipedia.org/wiki /Precursor-mRNA}} \newline
\textcolor{gray}{\textcircled{3}\texttt{en.wikipedia.org/wiki /RNA-splic..}} \newline
\textcolor{gray}{\textcircled{4}\texttt{en.wikipedia.org/wiki/ Alternative-splice}}

\\ \hline
\renee + LEVER Predictions (top-100) & exon \newline definition & 
\textcircled{1}\texttt{merriam-webster.com dictionary/exon} \textbf{(2)}, \newline \textcircled{2}\texttt{en.wikipedia.org wiki/Exon} \textbf{(1)} \newline \textcircled{3}\texttt{net.science/diff-bw -exons-and-introns} \textbf{(15)}
& 

\textcircled{1}\texttt{en.wikipedia.org/ Exome-sequencing} \textbf{(12)}, \newline 
\textcircled{2}\texttt{broadinstitute.org /what-exome-seq} \textbf{(51)}, \newline
\textcircled{3}\texttt{blogs.scientific american.com/10-things- exome-seq..}\textbf{(91)}
&
\textcolor{gray}{\textcircled{1}\texttt{news-medical.net /What-are-Genes}} \newline
\textcolor{gray}{\textcircled{2}\texttt{www.genomenewsnetwork /whats a-genome/}} \newline
\textcolor{gray}{\textcircled{3}\texttt{en.wikipedia.org/wiki /Human-genome}} \newline
\textcolor{gray}{\textcircled{4}\texttt{en.wikipedia.org/wiki /Gene-expression}}

& 

\textcolor{gray}{\textcircled{1}\texttt{en.wikipedia.org/wiki /Messenger-RNA}} \newline
\textcolor{gray}{\textcircled{2}\texttt{en.wikipedia.org/wiki /Precursor-mRNA}} \newline
\textcolor{gray}{\textcircled{3}\texttt{en.wikipedia.org/wiki /RNA-splic..}} \newline
\textcolor{gray}{\textcircled{4}\texttt{en.wikipedia.org/wiki/ Alternative-splic..}}
\\ \hline
 \renee + \alg predictions (top-100)&  exon \newline definition &
\textcircled{1}\texttt{merriam-webster.com dictionary/exon} \textbf{(4)}, \newline \textcircled{2}\texttt{en.wikipedia.org wiki/Exon} \textbf{(1)} \newline \textcircled{3}\texttt{net.science/diff-bw -exons-and-introns} \textbf{(10)}
&  
\textcircled{1}\texttt{en.wikipedia.org/ Exome-sequencing} \textbf{(18)}, \newline 
\textcircled{2}\texttt{broadinstitute.org /what-exome-seq} \textbf{(11)}, \newline
\textcircled{3}\texttt{blogs.scientific american.com/10-things- exome-seq..}\textbf{(35)}
&  
\textcircled{1}\texttt{news-medical.net /What-are-Genes} \textbf{(82)} \newline
\textcircled{2}\texttt{www.genomenewsnetwork /whats a-genome/} \textbf{(20)} \newline
\textcircled{3}\texttt{en.wikipedia.org/wiki /Human-genome} \textbf{(80)} \newline
\textcircled{4}\texttt{en.wikipedia.org/wiki /Gene-expression} \textbf{(52)}
& 
\textcircled{1}\texttt{en.wikipedia.org/wiki /Messenger-RNA} \textbf{61} \newline
\textcircled{2}\texttt{en.wikipedia.org/wiki /Precursor-mRNA} \textbf{(73)} \newline
\textcircled{3}\texttt{en.wikipedia.org/wiki /RNA-splic..} \textbf{(12)} \newline
\textcircled{4}\texttt{en.wikipedia.org/wiki/ Alternative-splic..} \textbf{(16)}

  \\ \hline
\end{tabular}
\label{tab:label_concepts}
\end{table}


\section{Related Works (Extended Version)}
\label{sec:appendix-rworks}

\textbf{Extreme classification}: Extreme Classification or XC is a prominent supervised learning formulation for large-scale (in orders of millions of documents) retrieval problems and has been very influential due to its success in many practical scenarios.
The methods proposed for extreme classification can be broadly categorized into two families: one-versus-all classifier-based methods \cite{jain2016extreme, prabhu2014fastxml} and dual encoders \cite{kharbanda2022cascadexml, jain2023renee, gupta2022elias, dahiya2023ngame, gupta2024dualencoders}.
Key innovations in this domain include transition from sparse feature based classifier learning \cite{agrawal2013multi,babbar2017dismec,jain2016extreme,babbar2019data,bhatia2015sparse,barezi2019submodular,jain2019slice,prabhu2014fastxml,prabhu2018extreme,prabhu2018parabel,khandagale2020bonsai} to deep-encoder based representations \cite{dahiya2021deepxml, dahiya2023deep, dahiya2023ngame}, efficient end-to-end training frameworks \cite{jain2023renee}, Siamese networks for query/document representations \cite{lu2020twinbert,qu2020rocketqa,lu2021less}, exploration of more advanced deep learning architectures for the encoder and the classifier components \cite{zhang2021fast,dahiya2021siamesexml,dahiya2021deepxml,gupta2022elias,kharbanda2022cascadexml,ye2020pretrained,you2019attentionxml}, and effective negative sampling strategies for training \cite{dahiya2023ngame,jiang2021lightxml,rawat2021disentangling,guo2019breaking,xiong2020approximate,reddi2019stochastic}. 

\textbf{Missing label and long-tail biases:} 
XC models, as other retrieval models, are susceptible to biases in their training set collected from system log data. Key biases that significantly affect the training and evaluation of XC models include document distribution bias \cite{buvanesh2024enhancing, schultheis2022missing}, selection bias\cite{marlin2012collaborative}, position bias\cite{collins2018study}, exposure bias \cite{lee2023uctrl, zhang2020causaldebiasing} and inductive bias (refer to \cite{chen2023bias} for a recent survey). Collectively, these biases lead to the \textit{missing label} problem wherein some relevant documents for a query are missing not at random. Additionally, certain documents are less represented than others in the training data, leading to a long-tail distribution over documents.
Consequently, missing labels and long-tail biases presents two critical challenges in XC \cite{schultheis2022missing}: 
1) relevant documents goes missing in the
observed training data; 2)  infrequent (tail) documents are much harder
to predict than frequent (head) documents due to data imbalance.
In Section~\ref{sec:theory}, we formalize how the systematic nature of missingness due to production systems leads not only to missing label bias but also to missing knowledge in training data. 

\textbf{Addressing missing label and long-tail biases:}
Within the XC literature, the most commonly adopted solution to the missing labels problem is propensity-based learning \cite{jain2016extreme,qaraei2021convex,wydmuch2021propensity,wei2021towards}. The propensity score estimates the likelihood for a query-document pair to go missing, given that the document is relevant. The standard loss functions and metrics can be corrected for the missing label bias by reweighting the individual terms by the propensity values.
Assuming that the propensities only depend on the document and that these values are known or can be estimated from some external meta-data \cite{mcauley2015inferring}, \citet{jain2016extreme} introduced unbiased loss functions and various evaluation metrics like nDCG@K, Recall@K and Precision@K. But, as we show in this paper, both theoretically and empirically, the missing label bias in XC can cause \textit{missing knowledge }problem which cannot be recovered by propensity based learning. Evaluation of XC models on unbiased datasets using propensity scored metrics, such as PSP@k, PSN@k \cite{bhatia2016extreme,jain2016extreme}, will also suffer from the same limitations. 


In similar settings such as recommendation systems \cite{saito2020unbiased,wang2019doubly,schnabel2016recommendations}, and positive unlabeled learning (PUL) \cite{bekker2018learning,bekker2020learning,jaskie2019positive,kato2019learning}, the missing labels problem has been addressed by propensity based methods \cite{jain2016extreme, qaraei2021convex}, imputation error based methods \cite{buvanesh2024enhancing, kharbanda2023gandalf}, and doubly robust methods \cite{wang2019doubly,li2023stabledr, Kweon2024doublycaliberated} that try to combine both these paradigms. Studies addressing this problem for recommendation systems using propensity-based corrections mechanisms \cite{tobias2016treatments, joachims2017unbiased, saito2020unbiased} often require an unbiased subset of training data over the user-item space to train the propensity models \cite{wang2021combatingselection, tobias2016treatments}. However, obtaining such unbiased subsets for large-scale retrieval applications is difficult and their absence may lead to high variance in estimated propensities. 

\textbf{Teacher Models and Data Augmentation:}
There is now a significant body of work that tries to train performant XC or retrieval models by augmenting \cite{dai2022promptagatorfewshotdenseretrieval, saadfalcon2023udapdr, jeronymo2023inparsv2largelanguagemodels, bonifacio2022inparsdataaugmentationinformation} the training dataset with external resources like teacher models, and query/document meta-data. Although these approaches don't make explicit connections to the problem of missing label bias, as show in Section 3, they are relevant to solving the missing labels problem. 
Recent efforts \cite{mittal2024graphregularizedencodertraining, mohan2024oak} have shown that auxiliary information or structured metadata linked with the queries or the documents, though rarely available with XC or short-text retrieval datasets, can be utilized to boost performance of XC models. 
LEVER \cite{buvanesh2024enhancing} and Gandalf \cite{kharbanda2023gandalf} uses a teacher-model trained on the existing dataset to help XC models enhance their performance on tail documents. 
RocketQA \cite{qu2021rocketqa} uses a more expressive cross-encoder model to impute the training data for a student dual-encoder retrieval model.
SemsupXC \cite{aggarwal2023semsupxc} is another closely related work that attempts to incorporate external information by scraping information from the web in order to improve performance in zero-shot or few-shot settings, involving previously unseen documents. 
But, as their model is trained on the biased training data, the trained model will continue to suffer from missing knowledge issue even with the help of additional meta-data.

Recently, LLMs have been used for generating data for training task-specific models \cite{lee2023making}, SLMs/LLMs \cite{gunasekar2023textbooksneedphi,maini2024rephrasing,hartvigsen2022toxigen,rosenbaum2022clasp,rosenbaum2022linguist}, and multimodal models \cite{li2024datacomp}. This approach to training models using LLM-generated datasets may help in combating various problems like data scarcity in low-resource settings, lack of unbiased datasets, noisy and fake information in web-based training data, or where annotating sufficient data is a costly process \cite{li2023synthetic,adler2024nemotron}. The strategies used for this kind of dataset creation approaches include heuristic filtering \cite{rae2021scaling}, quality filtering \cite{muennighoff2024scaling, fang2024data, abdin2024phi}, deduplication \cite{abbas2023semdedup}, data mixing \cite{xie2024doremi, albalak2023efficient, du2022glam}, synthetic data generation \cite{gunasekar2023textbooksneedphi, maini2024rephrasing}, data augmentation \cite{whitehouse2023llm}, or generating task-specific synthetic data by transforming existing datasets for a similar task \cite{gandhi2024better}. There have also been some efforts to directly finetune SLMs as encoders for query / documents in a retrieval task \cite{mohankumar2023unified,muennighoff2024generative,ma2024finetuningllama}.

In contrast, we propose a way for \textit{completing} existing task specific dataset by addressing \textit{knowledge gaps} present in the dataset. 
While existing work mainly focus on creating task-specific datasets from scratch or augmenting to extend the scale of these datasets, it is possible that the biases in the source dataset / LLM may just get scaled up in this synthetic data creation process. 
In addition,  using oracles LLMs like GPT-4 or Claude is not scalable, so there is a need for more efficient ways to augment the training data. 


\end{appendix}







\end{document}